\documentclass{article}

\usepackage[preprint]{neurips_2026}

\usepackage[utf8]{inputenc} % allow utf-8 input
\usepackage[T1]{fontenc}    % use 8-bit T1 fonts
\usepackage[colorlinks,citecolor=blue,urlcolor=blue,breaklinks]{hyperref}
\usepackage{url}            % simple URL typesetting
\usepackage{booktabs}       % professional-quality tables
\usepackage{nicefrac}       % compact symbols for 1/2, etc.
\usepackage{microtype}      % microtypography

\usepackage{amsmath, amssymb, amsthm, amsfonts, mathtools, bm, graphics, caption, natbib, float, latexsym, titlesec, enumitem, comment}
\DeclareMathAlphabet{\mathbbold}{U}{bbold}{m}{n}
\usepackage{multirow}
\usepackage[dvipsnames]{xcolor}
\usepackage{lscape}
\usepackage{aligned-overset}
\usepackage{subcaption}
\usepackage{overpic}
\usepackage{tikz}
\usetikzlibrary{positioning}

\DeclareMathOperator{\relu}{\mathsf{ReLU}}
\DeclareMathOperator{\softmax}{\mathsf{Softmax}}
\DeclareMathOperator{\attn}{\mathsf{Attn}}

\renewcommand{\d}{\mathrm{d}}
\newcommand{\Var}{\mathrm{Var}}

\newcommand{\indep}{\perp\!\!\!\perp}

\DeclareMathOperator*{\argmin}{argmin}

\definecolor{softblue}{RGB}{70,110,160}
\definecolor{darkblue}{RGB}{40,70,120}

\definecolor{softred}{RGB}{160,70,70}
\definecolor{darkred}{RGB}{120,40,40}

\definecolor{softgreen}{RGB}{60,120,90}
\definecolor{darkgreen}{RGB}{40,90,65}

\definecolor{softgray}{RGB}{110,110,110}
\definecolor{darkgray}{RGB}{70,70,70}

\definecolor{softpurple}{RGB}{110,80,140}
\definecolor{softorange}{RGB}{190,120,60}

\theoremstyle{plain}
\newtheorem{thm}{Theorem}[section]
\newtheorem{prop}[thm]{Proposition}
\newtheorem{lemma}[thm]{Lemma}
\newtheorem{cor}[thm]{Corollary}
\theoremstyle{definition}
\newtheorem{defn}[thm]{Definition}
\newtheorem{assumption}[thm]{Assumption}
\theoremstyle{remark}
\newtheorem{remark}[thm]{Remark}

\def\hat{\widehat}
\def\tilde{\widetilde}

\title{Optimal In-context Adaptivity and Distributional Robustness of Transformers}

\author{%
  Tianyi Ma \\
  Statistical Laboratory\\
  University of Cambridge\\
  \texttt{tm681@cam.ac.uk} \\
  \And
  Tengyao Wang\\
  Department of Statistics\\
  London School of Economics and Political Science\\
  \texttt{t.wang59@lse.ac.uk}\\
  \And
  Richard J. Samworth \\
  Statistical Laboratory\\
  University of Cambridge\\
  \texttt{r.samworth@statslab.cam.ac.uk} \\
}

\begin{document}

\maketitle

\begin{abstract}
  We study in-context learning problems where a Transformer is pretrained on tasks drawn from a mixture distribution $\pi=\sum_{\alpha\in\mathcal{A}} \lambda_{\alpha} \pi_{\alpha}$, called the pretraining prior, in which each mixture component~$\pi_{\alpha}$ is a distribution on tasks of a specific difficulty level indexed by~$\alpha$.  Our goal is to understand the performance of the pretrained Transformer when evaluated on a different test distribution~$\mu$, consisting of tasks of difficulty $\beta\in\mathcal{A}$, and with potential distribution shift relative to~$\pi_\beta$. In particular, we consider nonparametric regression problems with random smoothness, and multi-index models with both random smoothness and random effective dimension. We prove that a Transformer pretrained on such mixture distributions can adapt to the difficulty of a new task in-context and achieve the optimal prediction rate, uniformly over test distributions in a chi-squared divergence ball. Thus, the pretrained Transformer is able to achieve faster rates of convergence on easier tasks and is robust to distribution shift at test time. Finally, we prove that even if an estimator had access to the test distribution~$\mu$, the convergence rate of its expected risk over~$\mu$ could not be faster than that of our pretrained Transformers, thereby providing a more appropriate optimality guarantee than minimax lower bounds.
\end{abstract}

\section{Introduction}
Transformers \citep{vaswani2017attention} have emerged as one of the dominant architectures in modern machine learning, as they have achieved state-of-the-art performance in many domains such as natural language processing \citep{devlin2019bert,brown2020language}, computer vision \citep{dosovitskiy2021an,liu2021swin} and protein prediction \citep{jumper2021highly}.
A striking ability of pretrained Transformers, first observed by \citet{brown2020language} in Large Language Models (LLMs), is the phenomenon of \emph{in-context learning} (ICL): given a prompt (context) containing examples and a query, a Transformer can learn the underlying pattern from the examples and produce accurate output for the query, without updating its parameters. 
% For instance, given the values of the function $(x,y)\mapsto xy^2$ or `$\text{subfield}\mapsto \text{field}$' at a few distinct points, LLMs are able to recover the value of the function at a new point:
% \begin{equation}\label{eq:ICL-example}
% \scalebox{0.8}{
% $
% \begin{aligned}
% &\overbrace{\underbrace{f(1,1)=1,\; f(1,2)=4,\; f(2,2)=8,\; f(2,3)=18,}_{\text{examples}}\;
%     \underbrace{f(3,3)=}_{\text{query}}}^{\text{prompt}}
%     \underbrace{27 \vphantom{f(2,2)}}_{\text{output}},\\
%     &\overbrace{\underbrace{\text{Genetics}\!\to\!\text{Biology},\;
%     \text{Relativity}\!\to\!\text{Physics},}_{\text{examples}}\;
%     \underbrace{\text{Topology}\!\to\!}_{\text{query}}}^{\text{prompt}}\underbrace{\text{Mathematics} \vphantom{y}}_{\text{output}}.
% \end{aligned}
% $
% }
% \end{equation}
The ICL ability of Transformers is recently exploited by \citet{hollmann2025accurate} to build tabular foundation models, which they show to outperform many machine learning methods, such as tree-based and boosting algorithms, on regression and classification problems for tabular data.

\begin{figure}[t]
\centering
\begin{minipage}[c]{0.49\textwidth}
    \centering
    \includegraphics[width=\textwidth]{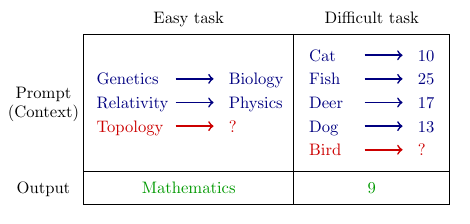}
    \captionof{figure}{Two ICL tasks. Blue indicates the (in-context) examples, red indicates the queries, and green indicates the outputs.}
    \label{fig:ICL-example}
\end{minipage}
\hfill
\begin{minipage}[c]{0.49\textwidth}
    \centering
    \includegraphics[width=0.8\textwidth]{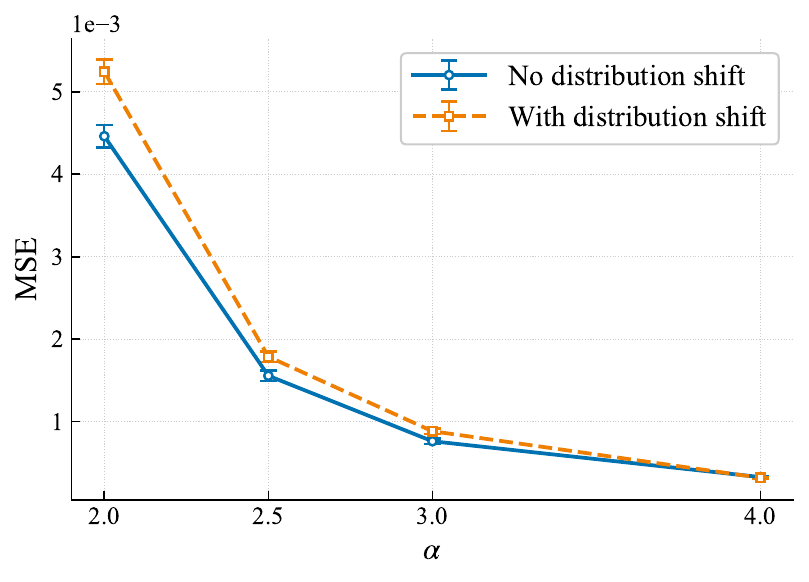}
    \captionof{figure}{Test-time MSE against the smoothness~$\alpha$ of the regression function.}
    \label{fig:simulation-preview}
\end{minipage}
\end{figure}

Extensive studies on ICL have been conducted from various perspectives \citep[e.g.][]{garg2022can,xie2022an,von2023transformers,bai2023transformers,kim2024transformers}; we refer the readers to Section~\ref{sec:related-work} for a more detailed literature review. Most theoretical analyses of ICL have focused on the setting where the distribution of pretraining data (which we call the \emph{pretraining prior}) and the distribution of the test data (which we call the \emph{test distribution}) are the same. However, in practice, Transformers are pretrained on a variety of tasks of different difficulties, and the test distribution will typically exhibit \emph{distribution shift} from the pretraining prior. 
For example, Figure~\ref{fig:ICL-example} provides an illustration of two ICL tasks. The in-context examples in the left column are mappings from subfields to their corresponding fields. This is a relatively easy task and LLMs can provide correct answers when given two such examples. In the right column of Figure~\ref{fig:ICL-example}, each word is mapped to $(\text{alphabetical position of its first letter}) \times (\text{total number of letters}) + 1$, e.g.~`C' is the third letter in the alphabet and `Cat' has three letters in total, so `Cat' is mapped to $3\times 3 + 1 = 10$. This is a harder problem and LLMs typically require more than two examples to correctly answer the question; see Appendix~\ref{sec:ICL-example}.  Thus, LLMs can adapt to the difficulties of the tasks in-context, in the sense that fewer examples are needed (or equivalently we witness faster rates of convergence) for easy tasks. Moreover, at test time we may focus more on a specific type of tasks, which represents a distribution shift from the pretraining prior.

To formulate a theoretical framework, we focus on regression problems (in line with most prior work on ICL theory). As an empirical preview, Figure~\ref{fig:simulation-preview} shows the test-time mean squared error (MSE) of a Transformer pretrained on a mixture of regression tasks with different smoothness levels. At test time, the Transformer is evaluated on new tasks of a certain smoothness, both under the original task distribution and under a distribution shift (see Section~\ref{sec:simulations} for more details). We see that the MSE decreases as the smoothness increases (higher smoothness means easier tasks), while the shifted and unshifted curves remain close, suggesting both in-context adaptivity to task difficulty and robustness to distribution shift. The main goal of this paper is to study these phenomena theoretically, and we provide an informal theorem below (see Section~\ref{sec:adaptive-icl} for rigorous statements).

\begin{center}
\fbox{
\parbox{0.9\textwidth}{\textbf{Informal theorem.}\;
\emph{A Transformer pretrained on a mixture of regression tasks of varying smoothness and effective dimension adapts to the difficulty of a new test task in-context and achieves the corresponding optimal prediction rate. These guarantees also hold under certain distribution shifts between the pretraining and test tasks.}}}
\end{center}

Our main idea is to view a pretrained Transformer as learning the posterior regression function induced by the pretraining prior. This perspective connects in-context learning with adaptive nonparametric Bayesian regression: once the Transformer approximates this posterior predictor, posterior contraction theory yields task-specific optimal rates, including adaptation to unknown smoothness and effective dimension. Moreover, the same Bayesian viewpoint allows us to establish robustness to distribution shift between the pretraining prior and the test distribution. To the best of our knowledge, our work provides the first proof that Transformer-based in-context learning can be optimally adaptive to task difficulty, while at the same time being robust to distribution shifts. 
% and consider a Transformer pretrained on a mixture distribution containing tasks of different difficulties, but at test time, tasks of a specific difficulty are drawn from a test distribution with potential distribution shift from the corresponding mixture component of the pretraining prior. We study the following natural but important questions:
% \vspace{-0.8em}
% \begin{quote}
%     \emph{Do pretrained Transformers adapt to the difficulty of the test distribution, and are they robust to distribution shift?}
% \end{quote}
% \vspace{-0.8em}
% We provide affirmative answers to these question; see Figure~\ref{fig:illustration} for an illustration of the main goals of this paper.

% \begin{figure}
% %https://excalidraw.com/#json=NaXHdoOG_DnZZkok6IR1H,3srmQNr40Ej7I-nOMzm_uA
%     \centering
%     \includegraphics[width=0.8\linewidth]{figures/illustration.pdf}
%     \caption{An illustration of the problem studied by this paper.}
%     \label{fig:illustration}
% \end{figure}

\subsection{Contributions}
\begin{itemize}[itemsep=-2pt, topsep=0pt, leftmargin=2em]
    \item In Proposition~\ref{prop:R_mu-decomposition}, we provide a key test-time risk upper bound, which decomposes into a term that captures both the degree of distribution shift and the proximity of the pretrained Transformer to the posterior regression function, and another term that represents the rate of convergence of the posterior regression function to the true regression function. 
    \item In Section~\ref{sec:universal-approximation}, we prove a universal approximation theorem for Transformers with softmax attention and non-polynomial activation in FFN layers. We further show that a large Transformer pretrained on sufficient data can learn the posterior regression function arbitrarily well when the activation in FFN layers is one of $\mathsf{ReLU}$, $\mathsf{GELU}$ or $\mathsf{SiLU}$. 
    \item In Sections~\ref{sec:holder-regression} and~\ref{sec:multi-index-regression}, we consider general nonparametric regression problems and multi-index models respectively, with regression functions belonging to Besov smoothness classes. We prove that the posterior regression function is optimally adaptive to the difficulty of every task in the support of the pretraining prior (and hence robust to distribution shift). In combination with the results mentioned above, this establishes the adaptivity and distributional robustness of pretrained Transformers.
    \item In Section~\ref{sec:lower-bounds}, we show that the expected risk of any estimator over the test distribution~$\mu$ cannot achieve a faster rate of convergence than our pretrained Transformers (even if the estimator is allowed to depend on $\mu$). This provides a more appropriate form of optimality guarantee for pretrained Transformers than minimax lower bounds usually considered in ICL literature. 
    \item Finally, in Section~\ref{sec:simulations}, we verify our theoretical results with empirical experiments.  These confirm our main conclusions that pretrained Transformers are able to adapt to different task difficulties and are robust to distribution shifts at test time.
\end{itemize}

\subsection{Related work} \label{sec:related-work}

\textbf{In-context learning}\;\;
There is a large body of theoretical research that aims to explain the ICL ability of Transformers.
One line of work shows that Transformers can perform ICL by approximating certain algorithms, such as gradient descent \citep{ahn2023transformers,akyurek2023what,bai2023transformers,von2023transformers}, functional gradient descent \citep{cheng2024transformers}, basis function regression \citep{kim2024transformers}, reinforcement learning \citep{lin2024transformers} and domain adaptation \citep{hataya2024automatic}.  Other lines of work analyse ICL through the lens of Bayesian inference \citep{xie2022an,zhang2022analysis,nagler2023statistical,wang2023large,deora2025incontext,zhang2025bayesianICL}, meta-learning \citep{dai2023gpt,jeon2024an,wu2025why} and training dynamics \citep{oko2024pretrained,zhang2024trained,kumano2025adversarially,kuwataka2025test}. In particular, \citet{bai2023transformers} show that Transformers can perform gradient descent on penalised linear regression, generalised linear models and two-layer neural networks, as well as implement in-context algorithm selection, while  \citet{kim2024transformers} prove that Transformers are minimax optimal for nonparametric regression. Both papers consider the case where the test distribution coincides with the pretraining prior. On the other hand, \citet{xie2022an} study  the setting where the test prompts have different formats compared to the pretraining data, and provide consistency results under a Bayesian hidden Markov model. The ability of Transformers to generalise to out-of-distribution tasks has been theoretically studied, under linear regression settings, by \citet{ahuja2023a,demir2025optimal,goddard2025when,wang2025can,kwon2026outofdistribution}. There are also numerous empirical studies on the behaviour and properties of ICL \citep[e.g.][]{garg2022can,kirsch2022generalpurpose,panwar2024incontext,reuter2025can}. Finally, we note the concurrent and independent work of \citet{wakayama2025in}, who study the adaptivity of ICL from a Bayesian perspective. Their results are complementary to ours, and neither subsumes the other. 
% To the best of our knowledge, we are the first to simultaneously justify the optimality, adaptivity and distributional robustness of ICL.

\textbf{Posterior contraction theory}\;\; Posterior contraction theory is an active area of research in statistics that studies the concentration of the posterior distribution around a true data generating distribution (i.e.~from a frequentist perspective). Pioneered by \citet{ghosal2000convergence} and \citet{shen2001rates}, posterior contraction for Bayesian nonparametric models has been investigated in different problems \citep[e.g.][]{ghosal2008nonparametric,Vaart2009Adaptive,finocchio2023posterior,egels2025posterior}. In this paper, we modify the tools developed by \citet{ghosal2008nonparametric} to derive rates of convergence for the posterior regression function. Interested readers are referred to the book by \citet{ghosal2017fundamentals} for an introduction to the field.

\section{Problem set-up and ICL excess risk decomposition}\label{sec:set-up}

We first remark that commonly used notation in the main text are summarised in Appendix~\ref{sec:notation-main-text}. Unless stated otherwise, all sets and functions are assumed to be measurable.
We now present the generating mechanism of our pretraining data. Let $\mathcal{X} \subseteq \mathbb{R}^d$, $\mathcal{Y} \subseteq \mathbb{R}$ and let $\mathcal{P}$ be a space of probability distributions on $\mathcal{X} \times \mathcal{Y}$ (dominated by a common $\sigma$-finite measure). For $T\in \mathbb{N}$, define $[T] \coloneqq \{1,\ldots,T\}$. The \emph{pretraining prior} $\pi$ is a distribution on $\mathcal{P}$, and our pretraining data take the form $(\mathcal{D}_n^{(t)}, X^{(t)}, Y^{(t)})_{t=1}^T$, generated as follows:
\begin{enumerate}[itemsep=0pt, topsep=2pt, leftmargin=20pt]
    \item Draw random distributions $P_1,\ldots,P_T \overset{\mathrm{iid}}{\sim}\pi$.
    \item For each $t\in[T]$, draw $\mathcal{D}_n^{(t)} \coloneqq (X_i^{(t)},Y_i^{(t)})_{i=1}^n \sim P_t^{\otimes n}$ and $(X^{(t)},Y^{(t)}) \sim P_t$ independently.
\end{enumerate}
Thus, for each $t\in[T]$, $\mathcal{D}_n^{(t)}$ consists of $n$ (in-context) examples, $X^{(t)}$ is our query and our goal is to predict its corresponding output $Y^{(t)}$; see Figure~\ref{fig:ICL-example}. We emphasise that the \emph{test distribution} $\mu$ on $\mathcal{P}$ of new prompts at test time will typically be different from the (possibly unknown) distribution $\pi$ that generated our pretraining data.

Next, for a class $\mathcal{F}$ of functions from $\mathcal{X} \times (\mathcal{X}\times\mathcal{Y})^n$ to~$\mathcal{Y}$ (e.g.~Transformers), define the \emph{empirical risk minimiser} $\hat{f}_T$ over $\mathcal{F}$ by
\begin{gather}
    \hat{f}_T \in \argmin_{f\in\mathcal{F}} \hat{\mathcal{R}}_T(f), \label{eq:f-hat}
\end{gather}
where $\hat{\mathcal{R}}_T(f) \coloneqq \frac{1}{T} \sum_{t=1}^T \bigl\{Y^{(t)} - f(X^{(t)},\mathcal{D}_n^{(t)})\bigr\}^2$ and $(\mathcal{D}_n^{(t)}, X^{(t)}, Y^{(t)})_{t=1}^T$ are generated according to $\pi$.
We remark that, given a new set of examples $D_n\in (\mathcal{X}\times\mathcal{Y})^n$, the prediction at a query $x\in\mathcal{X}$ is given by $\hat{f}_T(x,D_n)$, and we do not need to update $\hat{f}_T$ accordingly at test time. 

For a distribution $\nu$ on $\mathcal{P}$, we define the \emph{$\nu$-risk} of $g: \mathcal{X} \times (\mathcal{X}\times\mathcal{Y})^n \to \mathcal{Y}$ as\footnote{We adopt the convention that 
    for $P\in\mathcal{P}$, $\mathbb{E}_P$ means that the random variables are generated from~$P$, i.e.~$(\mathcal{D}_n,X,Y) \sim P^{\otimes(n+1)}$, whereas for a distribution~$\nu$ on~$\mathcal{P}$, $\mathbb{E}_{\nu}\mathbb{E}_P$ means that the distribution~$P$ is also random and drawn from~$\nu$.}
\begin{align*}
    \mathcal{R}_{\nu}(g) &\coloneqq \mathbb{E}_{P\sim\nu} \mathbb{E}_{(\mathcal{D}_n,X,Y)|P\sim P^{\otimes (n+1)}} \bigl[\bigl\{Y \!-\! g(X,\mathcal{D}_n)\bigr\}^2\bigr] \equiv \mathbb{E}_{\nu} \mathbb{E}_{P}\bigl[\bigl\{Y - g(X,\mathcal{D}_n)\bigr\}^2\bigr].
\end{align*}
%which is the mean squared error between $Y$ and $g(X,\mathcal{D}_n)$, where $(\mathcal{D}_n,X,Y)\,|\,P \sim P^{\otimes (n+1)}$ and $P\sim\nu$.
If $\tilde{f}_T$ is random and depends on the pretraining data $(\mathcal{D}_n^{(t)}, X^{(t)}, Y^{(t)})_{t=1}^T$, then we interpret $\mathcal{R}_{\nu}(\tilde{f}_T)$ as the expectation conditional on the pretraining data, which is assumed to be independent of the test data $(\mathcal{D}_n,X,Y)$, so $\mathcal{R}_{\nu}(\tilde{f}_T)$ is also random.  For a test distribution $\mu$ on~$\mathcal{P}$, we further define the \emph{ICL excess $\mu$-risk} of $\tilde{f}_T$ by
\begin{align}
    \mathcal{R}^{\mathrm{ICL}}_{\mu}(\tilde{f}_T) &\coloneqq \mathbb{E}\mathcal{R}_{\mu}(\tilde{f}_T) - \mathbb{E}_\mu \mathbb{E}_P \bigl[\bigl\{Y - \mathbb{E}_P(Y\,|\,X)\bigr\}^2\bigr] = \mathbb{E}\,\mathbb{E}_{\mu}\mathbb{E}_P \bigl[ \bigl\{\tilde{f}_T(X,\mathcal{D}_n) - \mathbb{E}_P(Y\,|\,X)\bigr\}^2 \bigr], \label{eq:ICL-excess-risk}
\end{align}
where the term $\mathbb{E}_\mu \mathbb{E}_P \bigl[\bigl\{Y - \mathbb{E}_P(Y\,|\,X)\bigr\}^2\bigr]$ is the oracle risk assuming knowledge of the distribution~$P$ that generates $(X,Y)$, and in particular of the true regression function $\mathbb{E}_P(Y\,|\,X)$ under~$P$; the equality~\eqref{eq:ICL-excess-risk} is shown in the proof of Proposition~\ref{prop:R_mu-decomposition}, see~\eqref{eq:ICL-excess-risk-alternative-form}. 
Finally, we define the \emph{posterior regression function} $g_\pi : \mathcal{X} \times (\mathcal{X}\times\mathcal{Y})^n \to \mathcal{Y}$ (with respect to $\pi$) by
\begin{align}
    g_\pi (x,D_n) &\coloneqq  \mathbb{E}_{\pi} \mathbb{E}_{P} \bigl(Y \,|\, X=x,\, \mathcal{D}_n=D_n\bigr) \label{eq:def-posterior-regression-function}
\end{align}
for $x\in\mathcal{X}$ and $D_n \in (\mathcal{X}\times\mathcal{Y})^n$.  Thus $g_\pi(x,D_n)$ is the posterior mean of $Y$ given $X=x$ and $\mathcal{D}_n=D_n$, where $(\mathcal{D}_n,X,Y)\,|\, P \sim P^{\otimes (n+1)}$ and $P\sim\pi$. 
%\blue{Formally, an \emph{in-context learning algorithm} is a Borel measurable function $\tilde f_T: \mathcal{X}\times (\mathcal{X}\times \mathcal{Y})^n \times \bigl((\mathcal{X}\times \mathcal{Y})^{n+1} \bigr)^T \to \mathcal{Y}$, so that $\tilde f_T\bigl(X,\mathcal{D}_n; (\mathcal{D}_n^{(t)},X^{(t)},Y^{(t)})_{t=1}^T\bigr)$ is the prediction at the test prompt $(X,\mathcal{D}_n)$ based on the pretraining data $(\mathcal{D}_n^{(t)},X^{(t)},Y^{(t)})_{t=1}^T$. For simplicity of exposition, we often drop the dependence of $\tilde{f}_T$ on the pretraining data and write $\tilde{f}_T(X,\mathcal{D}_n)$ instead. TM: this shouldn't be called an ICL algorithm! The crucial feature of an ICL algorithm is that it does not need to update its parameters given new prompts, and this definition does not reflect this.}

The following proposition provides a key upper bound on the ICL excess $\mu$-risk, which is crucial for our study of the adaptivity and distributional robustness of ICL. Its proof can be found in Appendix~\ref{sec:proof-of-risk-decomposition}.
\begin{prop} \label{prop:R_mu-decomposition}
    Let $R>0$, and suppose that for all $P \in \mathcal{P}$, we have $|\mathbb{E}_P(Y\,|\,X)| \leq R$ almost surely.  
    %\blue{Let
    %$\tilde{f}_T$ be an in-context learning algorithm satisfying $\|\tilde{f}_T\|_{\infty} \leq R$ almost surely,} 
    Further let $\hat{f}_T$ be defined as in~\eqref{eq:f-hat}, let $\tilde{f}_T \coloneqq \mathsf{clip}_R \circ \hat{f}_T$ and let $\mathcal{E}(\tilde{f}_T) \coloneqq \mathbb{E}\mathcal{R}_{\pi}(\tilde{f}_T) - \mathcal{R}_{\pi}(g_\pi)$.  Then we have
    \begin{align}
    &\mathcal{R}^{\mathrm{ICL}}_{\mu}(\tilde{f}_T) \leq 4R\sqrt{ \bigl\{\chi^2(\mu,\pi)+1\bigr\} \mathcal{E}(\tilde{f}_T)} + 2\mathbb{E}_\mu \mathbb{E}_P \bigl[\bigl\{g_\pi(X,\mathcal{D}_n) - \mathbb{E}_P(Y\,|\,X)\bigr\}^2\bigr]. \label{eq:ICL-excess-risk-ub}
    \end{align}
\end{prop} 
% \red{We can apply Markov's inequality to show that the risk of a randomly drawn task is small with high probability, i.e.~we can show that the pointwise risk is optimal with high probability.}
The upper bound in~\eqref{eq:ICL-excess-risk-ub} consists of two terms. The first involves both the extent of the distribution shift measured by the $\chi^2$-divergence, as well as the expected difference between the $\pi$-risk of~$\tilde{f}_T$ and that of the posterior regression function~$g_\pi$. It is well-known that $g_\pi$ minimises the $\pi$-risk~$\mathcal{R}_{\pi}$ over all measurable functions, so intuitively, if~$\tilde{f}_T$ is a large Transformer pretrained on a sufficient corpus of data, then $\mathcal{E}(\tilde{f}_T)$ should be negligible.  Indeed, this is proved in our Proposition~\ref{prop:ERM-approx-posterior}. The second term reflects the proximity of the posterior regression function to the oracle regression function. In fact, in Sections~\ref{sec:holder-regression} and~\ref{sec:multi-index-regression}, we will control $\mathbb{E}_P \bigl[\bigl\{g_\pi(X,\mathcal{D}_n) - \mathbb{E}_P(Y\,|\,X)\bigr\}^2\bigr]$ uniformly for all $P$ in the support of $\mu$. This reveals the interesting feature that only the first term on the right-hand side of~\eqref{eq:ICL-excess-risk-ub} depends on the degree of distribution shift, and this term can be made arbitrarily small by increasing the size of the Transformer and the pretraining dataset, so that the second term ultimately determines the rate of convergence. 

\citet{bai2023transformers} showed that Transformers can approximate algorithm selection based on a train-validation split.  Their arguments  proceeded via a different risk decomposition to that in Proposition~\ref{prop:R_mu-decomposition}, where Transformers are regarded as approximators of the oracle regression functions $\mathbb{E}_P(Y\,|\,X)$; see also \citet{kim2024transformers}.  By contrast, in Proposition~\ref{prop:R_mu-decomposition} we view Transformers as approximators of the posterior regression function, which has three key advantages.  First, we are able to control the ICL excess risk under distribution shift between the pretraining prior and the test distribution.  Second, we are able to demonstrate adaptivity of Transformers to the difficulty of the test task, instead of obtaining a rate of convergence that is a weighted average over difficulty levels (typically dominated by the difficult tasks).  Finally, by approximating the posterior regression function, Proposition~\ref{prop:R_mu-decomposition} can be applied to settings where no simple closed form algorithms are available or optimal, thereby (in combination with Theorem~\ref{thm:universal-approximation} below) explaining the empirically-observed superiority of Transformer-based ICL compared to classical algorithms for nonparametric regression \citep[e.g.][]{hollmann2025accurate}.

% We remark that Proposition~\ref{prop:R_mu-decomposition} is different from previous risk decompositions in the literature \citep[e.g.][]{bai2023transformers,kim2024transformers}, as we view Transformers as approximators of the Bayes regression function rather than the oracle regression functions $\mathbb{E}_P(Y\,|\,X)$. 
% In contrast, although \citet{bai2023transformers} showed that Transformers can approximate algorithm selection based on a train-validation split, their ICL excess risk is only controlled when the test distribution is the same as the pretraining prior. This leads to the risk being an average over all the difficulty levels \citep[Theorem~12]{bai2023transformers}, so the rate of convergence is dominated by the difficult tasks, thereby losing adaptivity and distributional robustness. Finally, by approximating the posterior regression function, our Proposition~\ref{prop:R_mu-decomposition} can be applied to setting where no simple closed form algorithms are available or optimal, therefore explaining the superiority of Transformer-based ICL compared to classical statistical algorithms for nonparametric regression.
%their ICL excess risk is only controlled under the assumption that the test distribution is the same as the pretraining prior. This leads to the risk being an average over all the difficulty levels \citep[Theorem~12]{bai2023transformers}, and the rate of convergence is dominated by the difficult tasks, thereby loosing adaptivity and distributional robustness features.

\section{Transformers, universal approximation and learnability} \label{sec:Transformers, Universal Approximation and Learnability}
\subsection{Transformers} \label{sec:transformers}
In this section, we set up basic notation for Transformer models that we consider.
Given $d_{\mathrm{model}}\geq d+2$, called the \emph{model dimension}, examples $D_n = (x_i,y_i)_{i=1}^n \in (\mathcal{X} \times \mathcal{Y})^n$ and a query $x \in \mathcal{X}$, the input matrix $Z_{\mathrm{in}} \in \mathbb{R}^{(n+1) \times d_{\mathrm{model}}}$ is defined as 
\begin{align}
    Z_{\mathrm{in}} \equiv Z_{\mathrm{in}}(D_n,x) \coloneqq \begin{pmatrix}
        x_1^\top & y_1 & 0_{d_{\mathrm{model}}-d-2}^\top & 0\\
        \vdots & \vdots & \vdots & \vdots\\
        x_n^\top & y_n & 0_{d_{\mathrm{model}}-d-2}^\top & 0\\
        x^\top & 0 & 0_{d_{\mathrm{model}}-d-2}^\top & 1
    \end{pmatrix} \label{eq:input-matrix}
\end{align}
where the last column of $Z_{\mathrm{in}}$ is our positional encoding used to identify the query, and we pad each row with zeros to make them $d_{\mathrm{model}}$-dimensional. 

Let $N\in\mathbb{N}$ denote an upper bound on the context length. Our universal approximation theory (Theorem~\ref{thm:universal-approximation}) holds uniformly for all $n\in[N]$, so it is convenient to consider Transformers as functions from\footnote{The space $\bigcup_{n=1}^{N}\mathbb{R}^{(n+1) \times d_{\mathrm{model}}}$ is equipped with the disjoint union topology and its corresponding Borel $\sigma$-algebra.} $\bigcup_{n=1}^{N}\mathbb{R}^{(n+1) \times d_{\mathrm{model}}}$ to $\bigcup_{n=1}^{N}\mathbb{R}^{(n+1) \times d_{\mathrm{model}}}$. We consider (encoder-only) Transformers with softmax attention, and the activation function in the FFN layers is denoted by~$\rho$. Let $\mathcal{F}_{\mathrm{TF}}(L, H, d_{\mathrm{model}}, d_{\mathrm{ffn}})$ denote the set of all Transformers with $L$ Transformer blocks, $H$ attention heads, model dimension $d_{\mathrm{model}}$ and FFN width $d_{\mathrm{ffn}}$; see Appendix~\ref{sec:transformer-defn} for precise definitions.

For a Transformer $\mathsf{TF}\in\mathcal{F}_{\mathrm{TF}}(L, H, d_{\mathrm{model}}, d_{\mathrm{ffn}})$, examples $D_n = (x_i,y_i)_{i=1}^n$ and a query~$x$, we will slightly abuse our notation by writing $\mathsf{TF}(x, D_n) \coloneqq \mathsf{TF}\bigl(Z_{\mathrm{in}}(D_n,x)\bigr)$ where the input matrix $Z_{\mathrm{in}}(D_n,x)$ is defined by~\eqref{eq:input-matrix}.  Finally, we define\footnote{More generally, $\mathsf{Read}(Z)$ can be an affine transformation of the last row of $Z$.} $\mathsf{Read}: \bigcup_{n=1}^{N}\mathbb{R}^{(n+1) \times d_{\mathrm{model}}} \to \mathbb{R}$ by $\mathsf{Read}(Z) \coloneqq Z_{n+1,d+1}$ for $Z \in\mathbb{R}^{(n+1)\times d_{\mathrm{model}}}$ and $n\in[N]$.  Thus, $\mathsf{Read}\circ\mathsf{TF}(x,D_n)$ can be interpreted as the predicted value of the response at the query $x$ based on examples $D_n$.

\subsection{Universal approximation and learnability}\label{sec:universal-approximation}
We now consider pretraining priors induced by randomly drawn regression functions, and prove universal approximation and learnability results for Transformers.
Let $\mathcal{G}$ be a space of real-valued functions on $\mathcal{X} \subseteq \mathbb{R}^d$ that are uniformly bounded by $R>0$ and let $\tilde{\pi}$ be a distribution on $\mathcal{G}$. Let $X$ be a random variable on $\mathcal{X}$ with distribution $P_X$, and let $\xi\indep X$ be a zero-mean random variable on $\mathbb{R}$. For $g\in\mathcal{G}$, let $P_g$ be the distribution of $(X,Y_g)$ where $Y_g=g(X)+\xi$. Further let $\pi$ be the distribution of the random measure~$P_{\tilde{g}}$, where the randomness comes from $\tilde{g} \sim \tilde{\pi}$.  Suppose throughout this subsection that there exist $R'>0$ and a ($\sigma$-finite) measure $\nu$ on $\mathbb{R}$ such that for all $g\in\mathcal{G}$ and almost all realisations of $g(X)$, the conditional distribution of $Y_g \,|\, g(X)$ is absolutely continuous with respect to $\nu$, with Radon–Nikodym derivative bounded by~$R'$. For example, if the noise $\xi$ has a $N(0,\sigma^2)$ distribution, then the conditional distribution of $Y_g \,|\, \{g(X)=z\}$ has bounded Lebesgue density given by $y \mapsto \frac{1}{\sqrt{2\pi\sigma^2}} e^{-(y-z)^2/(2\sigma^2)}$. 

Furthermore, assume that the activation function $\rho:\mathbb{R} \rightarrow \mathbb{R}$ in the FFN layers is Borel measurable and locally essentially bounded (i.e. for every compact set $A\subseteq \mathbb{R}$, there exists $C_A>0$ such that $|f(x)|\leq C_A$ for Lebesgue almost all $x\in A$), that the closure of the set of discontinuities of~$\rho$ has zero Lebesgue measure, and that there is no polynomial $h$ such that $\rho = h$ Lebesgue almost everywhere. The following theorem shows that Transformers are universal approximators for the posterior regression function~$g_{\pi}$ given by~\eqref{eq:def-posterior-regression-function}.

\begin{thm} \label{thm:universal-approximation}
    For any $L\geq 3$, $H\geq 1$, $\epsilon>0$ and $N\in\mathbb{N}$, there exist $d_{\mathrm{model}}\geq d+2$, $d_{\mathrm{ffn}}\in\mathbb{N}$ and a Transformer $\mathsf{TF} \in \mathcal{F}_{\mathrm{TF}}(L,H,d_{\mathrm{model}},d_{\mathrm{ffn}})$ such that 
    \begin{align*}
        \max_{n\in[N]}\mathbb{E}_{\pi}\mathbb{E}_P\Bigl\{\bigl(\mathsf{Read}\circ\mathsf{TF}(X,\mathcal{D}_n) - g_\pi(X,\mathcal{D}_n)\bigr)^2\Bigr\} \leq \epsilon.
    \end{align*}
\end{thm}
In Appendix~\ref{sec:proof-of-universal-approximation}, we in fact show that there exists a Transformer with $L = 3$ and $H = 1$ such that Theorem~\ref{thm:universal-approximation} holds. In practice, deeper Transformers will have more expressive power.

From now on, we assume that the activation function $\rho$ in the FFN layers belongs to the set $\{\mathsf{ReLU}, \mathsf{GELU}, \mathsf{SiLU}\}$, which are commonly used activation functions in modern Transformer models (see~\eqref{eq:activations} for their definitions), while the attention layers use softmax activation. For $R>0$, define $\mathsf{clip}_R(x) \coloneqq -R\vee x \wedge R$ for $x\in\mathbb{R}$.
The following proposition shows that, when the noise $\xi$ is sub-Gaussian, a large Transformer pretrained with sufficient data can achieve a $\pi$-risk that is arbitrarily close to the $\pi$-risk of~$g_{\pi}$.  Note that we do not require the input matrix or the parameters of the Transformer to be bounded.  We achieve this by showing that the pseudo-dimension of the class is finite using results from a branch of mathematical logic called model theory, see Appendix~\ref{sec:proof-of-learnability}.
\begin{prop}\label{prop:ERM-approx-posterior}
    Suppose that the noise $\xi$ is sub-Gaussian. For any $L\geq 3$, $H\geq 1$, $n\in\mathbb{N}$ and $\epsilon>0$, there exist $d_{\mathrm{model}}^\circ \geq d+2$ and $d_{\mathrm{ffn}}^\circ\in\mathbb{N}$ such that the following holds.  Suppose $d_{\mathrm{model}}\geq d_{\mathrm{model}}^\circ$, $d_{\mathrm{ffn}}\geq d_{\mathrm{ffn}}^\circ$, and $\hat{f}_T$ is defined as in~\eqref{eq:f-hat} with $\mathcal{F} \coloneqq \mathsf{Read} \circ \mathcal{F}_{\mathrm{TF}}(L,H,d_{\mathrm{model}},d_{\mathrm{ffn}})$ and pretraining data generated according to~$\pi$. Let $\tilde{f}_T \coloneqq \mathsf{clip}_{R}\circ\hat{f}_T$. Then, for all $T$ sufficiently large,
    \begin{align*}
        \mathcal{E}(\tilde{f}_T) \coloneqq \mathbb{E}\mathcal{R}_{\pi}(\tilde{f}_T) - \mathcal{R}_{\pi}(g_{\pi}) \leq \epsilon.
    \end{align*} 
\end{prop}
This proposition confirms formally the intuition provided after Proposition~\ref{prop:R_mu-decomposition} that the first term in the decomposition~\eqref{eq:ICL-excess-risk-ub} can be made negligible.  In other words, Transformers can learn the posterior regression function from data, without knowing the pre-training prior.

\section{Optimal in-context adaptivity and distributional robustness} \label{sec:adaptive-icl}

In this section, we consider mixture distribution pretraining priors $\pi=\sum_{\alpha\in\mathcal{A}} \lambda_{\alpha}\pi_{\alpha}$ for some finite index set $\mathcal{A}$, and test distributions $\mu$ on $\mathcal{P}$ with $\chi^2(\mu,\pi_{\beta}) \leq \kappa$ for some $\beta\in\mathcal{A}$, so that draws from $\mu$ contain tasks of difficulty $\beta$ with potential distribution shift relative to~$\pi_\beta$. We have already seen that Proposition~\ref{prop:R_mu-decomposition} provides a decomposition of the ICL excess $\mu$-risk, and Proposition~\ref{prop:ERM-approx-posterior} shows that the first term on the right-hand side of~\eqref{eq:ICL-excess-risk-ub} can be made negligible.  Our goal here, then, is to control the key quantity in the other term, namely 
$\mathbb{E}_P \bigl[\bigl\{g_\pi(X,\mathcal{D}_n) - \mathbb{E}_P(Y\,|\,X)\bigr\}^2\bigr]$, uniformly over all~$P$ in the support of $\pi_{\beta}$.  To this end, we specialise the general posterior contraction theory for nonparametric regression that we provide in Appendix~\ref{sec:bayesian-nonparametric-regression} to Besov regression functions with random smoothness (in Section~\ref{sec:holder-regression}) and multi-index models (in Section~\ref{sec:multi-index-regression}). All proofs are deferred to Appendix~\ref{sec:proofs-of-main-results}.  %In the latter setting, the reduced complexity of the regression function class allows us to avoid the curse of dimensionality.

\subsection{In-context learning for Besov functions} \label{sec:holder-regression}

For this subsection, we take\footnote{Our results in this section extend to any compact domain $\mathcal{X} \subseteq \mathbb{R}^d$ by a scaling argument.} $\mathcal{X}\coloneqq [0,1]^d$, $\mathcal{Y}\coloneqq \mathbb{R}$. To capture the smoothness of regression functions, we consider Besov spaces (see Appendix~\ref{sec:function-spaces}), which capture a broad range of smoothness behaviours and are among the most widely used function classes for studying optimal rates in nonparametric regression; see, e.g., \citet{donoho1998minimax}, \citet{suzuki2018adaptivity} and \citet{gine2021mathematical}. 

To keep the main text focused on our core statistical message, we defer the precise construction of our random Besov functions to
Appendix~\ref{sec:defn-random-besov}. Informally, for each smoothness level $\alpha>0$, we construct a distribution $\tilde{\pi}_\alpha$ on a class
$\mathcal G_\alpha$ of regression functions with Besov
smoothness $\alpha$. The construction is based on random wavelet coefficients chosen so that, for some $C_0>0$,
\[
    B^\alpha_{\infty,\infty}([0,1]^d,C_0)
    \subseteq \mathcal G_\alpha
    \subseteq
    B^\alpha_{\infty,\infty}([0,1]^d,2C_0).
\]

Given $g:[0,1]^d \to \mathbb{R}$, let $P_{g}$ denote the distribution of $(X,Y_g)$ where $X\sim P_X$ satisfies Assumption~\ref{assumption:P-X-besov},
\begin{align*}
    Y_g=g(X)+\xi, \quad \xi\sim N(0,\sigma^2) \quad\text{and}\quad \xi \indep X.
\end{align*}
We then define~$\pi_{\alpha}$ to be the distribution of the random measure $P_{\tilde{g}_{\alpha}}$, where the randomness comes from the random regression function $\tilde{g}_{\alpha} \sim \tilde{\pi}_{\alpha}$, and let $\mathcal{P}_{\alpha}\coloneqq \mathrm{supp}(\pi_{\alpha})$. Hence $\pi_\alpha$
is a distribution over regression tasks whose regression functions have
smoothness $\alpha$.

Let $\mathcal{A} \subseteq (0,S)$ be a finite set, where $S>0$ is the regularity of the wavelet basis in Appendix~\ref{sec:defn-random-besov}, and let $(\lambda_{\alpha})_{\alpha\in\mathcal{A}}$ be a sequence of positive weights such that $\sum_{\alpha\in\mathcal{A}}\lambda_{\alpha}=1$. Our pretraining prior~$\pi$ is defined to be the mixture distribution
\begin{align}
    \pi \coloneqq \sum_{\alpha\in\mathcal{A}}\lambda_{\alpha} \pi_{\alpha}. \label{eq:mixture-training-prior-distribution}
\end{align}
Thus $\pi$ consists of random regression tasks where the smoothness is randomly drawn from $\mathcal{A}$ with probabilities given by $(\lambda_{\alpha})_{\alpha\in\mathcal{A}}$.
Finally, we define $\mathcal{P}\coloneqq \bigcup_{\alpha\in\mathcal{A}} \mathcal{P}_{\alpha}$ and $\mathcal{G}\coloneqq \bigcup_{\alpha\in\mathcal{A}} \mathcal{G}_{\alpha}$. 

The following proposition provides the rate of convergence of the posterior regression function, when the true regression function belongs to $\mathcal{G}_{\beta}$ for some smoothness level $\beta\in\mathcal{A}$.
\begin{prop} \label{prop:posterior-regression-function-wavelet-prior}
    For any $\beta\in\mathcal{A}$ and $g^\circ \in\mathcal{G}_{\beta}$, write $P_0 \equiv P_{g^\circ}$. Let $\pi$ be the prior distribution defined in~\eqref{eq:mixture-training-prior-distribution} and let $g_{\pi}$ denote the posterior regression function with respect to~$\pi$. Then there exists $C>0$, not depending on $n$, such that for all $n \in \mathbb{N}$, we have
    \begin{align*}
    \mathbb{E}_{P_0} \bigl[\bigl\{g_\pi(X,\mathcal{D}_n) - g^\circ(X)\bigr\}^2\bigr] \leq Cn^{-2\beta/(2\beta+d)}.
    \end{align*}
\end{prop}
Combining Propositions~\ref{prop:R_mu-decomposition},~\ref{prop:ERM-approx-posterior} and~\ref{prop:posterior-regression-function-wavelet-prior}, we are now ready to state the main theorem of this section. 
\begin{thm} \label{thm:holder-adaptive-icl}
    For any $L\geq 3$, $H\geq 1$, $n\in\mathbb{N}$ and $\kappa>0$, 
    % $\beta\in\mathcal{A}$ and let $\mu$ be a distribution on $\mathcal{P}_\beta$ such that $\mu \ll \pi_{\beta}$ and $\sup_{P\in\mathcal{P}_\beta} \frac{\d\mu}{\d\pi_\beta}(P) \leq \kappa$ for some $\kappa>0$. 
    there exist $d_{\mathrm{model}}^\circ\geq d+2$ and $d_{\mathrm{ffn}}^\circ\in\mathbb{N}$ such that the following holds.  Suppose $d_{\mathrm{model}}\geq d_{\mathrm{model}}^\circ$, $d_{\mathrm{ffn}} \geq d_{\mathrm{ffn}}^\circ$, $\hat{f}_T$ is defined as in~\eqref{eq:f-hat} with $\mathcal{F} \coloneqq \mathsf{Read} \circ \mathcal{F}_{\mathrm{TF}}(L,H,d_{\mathrm{model}},d_{\mathrm{ffn}})$ and pretraining prior $\pi$ defined in~\eqref{eq:mixture-training-prior-distribution}. Let $R \coloneqq \sup_{g\in\mathcal{G}}\|g\|_{\infty}$ and $\tilde{f}_T \coloneqq \mathsf{clip}_{R}\circ\hat{f}_T$. Then, for all $T$ sufficiently large and $\beta\in\mathcal{A}$,
    \begin{align*}
        \sup_{\mu \,:\, \chi^2(\mu,\pi_{\beta})\leq\kappa} \mathcal{R}^{\mathrm{ICL}}_{\mu}(\tilde{f}_T) \leq Cn^{-2\beta/(2\beta+d)},
    \end{align*}
    where $C>0$ does not depend on $n$ and $\kappa$.
\end{thm}

The Transformer $\tilde{f}_T$ in Theorem~\ref{thm:holder-adaptive-icl} is pretrained on the mixture distribution $\pi = \sum_{\alpha\in\mathcal{A}} \lambda_{\alpha}\pi_{\alpha}$ with different smoothness levels, whereas the test distribution $\mu$ consists of random regression functions of a fixed smoothness level~$\beta$, with a distribution shift such that $\chi^2(\mu,\pi_{\beta})\leq\kappa$. When we evaluate the performance of $\tilde{f}_T$ under the test distribution $\mu$, Theorem~\ref{thm:holder-adaptive-icl} shows that the ICL excess $\mu$-risk of $\tilde{f}_T$ is adaptive to the unknown smoothness~$\beta$, in the sense that its rate of convergence $n^{-2\beta/(2\beta+d)}$ is optimal even if $\beta$ were known; see Theorem~\ref{thm:bayes-risk-lb}.  Moreover, this risk bound holds uniformly over test distributions in a $\chi^2$-divergence ball around $\pi_\beta$.  Thus, although the Transformer is pretrained on tasks of different levels of difficulties, it is able to achieve faster (and optimal) rates of convergence on easier tasks (without knowing the difficulty of the tasks) and is robust to distribution shift. 
Finally, we remark that  $d_{\mathrm{model}}^\circ,d_{\mathrm{ffn}}^\circ$ and $T$ should increase with $\kappa$ in order to keep~$C$ independent of $\kappa$. Thus, for a Transformer with fixed architecture, fine tuning for downstream tasks may still be beneficial in cases where (i) the Transformer is not large enough, (ii) the amount of pretraining data is not sufficient, (iii) there is significant distribution shift ($\kappa$ is large), or (iv) there is shift in the support of the pretraining prior and the test distribution.

\subsection{In-context learning for multi-index models} \label{sec:multi-index-regression}

To avoid the curse of dimensionality in the previous section, we will now consider multi-index models with\footnote{Again, our results in this section extend to any compact domain $\mathcal{X} \subseteq \mathbb{R}^d$ by a scaling argument.} $\mathcal{X}=\mathbb{B}^d \coloneqq \{x\in\mathbb{R}^d : \|x\|_2 \leq 1\}$ and $\mathcal{Y}=\mathbb{R}$. In general, a multi-index model takes the form $x \mapsto g\bigl(A(x)\bigr)$ where $A:\mathbb{R}^d \to \mathbb{R}^p$ is an affine function, $p\in[d]$, and $g:\mathbb{R}^p \to \mathbb{R}$ has smoothness~$\alpha$. Here, $p$ is called the effective dimension and $\alpha$ is the smoothness of the multi-index model. %When $p<d$, we would expect a faster rate of convergence for the multi-index model.

The formal definition of our random multi-index model is stated in Appendix~\ref{sec:defn-random-multi-index}. Informally, we define $\mathcal{G}_{\alpha,p}$ to be a set of multi-index models with smoothness~$\alpha$ and effective dimension~$p$, and let\footnote{Note that we index difficulty by the pair of parameters $(\alpha,p)$ in this subsection.} $\pi_{\alpha,p}$ be a distribution that generates random tasks with random regression functions drawn from~$\mathcal{G}_{\alpha,p}$. For a finite set $\mathcal{A}' \subseteq (0,S)\times[d]$ and a sequence of positive weights $(\lambda_{\alpha,p})_{(\alpha,p)\in\mathcal{A}'}$ such that $\sum_{(\alpha,p)\in\mathcal{A}'}\lambda_{\alpha,p}=1$, our pretraining prior~$\pi$ for multi-index models is defined to be the mixture distribution
\begin{align}
    \pi \coloneqq \sum_{(\alpha,p)\in\mathcal{A}'}\lambda_{\alpha,p} \pi_{\alpha,p}. \label{eq:mixture-training-prior-distribution-multi-index-model}
\end{align}
Thus $\pi$ generates random multi-index regression tasks, where the smoothness and effective dimensions of the regression functions are randomly drawn from $\mathcal{A}'$ with probabilities given by $(\lambda_{\alpha,p})_{(\alpha,p)\in\mathcal{A}'}$.
Finally, we define $\mathcal{P}\coloneqq \bigcup_{(\alpha,p)\in\mathcal{A}'} \mathcal{P}_{\alpha,p}$ and $\mathcal{G}\coloneqq \bigcup_{(\alpha,p)\in\mathcal{A}'} \mathcal{G}_{\alpha,p}$. Proposition~\ref{prop:posterior-regression-function-multi-index}, stated and proved in Appendix~\ref{sec:posterior-contraction-multi-index}, is an analogous result to Proposition~\ref{prop:posterior-regression-function-wavelet-prior} but for multi-index models.  Combining Propositions~\ref{prop:R_mu-decomposition},~\ref{prop:ERM-approx-posterior} and~\ref{prop:posterior-regression-function-multi-index}, we can now conclude that a pretrained Transformer simultaneously adapts to both the effective dimension of the multi-index models and the smoothness of the regression function, and is robust to distribution shift.

\begin{thm} \label{thm:multi-index-adaptive-icl}
    For any $L\geq 3$, $H\geq 1$, $n\in\mathbb{N}$ and $\kappa>0$, there exist $d_{\mathrm{model}}^\circ\geq d+2$ and $d_{\mathrm{ffn}}^\circ\in\mathbb{N}$ such that the following holds. Suppose $d_{\mathrm{model}}\geq d_{\mathrm{model}}^\circ$, $d_{\mathrm{ffn}} \geq d_{\mathrm{ffn}}^\circ$, $\hat{f}_T$ is defined as in~\eqref{eq:f-hat} with $\mathcal{F} \coloneqq \mathsf{Read} \circ \mathcal{F}_{\mathrm{TF}}(L,H,d_{\mathrm{model}},d_{\mathrm{ffn}})$ and pretraining prior $\pi$ defined in~\eqref{eq:mixture-training-prior-distribution-multi-index-model}. Let $R \coloneqq \sup_{g\in\mathcal{G}}\|g\|_{\infty}$ and $\tilde{f}_T \coloneqq \mathsf{clip}_{R}\circ\hat{f}_T$. Then, for all $T$ sufficiently large and $(\beta,r)\in\mathcal{A}'$,
    \begin{align*}
        \sup_{\mu \,:\, \chi^2(\mu, \pi_{\beta,r}) \leq \kappa} \mathcal{R}^{\mathrm{ICL}}_{\mu}(\tilde{f}_T) \leq Cn^{-2\beta/(2\beta+r)},
    \end{align*}
    where $C>0$ does not depend on $n$ and $\kappa$.
\end{thm}

The rate in Theorem~\ref{thm:multi-index-adaptive-icl} depends on the effective dimension $r$ rather than the ambient dimension $d$, thus avoiding the curse of dimensionality incurred in the setting of Theorem~\ref{thm:holder-adaptive-icl}.

\subsection{Lower bounds on the ICL excess risk} \label{sec:lower-bounds}

Our final theoretical contribution is to provide lower bounds on the ICL excess risk.  Prior work has sought to establish the optimality of ICL algorithms via lower bounds on the minimax risk
\begin{align}
    \inf_{\hat{g}_n \in \hat{\mathcal{G}}_n} \sup_{P\in \mathrm{supp}(\mu)} \mathbb{E}_P \bigl[ \bigl\{\hat{g}_n(X,\mathcal{D}_n) - \mathbb{E}_P(Y\,|\,X)\bigr\}^2 \bigr], \label{eq:minimax-risk}
\end{align}
where $\hat{\mathcal{G}}_n$ denotes the set of all measurable functions from $\mathcal{X} \times (\mathcal{X}\times\mathcal{Y})^n$ to $\mathcal{Y}$.  Although the minimax risk is often regarded as the gold standard in traditional statistical learning problems, the ICL excess $\mu$-risk $\mathcal{R}_{\mu}^{\mathrm{ICL}}$ is an expectation over $P\sim \mu$, not a worst-case risk over $P\in\mathrm{supp}(\mu)$. Thus lower bounds on~\eqref{eq:minimax-risk} do not provide lower bounds for the ICL excess $\mu$-risk.  Our Theorem~\ref{thm:bayes-risk-lb} below, however, does indeed provide a lower bound on the ICL excess $\mu$-risk, and in a strong sense: even an estimator that has knowledge of the test distribution~$\mu$ cannot achieve a better rate of convergence than the pretrained Transformers in Theorems~\ref{thm:holder-adaptive-icl} and~\ref{thm:multi-index-adaptive-icl}.
\begin{thm} \label{thm:bayes-risk-lb}
    (a) Under the setting of Section~\ref{sec:holder-regression}, let $\beta\in\mathcal{A}$ and let $\mu\ll\pi_{\beta}$ be such that $\sup_{P \in \mathcal{P}_\beta}\frac{\d\mu}{\d\pi_{\beta}}(P) < \infty$. Then there exists $c>0$, not depending on $n$, such that
    \begin{align*}
        \inf_{\hat{g}_n\in\hat{\mathcal{G}}_n} \mathcal{R}_{\mu}^{\mathrm{ICL}}(\hat{g}_n) \geq cn^{-2\beta/(2\beta+d)}.
    \end{align*}
    (b) In the setting of Section~\ref{sec:multi-index-regression}, let $(\beta,r)\in\mathcal{A}'$ and let $\mu\ll\pi_{\beta,r}$ be such that $\sup_{P\in\mathcal{P}_{\beta,r}}\frac{\d\mu}{\d\pi_{\beta,r}}(P) < \infty$. Then there exists $c'>0$, not depending on $n$, such that
    \begin{align*}
        \inf_{\hat{g}_n\in\hat{\mathcal{G}}_n} \mathcal{R}_{\mu}^{\mathrm{ICL}}(\hat{g}_n) \geq c'n^{-2\beta/(2\beta+r)}.
    \end{align*}
\end{thm}
We remark that, if~$\tilde{f}_T$ is defined as in Theorems~\ref{thm:holder-adaptive-icl} or~\ref{thm:multi-index-adaptive-icl}, then conditional on the pretraining data, we have $\tilde{f}_T \in \hat{\mathcal{G}}_n$. 

\section{Simulations}\label{sec:simulations}
The goal of this section is to verify our theoretical results by synthetic experiments.  
To this end, we consider single-index models with $n=30$, $d=5$, $p=1$, $\sigma=0.01$ and four levels of Besov smoothness $\alpha\in\{2,\, 2.5,\, 3,\, 4\}$. For our Transformer, we use a GPT-2 medium-style model \citep{radford2019language} with $L=24$, $H=16$, $d_{\mathrm{model}} = 1024$, $d_{\mathrm{ffn}}=4096$ and GELU activation function in the FFN layers. Further experimental details are provided in Appendix~\ref{sec:details-for-simulations}. The code for our simulations is available at \url{https://github.com/tianyima2000/Adaptive_ICL}.

At test time, we randomly generate new single-index models with fixed smoothness levels under two settings: (i) no distribution shift, and (ii) with distribution shift.
See Figure~\ref{fig:random-besov} for random Besov functions with and without distribution shift. Details for the data generating mechanism are provided in Appendix~\ref{sec:data-generating-mechanism}.
Figure~\ref{fig:MSE}(a) shows the test-time MSE of our pretrained Transformer for different levels of $\alpha$. We see that our Transformer is able to adapt to the smoothness and achieve a faster rate of convergence as $\alpha$ increases; moreover, our Transformer is also robust to distribution shift at test time. Figure~\ref{fig:MSE}(b) plots the logarithm of the MSE against $2\alpha/(2\alpha+1)$, revealing that the log MSE is approximately linear in $2\alpha/(2\alpha+1)$.  This is in line with the conclusion of Theorem~\ref{thm:multi-index-adaptive-icl}. We also provide additional empirical studies in Appendix~\ref{sec:further-simulations}, including further simulations on within-support and out-of-support distribution shifts, as well as simulations plotting the logarithm of the MSE against the logarithm of the number of pretraining sequences.

\begin{figure}[htbp]
    \centering
    \begin{subfigure}{0.45\linewidth}
        \includegraphics[width=\linewidth]{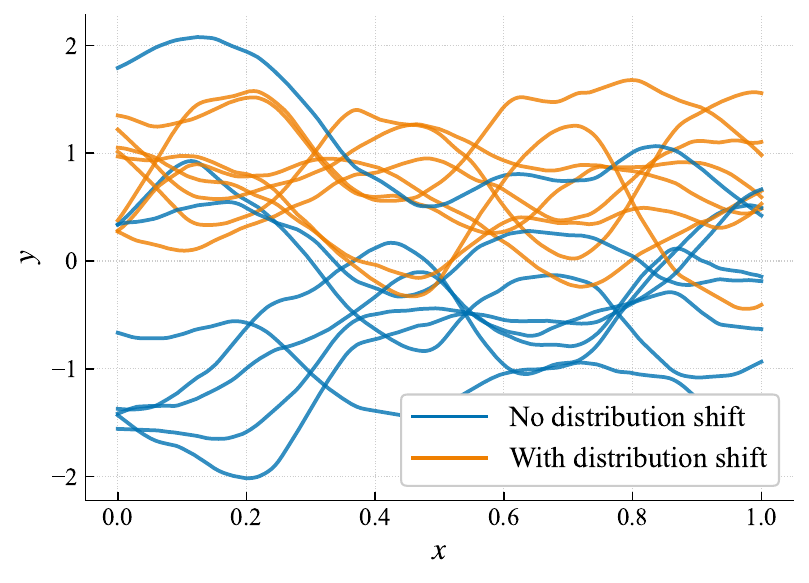}
        \caption{$\alpha=2$.}
    \end{subfigure}
    \begin{subfigure}{0.45\linewidth}
        \includegraphics[width=\linewidth]{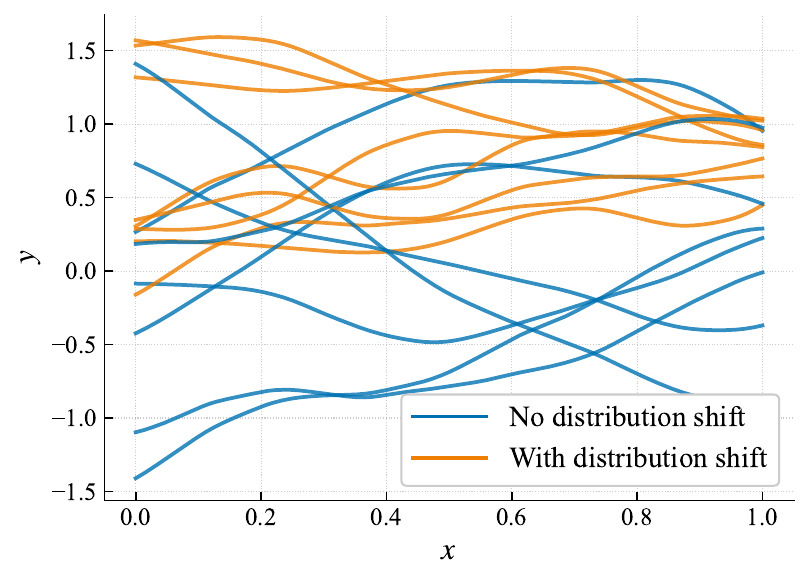}
        \caption{$\alpha=4$.}
    \end{subfigure}
    \caption{Random Besov functions with smoothness $\alpha\in\{2,4\}$.}
    \label{fig:random-besov}
\end{figure}

\begin{figure}[htbp]
    \centering
    \begin{subfigure}[t]{0.45\linewidth}
        \includegraphics[width=\linewidth]{figures/loss.pdf}
        \caption{MSE against $\alpha$.}
    \end{subfigure}
    \begin{subfigure}[t]{0.45\linewidth}
        \includegraphics[width=\linewidth]{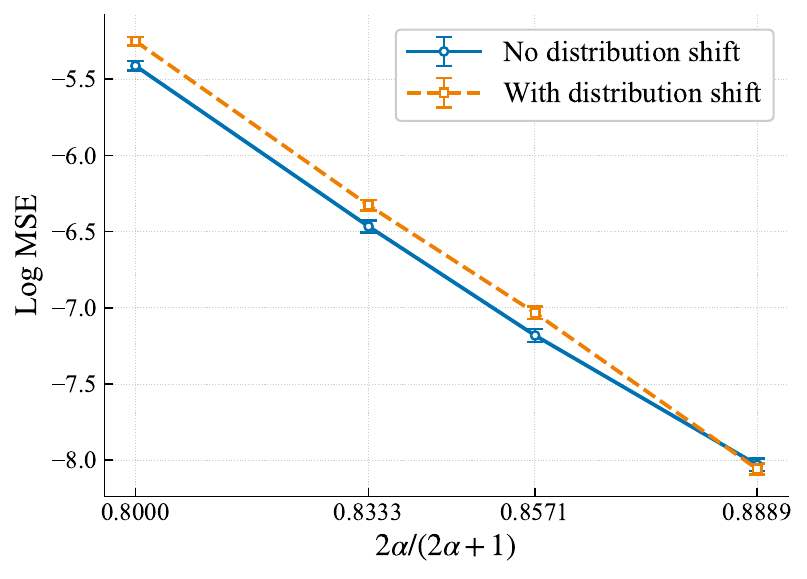}
        \caption{Log MSE against $\frac{2\alpha}{2\alpha+1}$.}
    \end{subfigure}
    \caption{MSE and log MSE at test time based on 100{,}000 random test functions.  Error bars represent 90\% confidence intervals.}
    \label{fig:MSE}
\end{figure}

\section{Conclusion}
In this work, we show that pretrained Transformers adapt to different task difficulties at optimal rates, and are robust to distribution shifts at test time. 
% This is proved by first showing that Transformers are universal approximators of posterior regression functions, and then establishing rates of convergence of the posterior regression functions on each task. 
Our results connect ICL with posterior contraction theory, and provide a rigorous explanation for the empirical success of ICL. 

A natural extension of this work would be to combine our theory with the quantitative approximation theory derived in independent, concurrent work by \citet{wakayama2025in} to provide more explicit values of $d_{\mathrm{model}}^\circ$, $d_{\mathrm{ffn}}^\circ$ and~$T$ that guarantee the conclusions of Theorems~\ref{thm:holder-adaptive-icl} and~\ref{thm:multi-index-adaptive-icl}.  More broadly, our work opens the door to providing adaptive and distributionally robust ICL guarantees for Transformers in other settings, including heavy-tailed noise, classification and dependent data.

\section*{Acknowledgement} 
The research of the first and third author was supported by RJS's European Research Council Advanced Grant 101019498.

\bibliography{bibliography}

\begin{thebibliography}{71}
\providecommand{\natexlab}[1]{#1}
\providecommand{\url}[1]{\texttt{#1}}
\providecommand{\urlprefix}{URL }
\providecommand{\eprint}[2][]{\url{#2}}

\bibitem[{Ahn et~al.(2023)Ahn, Cheng, Daneshmand and Sra}]{ahn2023transformers}
Ahn, K., Cheng, X., Daneshmand, H. and Sra, S. (2023) Transformers learn to implement preconditioned gradient descent for in-context learning. \emph{Advances in Neural Information Processing Systems}, \textbf{36}, 45614--45650.

\bibitem[{Ahuja and Lopez-Paz(2023)}]{ahuja2023a}
Ahuja, K. and Lopez-Paz, D. (2023) A Closer Look at In-Context Learning under Distribution Shifts. In \emph{Workshop on Efficient Systems for Foundation Models @ ICML2023}.

\bibitem[{Aky{\"u}rek et~al.(2023)Aky{\"u}rek, Schuurmans, Andreas, Ma and Zhou}]{akyurek2023what}
Aky{\"u}rek, E., Schuurmans, D., Andreas, J., Ma, T. and Zhou, D. (2023) What learning algorithm is in-context learning? Investigations with linear models. In \emph{The Eleventh International Conference on Learning Representations}.

\bibitem[{Anderson and Pitt(1989)}]{anderson1989probabilistic}
Anderson, J. and Pitt, L. (1989) {Probabilistic behaviour of functions in the Zygmund spaces $\Lambda^*$ and $\lambda^*$}. \emph{Proceedings of the London Mathematical Society}, \textbf{3}, 558--592.

\bibitem[{Bai et~al.(2023)Bai, Chen, Wang, Xiong and Mei}]{bai2023transformers}
Bai, Y., Chen, F., Wang, H., Xiong, C. and Mei, S. (2023) Transformers as statisticians: Provable in-context learning with in-context algorithm selection. \emph{Advances in Neural Information Processing Systems}, \textbf{36}, 57125--57211.

\bibitem[{Boucheron, Lugosi and Massart(2013)}]{boucheron2013concentration}
Boucheron, S., Lugosi, G. and Massart, P. (2013) \emph{{Concentration Inequalities: A Nonasymptotic Theory of Independence}}. Oxford University Press.

\bibitem[{Brown et~al.(2020)Brown, Mann, Ryder, Subbiah, Kaplan, Dhariwal, Neelakantan, Shyam, Sastry, Askell, Agarwal, Herbert-Voss, Krueger, Henighan, Child, Ramesh, Ziegler, Wu, Winter, Hesse, Chen, Sigler, Litwin, Gray, Chess, Clark, Berner, McCandlish, Radford, Sutskever and Amodei}]{brown2020language}
Brown, T., Mann, B., Ryder, N., Subbiah, M., Kaplan, J.~D., Dhariwal, P., Neelakantan, A., Shyam, P., Sastry, G., Askell, A., Agarwal, S., Herbert-Voss, A., Krueger, G., Henighan, T., Child, R., Ramesh, A., Ziegler, D., Wu, J., Winter, C., Hesse, C., Chen, M., Sigler, E., Litwin, M., Gray, S., Chess, B., Clark, J., Berner, C., McCandlish, S., Radford, A., Sutskever, I. and Amodei, D. (2020) Language models are few-shot learners. \emph{Advances in Neural Information Processing Systems}, \textbf{33}, 1877--1901.

\bibitem[{Chen, Guntuboyina and Zhang(2016)}]{chen2016bayes}
Chen, X., Guntuboyina, A. and Zhang, Y. (2016) {On Bayes risk lower bounds}. \emph{Journal of Machine Learning Research}, \textbf{17}, 1--58.

\bibitem[{Cheng, Chen and Sra(2024)}]{cheng2024transformers}
Cheng, X., Chen, Y. and Sra, S. (2024) Transformers implement functional gradient descent to learn non-linear functions in context. \emph{Proceedings of the 41st International Conference on Machine Learning}, 8002--8037.

\bibitem[{Chikuse(2003)}]{chikuse2003statistics}
Chikuse, Y. (2003) \emph{Statistics on Special Manifolds}. Springer Science \& Business Media.

\bibitem[{Cohen, Daubechies and Vial(1993)}]{cohen1993wavelets}
Cohen, A., Daubechies, I. and Vial, P. (1993) Wavelets on the interval and fast wavelet transforms. \emph{Applied and Computational Harmonic Analysis}, \textbf{1}, 54--81.

\bibitem[{Dahmen(1997)}]{dahmen1997wavelet}
Dahmen, W. (1997) Wavelet and multiscale methods for operator equations. \emph{Acta Numerica}, \textbf{6}, 55--228.

\bibitem[{Dai et~al.(2023)Dai, Sun, Dong, Hao, Ma, Sui and Wei}]{dai2023gpt}
Dai, D., Sun, Y., Dong, L., Hao, Y., Ma, S., Sui, Z. and Wei, F. (2023) Why can {GPT} learn in-context? Language models secretly perform gradient descent as meta-optimizers. \emph{Findings of the Association for Computational Linguistics: ACL 2023}, 4005--4019.

\bibitem[{Daubechies(1988)}]{daubechies1988orthonormal}
Daubechies, I. (1988) Orthonormal bases of compactly supported wavelets. \emph{Communications on Pure and Applied Mathematics}, \textbf{41}, 909--996.

\bibitem[{Demir and Dogan(2025)}]{demir2025optimal}
Demir, S. and Dogan, Z. (2025) Optimal Attention Temperature Enhances In-Context Learning under Distribution Shift. \emph{arXiv preprint arXiv:2511.01292}.

\bibitem[{Deora et~al.(2025)Deora, Vasudeva, Behnia and Thrampoulidis}]{deora2025incontext}
Deora, P., Vasudeva, B., Behnia, T. and Thrampoulidis, C. (2025) In-context {O}ccam's Razor: How Transformers prefer simpler hypotheses on the fly. In \emph{Second Conference on Language Modeling}.

\bibitem[{Devlin et~al.(2019)Devlin, Chang, Lee and Toutanova}]{devlin2019bert}
Devlin, J., Chang, M.-W., Lee, K. and Toutanova, K. (2019) Bert: Pre-training of deep bidirectional transformers for language understanding. In \emph{Proceedings of the 2019 Conference of the North {A}merican Chapter of the Association for Computational Linguistics: Human Language Technologies, Volume 1 (Long and Short Papers)}, 4171--4186.

\bibitem[{Donoho and Johnstone(1998)}]{donoho1998minimax}
Donoho, D.~L. and Johnstone, I.~M. (1998) Minimax estimation via wavelet shrinkage. \emph{The annals of Statistics}, \textbf{26}, 879--921.

\bibitem[{Dosovitskiy et~al.(2021)Dosovitskiy, Beyer, Kolesnikov, Weissenborn, Zhai, Unterthiner, Dehghani, Minderer, Heigold, Gelly, Uszkoreit and Houlsby}]{dosovitskiy2021an}
Dosovitskiy, A., Beyer, L., Kolesnikov, A., Weissenborn, D., Zhai, X., Unterthiner, T., Dehghani, M., Minderer, M., Heigold, G., Gelly, S., Uszkoreit, J. and Houlsby, N. (2021) An image is worth 16x16 words: Transformers for image recognition at scale. In \emph{International Conference on Learning Representations}.

\bibitem[{Egels and Castillo(2025)}]{egels2025posterior}
Egels, P. and Castillo, I. (2025) Posterior and variational inference for deep neural networks with heavy-tailed weights. \emph{Journal of Machine Learning Research}, \textbf{26}, 1--58.

\bibitem[{Finocchio and Schmidt-Hieber(2023)}]{finocchio2023posterior}
Finocchio, G. and Schmidt-Hieber, J. (2023) {Posterior contraction for deep Gaussian process priors}. \emph{Journal of Machine Learning Research}, \textbf{24}, 1--49.

\bibitem[{Folland(1999)}]{folland1999real}
Folland, G.~B. (1999) \emph{Real Analysis: Modern Techniques and Their Applications}. John Wiley \& Sons.

\bibitem[{Garg et~al.(2022)Garg, Tsipras, Liang and Valiant}]{garg2022can}
Garg, S., Tsipras, D., Liang, P.~S. and Valiant, G. (2022) {What can transformers learn in-context? A case study of simple function classes}. \emph{Advances in Neural Information Processing Systems}, \textbf{35}, 30583--30598.

\bibitem[{Ghosal, Ghosh and van~der Vaart(2000)}]{ghosal2000convergence}
Ghosal, S., Ghosh, J.~K. and van~der Vaart, A.~W. (2000) Convergence rates of posterior distributions. \emph{Annals of Statistics}, \textbf{28}, 500--531.

\bibitem[{Ghosal, Lember and van~der Vaart(2008)}]{ghosal2008nonparametric}
Ghosal, S., Lember, J. and van~der Vaart, A. (2008) {Nonparametric Bayesian model selection and averaging}. \emph{Electronic Journal of Statistics}, \textbf{2}, 63--89.

\bibitem[{Ghosal and van~der Vaart(2017)}]{ghosal2017fundamentals}
Ghosal, S. and van~der Vaart, A.~W. (2017) \emph{Fundamentals of Nonparametric Bayesian Inference}. Cambridge University Press.

\bibitem[{Gin{\'e} and Nickl(2011)}]{gine2011rates}
Gin{\'e}, E. and Nickl, R. (2011) Rates of contraction for posterior distributions in $L^r$-metrics, $1\leq r\leq\infty$. \emph{Annals of Statistics}, \textbf{39}, 2883--2911.

\bibitem[{Gin{\'e} and Nickl(2021)}]{gine2021mathematical}
Gin{\'e}, E. and Nickl, R. (2021) \emph{Mathematical Foundations of Infinite-dimensional Statistical Models}. Cambridge University Press.

\bibitem[{Goddard et~al.(2025)Goddard, Smith, Ngampruetikorn and Schwab}]{goddard2025when}
Goddard, C., Smith, L.~M., Ngampruetikorn, V. and Schwab, D.~J. (2025) When can in-context learning generalize out of task distribution? In \emph{Forty-second International Conference on Machine Learning}.

\bibitem[{Gy{\"o}rfi et~al.(2002)Gy{\"o}rfi, Kohler, Krzy{\.z}ak and Walk}]{gyorfi2002distribution}
Gy{\"o}rfi, L., Kohler, M., Krzy{\.z}ak, A. and Walk, H. (2002) \emph{A Distribution-free Theory of Nonparametric Regression}. Springer.

\bibitem[{Hataya, Matsui and Imaizumi(2024)}]{hataya2024automatic}
Hataya, R., Matsui, K. and Imaizumi, M. (2024) Automatic domain adaptation by transformers in in-context learning. In \emph{ICML 2024 Workshop on In-Context Learning}.

\bibitem[{Hollmann et~al.(2025)Hollmann, M{\"u}ller, Purucker, Krishnakumar, K{\"o}rfer, Hoo, Schirrmeister and Hutter}]{hollmann2025accurate}
Hollmann, N., M{\"u}ller, S., Purucker, L., Krishnakumar, A., K{\"o}rfer, M., Hoo, S.~B., Schirrmeister, R.~T. and Hutter, F. (2025) Accurate predictions on small data with a tabular foundation model. \emph{Nature}, \textbf{637}, 319--326.

\bibitem[{Jeon et~al.(2024)Jeon, Lee, Lei and Roy}]{jeon2024an}
Jeon, H.~J., Lee, J.~D., Lei, Q. and Roy, B.~V. (2024) An information-theoretic analysis of in-context learning. In \emph{Forty-first International Conference on Machine Learning}.

\bibitem[{Jumper et~al.(2021)Jumper, Evans, Pritzel, Green, Figurnov, Ronneberger, Tunyasuvunakool, Bates, {\v{Z}}{\'\i}dek, Potapenko, Bridgland, Meyer, Kohl, Ballard, Cowie, Romera-Paredes, Nikolov, Jain, Adler, Back, Petersen, Reiman, Clancy, Zielinski, Steinegger, Pacholska, Berghammer, Bodenstein, Silver, Vinyals, Senior, Kavukcuoglu, Kohli and Hassabis}]{jumper2021highly}
Jumper, J., Evans, R., Pritzel, A., Green, T., Figurnov, M., Ronneberger, O., Tunyasuvunakool, K., Bates, R., {\v{Z}}{\'\i}dek, A., Potapenko, A., Bridgland, A., Meyer, C., Kohl, S. A.~A., Ballard, A.~J., Cowie, A., Romera-Paredes, B., Nikolov, S., Jain, R., Adler, J., Back, T., Petersen, S., Reiman, D., Clancy, E., Zielinski, M., Steinegger, M., Pacholska, M., Berghammer, T., Bodenstein, S., Silver, D., Vinyals, O., Senior, A.~W., Kavukcuoglu, K., Kohli, P. and Hassabis, D. (2021) Highly accurate protein structure prediction with AlphaFold. \emph{Nature}, \textbf{596}, 583--589.

\bibitem[{Karpinski and Macintyre(1995)}]{karpinski1995polynomial}
Karpinski, M. and Macintyre, A. (1995) Polynomial bounds for VC dimension of sigmoidal neural networks. In \emph{Proceedings of the Twenty-seventh Annual ACM Symposium on Theory of Computing}, 200--208.

\bibitem[{Kim, Nakamaki and Suzuki(2024)}]{kim2024transformers}
Kim, J., Nakamaki, T. and Suzuki, T. (2024) Transformers are minimax optimal nonparametric in-context learners. \emph{Advances in Neural Information Processing Systems}, \textbf{37}, 106667--106713.

\bibitem[{Kirsch et~al.(2022)Kirsch, Harrison, Sohl-Dickstein and Metz}]{kirsch2022generalpurpose}
Kirsch, L., Harrison, J., Sohl-Dickstein, J. and Metz, L. (2022) General-purpose in-context learning by meta-learning Transformers. In \emph{Sixth Workshop on Meta-Learning at the Conference on Neural Information Processing Systems}.

\bibitem[{Kumano, Kera and Yamasaki(2025)}]{kumano2025adversarially}
Kumano, S., Kera, H. and Yamasaki, T. (2025) Adversarially pretrained Transformers may be universally robust in-context learners. \emph{arXiv preprint arXiv:2505.14042}.

\bibitem[{Kuwataka and Suzuki(2025)}]{kuwataka2025test}
Kuwataka, K. and Suzuki, T. (2025) Test time training enhances in-context learning of nonlinear functions. \emph{arXiv preprint arXiv:2509.25741}.

\bibitem[{Kwon et~al.(2026)Kwon, Xu, Yaras, Balzano and Qu}]{kwon2026outofdistribution}
Kwon, S.~M., Xu, A.~S., Yaras, C., Balzano, L. and Qu, Q. (2026) Out-of-Distribution Generalization of In-Context Learning: A Low-Dimensional Subspace Perspective. In \emph{The 29th International Conference on Artificial Intelligence and Statistics}.

\bibitem[{Laskowski(1992)}]{laskowski1992vapnik}
Laskowski, M.~C. (1992) {Vapnik--Chervonenkis classes of definable sets}. \emph{Journal of the London Mathematical Society}, \textbf{2}, 377--384.

\bibitem[{Leshno et~al.(1993)Leshno, Lin, Pinkus and Schocken}]{leshno1993multilayer}
Leshno, M., Lin, V.~Y., Pinkus, A. and Schocken, S. (1993) Multilayer feedforward networks with a nonpolynomial activation function can approximate any function. \emph{Neural Networks}, \textbf{6}, 861--867.

\bibitem[{Lin, Bai and Mei(2024)}]{lin2024transformers}
Lin, L., Bai, Y. and Mei, S. (2024) Transformers as Decision Makers: Provable in-context reinforcement learning via supervised pretraining. In \emph{The Twelfth International Conference on Learning Representations}.

\bibitem[{Liu et~al.(2021)Liu, Lin, Cao, Hu, Wei, Zhang, Lin and Guo}]{liu2021swin}
Liu, Z., Lin, Y., Cao, Y., Hu, H., Wei, Y., Zhang, Z., Lin, S. and Guo, B. (2021) Swin transformer: Hierarchical vision transformer using shifted windows. In \emph{Proceedings of the IEEE/CVF International Conference on Computer Vision}, 10012--10022.

\bibitem[{Loshchilov and Hutter(2019)}]{loshchilov2019decoupled}
Loshchilov, I. and Hutter, F. (2019) Decoupled Weight Decay Regularization. In \emph{International Conference on Learning Representations}.

\bibitem[{Ma, Wang and Samworth(2025)}]{ma2025deep}
Ma, T., Wang, T. and Samworth, R.~J. (2025) Deep learning with missing data. \emph{arXiv preprint arXiv:2504.15388}.

\bibitem[{Marker(2002)}]{marker2002model}
Marker, D. (2002) \emph{Model Theory: An Introduction}. Springer.

\bibitem[{Nagler(2023)}]{nagler2023statistical}
Nagler, T. (2023) Statistical Foundations of Prior-Data Fitted Networks. In \emph{Proceedings of the 40th International Conference on Machine Learning}, 25660--25676.

\bibitem[{Oko et~al.(2024)Oko, Song, Suzuki and Wu}]{oko2024pretrained}
Oko, K., Song, Y., Suzuki, T. and Wu, D. (2024) Pretrained transformer efficiently learns low-dimensional target functions in-context. \emph{Advances in Neural Information Processing Systems}, \textbf{37}, 77316--77365.

\bibitem[{Panwar, Ahuja and Goyal(2024)}]{panwar2024incontext}
Panwar, M., Ahuja, K. and Goyal, N. (2024) {In-context learning through the Bayesian prism}. In \emph{The Twelfth International Conference on Learning Representations}.

\bibitem[{Radford et~al.(2019)Radford, Wu, Child, Luan, Amodei, Sutskever et~al.}]{radford2019language}
Radford, A., Wu, J., Child, R., Luan, D., Amodei, D., Sutskever, I. et~al. (2019) Language models are unsupervised multitask learners. \emph{OpenAI blog}, \textbf{1}, 9.

\bibitem[{Rei{\ss} and Schmidt-Hieber(2020)}]{reiss2020posterior}
Rei{\ss}, M. and Schmidt-Hieber, J. (2020) Posterior contraction rates for support boundary recovery. \emph{Stochastic Processes and their Applications}, \textbf{130}, 6638--6656.

\bibitem[{Reuter et~al.(2025)Reuter, Rudner, Fortuin and R{\"u}gamer}]{reuter2025can}
Reuter, A., Rudner, T. G.~J., Fortuin, V. and R{\"u}gamer, D. (2025) {Can transformers learn full Bayesian inference in context?} In \emph{Forty-second International Conference on Machine Learning}.

\bibitem[{Rudin(1987)}]{rudin1987real}
Rudin, W. (1987) \emph{Real and Complex Analysis}. McGraw-Hill, Inc.

\bibitem[{Shen and Wasserman(2001)}]{shen2001rates}
Shen, X. and Wasserman, L. (2001) Rates of convergence of posterior distributions. \emph{Annals of Statistics}, \textbf{29}, 687--714.

\bibitem[{Speissegger(1999)}]{Speissegger1999pfaffian}
Speissegger, P. (1999) The Pfaffian closure of an o-minimal structure. \emph{Journal für die Reine und Angewandte Mathematik}, \textbf{508}, 189--211.

\bibitem[{Suzuki(2019)}]{suzuki2018adaptivity}
Suzuki, T. (2019) Adaptivity of deep Re{LU} network for learning in Besov and mixed smooth Besov spaces: optimal rate and curse of dimensionality. In \emph{International Conference on Learning Representations}.

\bibitem[{Triebel(1983)}]{triebel1983theory}
Triebel, H. (1983) \emph{Theory of Function Spaces}. Birkh\"auser Basel.

\bibitem[{van~der Vaart and van Zanten(2009)}]{Vaart2009Adaptive}
van~der Vaart, A.~W. and van Zanten, J.~H. (2009) {Adaptive Bayesian estimation using a Gaussian random field with inverse Gamma bandwidth}. \emph{Annals of Statistics}, \textbf{37}, 2655--2675.

\bibitem[{Vaswani et~al.(2017)Vaswani, Shazeer, Parmar, Uszkoreit, Jones, Gomez, Kaiser and Polosukhin}]{vaswani2017attention}
Vaswani, A., Shazeer, N., Parmar, N., Uszkoreit, J., Jones, L., Gomez, A.~N., Kaiser, {\L}. and Polosukhin, I. (2017) Attention is all you need. \emph{Advances in Neural Information Processing Systems}, \textbf{30}.

\bibitem[{von Oswald et~al.(2023)von Oswald, Niklasson, Randazzo, Sacramento, Mordvintsev, Zhmoginov and Vladymyrov}]{von2023transformers}
von Oswald, J., Niklasson, E., Randazzo, E., Sacramento, J., Mordvintsev, A., Zhmoginov, A. and Vladymyrov, M. (2023) Transformers learn in-context by gradient descent. In \emph{International Conference on Machine Learning}, 35151--35174.

\bibitem[{Wakayama and Suzuki(2026)}]{wakayama2025in}
Wakayama, T. and Suzuki, T. (2026) {In-context learning is provably Bayesian inference: a generalization theory for meta-learning}. \emph{International Conference on Machine Learning}.

\bibitem[{Wang et~al.(2025)Wang, Wang, Ying and Wang}]{wang2025can}
Wang, Q., Wang, Y., Ying, X. and Wang, Y. (2025) Can In-context Learning Really Generalize to Out-of-distribution Tasks? In \emph{The Thirteenth International Conference on Learning Representations}.

\bibitem[{Wang et~al.(2023)Wang, Zhu, Saxon, Steyvers and Wang}]{wang2023large}
Wang, X., Zhu, W., Saxon, M., Steyvers, M. and Wang, W.~Y. (2023) Large language models are latent variable models: Explaining and finding good demonstrations for in-context learning. \emph{Advances in Neural Information Processing Systems}, \textbf{36}, 15614--15638.

\bibitem[{Wu, Wang and Yao(2025)}]{wu2025why}
Wu, S., Wang, Y. and Yao, Q. (2025) Why in-context learning models are good few-shot learners? In \emph{The Thirteenth International Conference on Learning Representations}.

\bibitem[{Xie et~al.(2022)Xie, Raghunathan, Liang and Ma}]{xie2022an}
Xie, S.~M., Raghunathan, A., Liang, P. and Ma, T. (2022) {An explanation of in-context learning as implicit Bayesian inference}. In \emph{International Conference on Learning Representations}.

\bibitem[{Xiong et~al.(2020)Xiong, Yang, He, Zheng, Zheng, Xing, Zhang, Lan, Wang and Liu}]{xiong2020layer}
Xiong, R., Yang, Y., He, D., Zheng, K., Zheng, S., Xing, C., Zhang, H., Lan, Y., Wang, L. and Liu, T. (2020) On layer normalization in the transformer architecture. In \emph{International Conference on Machine Learning}, 10524--10533.

\bibitem[{Zhang, Frei and Bartlett(2024)}]{zhang2024trained}
Zhang, R., Frei, S. and Bartlett, P.~L. (2024) Trained transformers learn linear models in-context. \emph{Journal of Machine Learning Research}, \textbf{25}, 1--55.

\bibitem[{Zhang et~al.(2022)Zhang, Liu, Cai, Wang and Wang}]{zhang2022analysis}
Zhang, Y., Liu, B., Cai, Q., Wang, L. and Wang, Z. (2022) An analysis of attention via the lens of exchangeability and latent variable models. \emph{arXiv preprint arXiv:2212.14852}.

\bibitem[{Zhang et~al.(2025)Zhang, Zhang, Yang and Wang}]{zhang2025bayesianICL}
Zhang, Y., Zhang, F., Yang, Z. and Wang, Z. (2025) {What and how does in-context learning learn? Bayesian model averaging, parameterization, and generalization}. \emph{Proceedings of The 28th International Conference on Artificial Intelligence and Statistics}, 1684--1692.

\bibitem[{Zygmund(2002)}]{zygmund2002trigonometric}
Zygmund, A. (2002) \emph{Trigonometric Series}. Cambridge University Press.

\end{thebibliography}
\bibliographystyle{custom}

%%%%%%%%%%%%%%%%%%%%%%%%%%%%%%%%%%%%%%%%%%%%%%%%%%%%%%%%%%%%%%%%%%%%%%%%%%%%%%%
%%%%%%%%%%%%%%%%%%%%%%%%%%%%%%%%%%%%%%%%%%%%%%%%%%%%%%%%%%%%%%%%%%%%%%%%%%%%%%%
% APPENDIX
%%%%%%%%%%%%%%%%%%%%%%%%%%%%%%%%%%%%%%%%%%%%%%%%%%%%%%%%%%%%%%%%%%%%%%%%%%%%%%%
%%%%%%%%%%%%%%%%%%%%%%%%%%%%%%%%%%%%%%%%%%%%%%%%%%%%%%%%%%%%%%%%%%%%%%%%%%%%%%%
\newpage
\appendix
\onecolumn

\section{Notation}
\subsection{Notation for the main text} \label{sec:notation-main-text}
For $n\in\mathbb{N}$, we write $[n] \coloneqq \{1,\ldots,n\}$. For $m,n\in\mathbb{N}$ with $m<n$, we write $[m:n] \coloneqq \{m,m+1,\ldots,n\}$ and similarly $[n:\infty) \coloneqq \{n,n+1,\ldots\}$.  Given $a,b\in\mathbb{R}$, we write $a\wedge b \coloneqq \min\{a,b\}$ and $a\vee b \coloneqq \max\{a,b\}$. For $d\in\mathbb{N}$, we define $0_d \coloneqq (0,\ldots,0)^\top \in \mathbb{R}^d$, $1_d \coloneqq (1,\ldots,1)^\top \in \mathbb{R}^d$ and $\mathbb{B}^d \coloneqq \{x\in\mathbb{R}^d : \|x\|_2 \leq 1\}$.  For $R>0$, we define the clipping operator $\mathsf{clip}_R:\mathbb{R}\to[-R,R]$ by $\mathsf{clip}_R(x) \coloneqq -R\vee x \wedge R$ for $x\in\mathbb{R}$. Given a Borel measurable set $A\subseteq \mathbb{R}^d$, we write $\mathrm{Vol}_d(A)$ for the $d$-dimensional Lebesgue measure of $A$.  
For two probability measures $P$ and~$Q$ on a measurable space $(\mathcal{X},\Sigma)$, we say that $P$ is absolutely continuous with respect to $Q$, written as $P\ll Q$, if $P(A) = 0$ for all $A \in \Sigma$ such that $Q(A)=0$, and in that case write $\frac{\d P}{\d Q}$ for the Radon--Nikodym derivative.  When $P$ and $Q$ are probability measures, the $\chi^2$-divergence from $Q$ to $P$ is defined as $\chi^2(P,Q) \coloneqq \int_{\mathcal{X}} \bigl(\frac{\d P}{\d Q}\bigr)^2 \,\d Q - 1$ if $P\ll Q$, and $\chi^2(P,Q) \coloneqq \infty$ otherwise.
% \begin{align*}
%     \chi^2(P,Q) \coloneqq \begin{cases}
%         \int_{\mathcal{X}} \bigl(\frac{\d P}{\d Q}\bigr)^2 \,\d Q - 1 \quad&\text{if } P\ll Q\\
%         \infty &\text{otherwise}.
%     \end{cases}
% \end{align*}
Further, we write $P^{\otimes n}$ for the $n$-fold product measure of~$P$ and $\mathrm{supp}(P)$ for the support of~$P$, i.e. the intersection of all closed $A \in \Sigma$ with $P(A) = 1$. For a set $\mathcal{F}$ of functions from $\mathcal{X}$ to $\mathcal{Y}$ and a set $\mathcal{G}$ of functions from $\mathcal{Y}$ to $\mathcal{Z}$, we write $\mathcal{F}\circ\mathcal{G}\coloneqq \{f\circ g : f\in\mathcal{F},\, g\in\mathcal{G}\}$, where $f\circ g$ denotes the composition of $f$ and $g$. For a function $f:\mathcal{X} \to \mathbb{R}$, we write $\|f\|_{\infty} \coloneqq \sup_{x\in\mathcal{X}} |f(x)|$.  

\subsection{Notation for the Appendix}\label{sec:notation-appendix}
Let $(\mathcal{P},d)$ be a metric space, $\epsilon>0$ and $\mathcal{P}'\subseteq\mathcal{P}$, we write $\mathcal{N}(\epsilon,\mathcal{P}',d)$ for the minimum number of closed $\epsilon$-balls in $d$ needed to cover $\mathcal{P}'$. For $m,n\in\mathbb{N}$, let $I_n \in \mathbb{R}^{n\times n}$ be the identity matrix; for a matrix $A\in\mathbb{R}^{n\times m}$, define the operator norm of $A$ by $\|A\|_{\mathrm{op}} \coloneqq \sup_{x\in\mathbb{R}^m : \|x\|_2=1} \|Ax\|_2$ and define its entrywise $\ell_{\infty}$-norm by $\|A\|_{\infty} \coloneqq \max_{i\in[n],j\in[m]} |A_{ij}|$.

Let $P$ and $Q$ be probability measures on $\mathcal{Z}$, and let $\nu$ be a measure on $\mathcal{Z}$ such that $P\ll\nu$ and $Q\ll\nu$ (such a measure always exists, e.g.~$\nu=P+Q$). Write $p=\frac{\d P}{\d \nu}$ and $q=\frac{\d Q}{\d \nu}$, then the \emph{Hellinger distance} between $P$ and $Q$ is defined as
\begin{align*}
    d_{\mathrm{H}}(P,Q) \coloneqq \Bigl(\int_{\mathcal{Z}} \bigl(\sqrt{p(z)} - \sqrt{q(z)}\bigr)^2 \,\d\nu(z)\Bigr)^{1/2}.
\end{align*}
The definition is independent of the choice of the dominating measure $\nu$, and $d_{\mathrm{H}}$ is a metric on the space of distributions on $\mathcal{Z}$. The \emph{Kullback--Leibler divergence} from $Q$ to $P$ is defined as
\begin{align*}
    \mathrm{KL}(P,Q) \coloneqq \mathbb{E}_P\biggl(\log\frac{\d P}{\d Q}(Z)\biggr) = \int_{\mathcal{Z}} \log\frac{\d P}{\d Q}(z) \,\d P(z)
\end{align*}
if $P\ll Q$, and infinity otherwise; this is the $P$-expectation of the log-likelihood ratio between $P$ and $Q$. We further define
\begin{align*}
    \mathrm{V}_2(P,Q) \coloneqq \Var_P\biggl(\log\frac{\d P}{\d Q}(Z)\biggr) = \int_{\mathcal{Z}} \biggl(\log\frac{\d P}{\d Q}(z) - \mathrm{KL}(P,Q)\biggr)^2 \,\d P(z)
\end{align*}
if $P\ll Q$, and infinity otherwise; this is the $P$-variance of the log-likelihood ratio between $P$ and $Q$.

For $d\in\mathbb{N}$, $p\in[1,\infty)$, a Borel measurable function $f:\mathbb{R}^d \to\mathbb{R}$ and a Borel measure~$\mu$ on $\mathbb{R}^d$, we define
\[
\|f\|_{L^p(\mu)} \coloneqq \biggl(\int_{\mathbb{R}^d} |f(x)|^p \,\d \mu(x) \biggr)^{1/p}.
\]
and write $L^p(\mu)$ for the set of functions $f:\mathbb{R}^d \to\mathbb{R}$ with $\|f\|_{L^p(\mu)}<\infty$.
If the domain of~$f$ is restricted to $\mathcal{X}\subseteq\mathbb{R}^d$ and $\mu$ is the Lebesgue measure restricted to a Borel set $\mathcal{X}$, then we write $\|f\|_{L^p(\mathcal{X})} \equiv \|f\|_{L^p(\mu)}$ and $L^p(\mathcal{X}) \equiv L^p(\mu)$; when $\mathcal{X}$ is clear from context, we further write $\|f\|_{L^p} \equiv \|f\|_{L^p(\mathcal{X})}$.
% Moreover, for a Borel measurable function $f:\mathcal{X} \to\mathbb{R}$, we define
% \[
% \|f\|_{L^\infty} \coloneqq \inf\{\lambda \in [0,\infty]: |f| \leq \lambda \ \text{Lebesgue-a.e.}\} \eqqcolon \esssup_{x\in\mathcal{X}} |f(x)|. 
% \]

For a zero-mean random variable $X$ on $\mathbb{R}$, we say that $X$ is \emph{sub-Gamma in the right tail} with variance parameter $\sigma^2>0$ and scale parameter $c>0$ if for all $\lambda\in[0,1/c)$, we have
\begin{align*}
    \mathbb{E} e^{\lambda X} \leq \exp\biggl(\frac{\sigma^2\lambda^2}{2(1-c\lambda)}\biggr).
\end{align*}
Similarly, $X$ is \emph{sub-Gamma in the left tail} if $-X$ is sub-Gamma in the right tail. And we say $X$ is \emph{sub-Gamma in both tails} is it is both sub-Gamma in the right tail and sub-Gamma in the left tail.

Finally, we define three commonly used activation functions in Transformers. For $x\in\mathbb{R}$, define
\begin{gather}
    \mathsf{ReLU}(x) \coloneqq \max\{0,x\}, \quad \mathsf{GELU}(x) \coloneqq x\Phi(x) \quad\text{and}\quad \mathsf{SiLU}(x) \coloneqq \frac{x}{1+e^{-x}}, \label{eq:activations}
\end{gather}
where $\Phi(x) \coloneqq \int_{-\infty}^x \frac{1}{\sqrt{2\pi}} e^{-t^2/2} \,\d t$ is the distribution function of standard Gaussian distribution.

\section{Proof of Proposition~\ref{prop:R_mu-decomposition}} \label{sec:proof-of-risk-decomposition}
% Before starting the proof, we re-emphasise our notation. For a measure $\mu$ on $\mathcal{P}$ and a measurable function $g:\mathcal{X}\times(\mathcal{X}\times\mathcal{Y})^n \to \mathbb{R}$, we write
% \begin{align*}
%     \mathbb{E}_{P\sim\mu,\,(\mathcal{D}_n,X,Y)\sim P^{\otimes (n+1)}} \bigl[\bigl\{Y - g(X,\mathcal{D}_n)\bigr\}^2\bigr] &= \mathbb{E}_{P\sim\mu} \mathbb{E}_{(\mathcal{D}_n,X,Y)\sim P^{\otimes (n+1)}} \bigl[\bigl\{Y - g(X,\mathcal{D}_n)\bigr\}^2\bigr]\\
%     &= \mathbb{E}_{\mu} \mathbb{E}_{P}\bigl[\bigl\{Y - g(X,\mathcal{D}_n)\bigr\}^2\bigr],
% \end{align*}
% where the expectation is taken only over the randomness of $(\mathcal{D}_n,X,Y)$ distributed according to a random probability distribution $P$ drawn from $\mu$. If $g=\hat{f}_T$, then the above expectation is interpreted as the expectation conditional on $\hat{f}_T$, so it is random.
\begin{proof}[Proof of Proposition~\ref{prop:R_mu-decomposition}]
    We may assume without loss of generality that $\mu\ll\pi$ since otherwise, the right-hand side of~\eqref{eq:ICL-excess-risk-ub} is infinite.  Note that
    \begin{align*}
        \mathbb{E}\mathcal{R}_{\pi}(\tilde{f}_T) - \mathcal{R}_{\pi}(g_\pi) &= \mathbb{E}\, \mathbb{E}_\pi \mathbb{E}_P \bigl[  \bigl\{Y - \tilde{f}_T(X,\mathcal{D}_n)\bigr\}^2 - \bigl\{Y - g_\pi(X,\mathcal{D}_n)\bigr\}^2 \bigr]\\
        &= \mathbb{E}\, \mathbb{E}_\pi \mathbb{E}_P \bigl\{ -2Y\tilde{f}_T(X,\mathcal{D}_n) + \tilde{f}_T^2(X,\mathcal{D}_n) + 2Yg_\pi(X,\mathcal{D}_n) - g_\pi^2(X,\mathcal{D}_n) \bigr\},
    \end{align*}
    where the outer expectation is taken over the randomness of $\tilde{f}_T$.
    Moreover,
    \begin{align*}
        \mathbb{E}_\pi \mathbb{E}_P \bigl\{ Y\tilde{f}_T(X,\mathcal{D}_n)\bigr\} &= \mathbb{E}_\pi \mathbb{E}_P \bigl[\mathbb{E}_{P\sim\pi,\, (\mathcal{D}_n,X,Y)\sim P^{\otimes (n+1)}} \bigl\{Y\tilde{f}_T(X,\mathcal{D}_n)\,|\,X,\mathcal{D}_n\bigr\}\bigr]\\
        &= \mathbb{E}_\pi \mathbb{E}_P \bigl\{\tilde{f}_T(X,\mathcal{D}_n) \cdot \mathbb{E}_{P\sim\pi,\, (\mathcal{D}_n,X,Y)\sim P^{\otimes (n+1)}} \bigl(Y\,|\,X,\mathcal{D}_n\bigr)\bigr\}\\
        &= \mathbb{E}_\pi \mathbb{E}_P \bigl\{ \tilde{f}_T(X,\mathcal{D}_n) g_\pi(X,\mathcal{D}_n)\bigr\},
    \end{align*}
    and similarly
    \begin{align*}
        \mathbb{E}_\pi \mathbb{E}_P \bigl\{Yg_\pi(X,\mathcal{D}_n)\bigr\} &= \mathbb{E}_\pi \mathbb{E}_P \bigl[\mathbb{E}_{P\sim\pi,\, (\mathcal{D}_n,X,Y)\sim P^{\otimes (n+1)}} \bigl\{Yg_\pi(X,\mathcal{D}_n)\,|\,X,\mathcal{D}_n\bigr\}\bigr]\\
        &= \mathbb{E}_\pi\mathbb{E}_P \bigl\{g_\pi^2(X,\mathcal{D}_n)\bigr\}.
    \end{align*}
    Therefore,
    \begin{align*}
        \mathcal{E}(\tilde{f}_T) &= \mathbb{E}\mathcal{R}_{\pi}(\tilde{f}_T) - \mathcal{R}_{\pi}(g_\pi) = \mathbb{E}\, \mathbb{E}_\pi\mathbb{E}_P \bigl\{ -2g_\pi(X,\mathcal{D}_n)\tilde{f}_T(X,\mathcal{D}_n) + \tilde{f}_T^2(X,\mathcal{D}_n) + g_\pi^2(X,\mathcal{D}_n) \bigr\}\nonumber\\
        &= \mathbb{E}\, \mathbb{E}_\pi\mathbb{E}_P \bigl[\bigl\{ \tilde{f}_T(X,\mathcal{D}_n) - g_\pi(X,\mathcal{D}_n) \bigr\}^2\bigr].
    \end{align*}
    By~\eqref{eq:def-posterior-regression-function}, we have $|g_{\pi}(X,\mathcal{D}_n)| \leq R$ almost surely.
    Now, for every fixed $P\in\mathcal{P}$, define $Z_P \coloneqq \mathbb{E}\,\mathbb{E}_P \bigl[\bigl\{ \tilde{f}_T(X,\mathcal{D}_n) - g_\pi(X,\mathcal{D}_n) \bigr\}^2\bigr]$, so that by Fubini's theorem and the Cauchy--Schwarz inequality,
    \begin{align}
        \mathbb{E}\, \mathbb{E}_\mu \mathbb{E}_P \bigl[\bigl\{ \tilde{f}_T(X,\mathcal{D}_n) &- g_\pi(X,\mathcal{D}_n) \bigr\}^2\bigr] = \mathbb{E}_{\mu}(Z_P) = \mathbb{E}_{\pi}\biggl( Z_P \cdot \frac{\d\mu}{\d\pi}(P)\biggr)\nonumber\\
        &\leq \sqrt{\mathbb{E}_{\pi}(Z_P^2) \cdot \mathbb{E}_{\pi}\biggl\{\biggl(\frac{\d\mu}{\d\pi}(P)\biggr)^2\biggr\}} \leq \sqrt{4R^2\mathbb{E}_{\pi}(Z_P) \bigl\{\chi^2(\mu,\pi)+1\bigr\}}\nonumber\\
        & = 2R\sqrt{\bigl\{\chi^2(\mu,\pi)+1\bigr\} \mathcal{E}(\tilde{f}_T)}. \label{eq:l2-mu-bound}
    \end{align}
    Further,
    \begin{align}
        \mathbb{E}_\mu \mathbb{E}_P \bigl\{Y\tilde{f}_T(X,\mathcal{D}_n)\bigr\} &= \mathbb{E}_\mu \mathbb{E}_P \bigl[ \mathbb{E}_P\bigl\{Y\tilde{f}_T(X,\mathcal{D}_n) \,|\, X,\mathcal{D}_n\bigr\} \bigr]\nonumber\\
        &= \mathbb{E}_\mu \mathbb{E}_P \bigl\{ \tilde{f}_T(X,\mathcal{D}_n) \cdot \mathbb{E}_P(Y \,|\, X,\mathcal{D}_n) \bigr\}\nonumber\\
        &= \mathbb{E}_\mu \mathbb{E}_P \bigl\{ \tilde{f}_T(X,\mathcal{D}_n) \cdot \mathbb{E}_P(Y \,|\, X) \bigr\}, \label{eq:prop-1-bound-1}
    \end{align}
    and similarly,
    \begin{align}
        \mathbb{E}_\mu \mathbb{E}_P \bigl\{Y \mathbb{E}_P(Y \,|\, X)\bigr\} = \mathbb{E}_\mu \mathbb{E}_P \bigl[ \mathbb{E}_P\bigl\{Y \mathbb{E}_P(Y \,|\, X) \,|\, X\bigr\}\bigr] = \mathbb{E}_\mu \mathbb{E}_P \bigl\{\bigl(\mathbb{E}_P(Y \,|\, X)\bigr)^2\bigr\}. \label{eq:prop-1-bound-2}
    \end{align}
    Hence, we deduce that
    \begin{align}
        \mathcal{R}_{\mu}^{\mathrm{ICL}}(\tilde{f}_T) &= \mathbb{E}\mathcal{R}_{\mu}(\tilde{f}_T) - \mathbb{E}_\mu \mathbb{E}_P \bigl[\bigl\{Y - \mathbb{E}_P(Y\,|\,X)\bigr\}^2\bigr] \nonumber\\
        &= \mathbb{E}\, \mathbb{E}_\mu \mathbb{E}_P \bigl\{ -2Y\tilde{f}_T(X,\mathcal{D}_n) + \tilde{f}_T^2(X,\mathcal{D}_n) + 2Y\mathbb{E}_P(Y\,|\,X) - \bigl(\mathbb{E}_P(Y\,|\,X)\bigr)^2 \bigr\} \nonumber\\
        \overset{(i)}&{=} \mathbb{E}\, \mathbb{E}_\mu \mathbb{E}_P \bigl[ \bigl\{\tilde{f}_T(X,\mathcal{D}_n) - \mathbb{E}_P(Y\,|\,X)\bigr\}^2 \bigr] \label{eq:ICL-excess-risk-alternative-form}\\
        &\leq 2\mathbb{E}\, \mathbb{E}_\mu \mathbb{E}_P \bigl[ \bigl\{\tilde{f}_T(X,\mathcal{D}_n) - g_\pi(X,\mathcal{D}_n)\bigr\}^2 \bigr] + 2\mathbb{E}_\mu \mathbb{E}_P \bigl[ \bigl\{g_\pi(X,\mathcal{D}_n) - \mathbb{E}_P(Y\,|\,X)\bigr\}^2 \bigr] \nonumber\\
        \overset{(ii)}&{\leq}  4R\sqrt{\bigl\{\chi^2(\mu,\pi)+1\bigr\} \mathcal{E}(\tilde{f}_T)} + 2\mathbb{E}_\mu \mathbb{E}_P \bigl[\bigl\{g_\pi(X,\mathcal{D}_n) - \mathbb{E}_P(Y\,|\,X)\bigr\}^2\bigr], \nonumber
    \end{align}
    where~$(i)$ follows from~\eqref{eq:prop-1-bound-1} and~\eqref{eq:prop-1-bound-2}, and $(ii)$ follows from~\eqref{eq:l2-mu-bound}.
\end{proof}

\section{Transformer definition and proofs for Section~\ref{sec:Transformers, Universal Approximation and Learnability}} \label{sec:proofs-for-sec3}

\subsection{Transformer definition} \label{sec:transformer-defn}
In this section, we formally define the classes of (encoder-only) Transformers that we will consider. 
For a vector $v=(v_1,\ldots,v_d)^\top \in\mathbb{R}^d$, let $\softmax(v)\in\mathbb{R}^d$ denote the vector whose $j$th entry is given by $\exp(v_j) / \sum_{\ell=1}^d \exp(v_\ell)$.  Often, we will apply the $\softmax$ function to matrices, by which we mean that the $\softmax$ function is applied row-wise.
A Transformer iteratively applies attention layers and feed forward network (FFN) layers to the input matrix~$Z_{\mathrm{in}}$~\eqref{eq:input-matrix}.  

\begin{defn}[Attention layer] \label{defn:attention-layer}
    Let $H\in\mathbb{N}$ and for $h\in[H]$, let $Q_h,K_h,V_h \in \mathbb{R}^{d_{\mathrm{model}}\times d_{\mathrm{model}}}$. The \emph{attention layer} with $H$ heads and parameters $\theta_{\mathrm{attn}}\coloneqq (Q_h,K_h,V_h)_{h=1}^H$ is the function $\mathsf{Attn}_{\theta_{\mathrm{attn}}}: \bigcup_{n=1}^{N}\mathbb{R}^{(n+1) \times d_{\mathrm{model}}} \to \bigcup_{n=1}^{N}\mathbb{R}^{(n+1) \times d_{\mathrm{model}}}$ such that for $n\in[N]$ and $Z \in \mathbb{R}^{(n+1) \times d_{\mathrm{model}}}$,
    \vspace*{-0.1cm}
    \begin{align*}
        \attn_{\theta_{\mathrm{attn}}}(Z) &\coloneqq Z + \sum_{h=1}^H \softmax\biggl(\frac{ZQ_h K_h^\top Z^\top}{\sqrt{d_{\mathrm{model}}}}\biggr) ZV_h \in \mathbb{R}^{(n+1) \times d_{\mathrm{model}}}.
    \end{align*}
\end{defn}

If $Z\in\mathbb{R}^{(n+1) \times d_{\mathrm{model}}}$ has rows $Z_1^\top,\ldots,Z_{n+1}^\top\in\mathbb{R}^{1\times d_{\mathrm{model}}}$, then the $i$th row $[\attn_{\theta_{\mathrm{attn}}}(Z)]_{i,:}$ of $\attn_{\theta_{\mathrm{attn}}}(Z)$ can be written as
\begin{align*}
    [\attn_{\theta_{\mathrm{attn}}}(Z)]_{i,:} = Z_i^\top + \sum_{h=1}^H \sum_{j=1}^{n+1} Z_j^\top V_h \cdot \frac{\exp(Z_i^\top Q_h K_h^\top Z_j/\sqrt{d_{\mathrm{model}}})}{\sum_{\ell=1}^{n+1} \exp(Z_i^\top Q_h K_h^\top Z_\ell/\sqrt{d_{\mathrm{model}}})} \in\mathbb{R}^{1\times d_{\mathrm{model}}}.
\end{align*}

%We will use $\rho:\mathbb{R}\to\mathbb{R}$ to denote the \emph{activation function} in FFN layers. 
\begin{defn}[FFN layer] \label{defn:ffn-layer}
    Let $d_{\mathrm{ffn}} \in \mathbb{N}$, $W_1\in\mathbb{R}^{d_{\mathrm{model}}\times d_{\mathrm{ffn}}}$, $W_2 \in \mathbb{R}^{d_{\mathrm{ffn}} \times d_{\mathrm{model}}}$, $v\in \mathbb{R}^{d_{\mathrm{ffn}}}$ and let $\rho:\mathbb{R} \rightarrow \mathbb{R}$.  The \emph{FFN layer} with parameters $\theta_{\mathrm{ffn}} \coloneqq (W_1,W_2,v)$ (and activation function~$\rho$) is the function $\mathsf{FFN}_{\theta_{\mathrm{ffn}}} : \bigcup_{n=1}^{N}\mathbb{R}^{(n+1) \times d_{\mathrm{model}}} \to \bigcup_{n=1}^{N}\mathbb{R}^{(n+1) \times d_{\mathrm{model}}}$ such that for $n\in[N]$ and $Z \in \mathbb{R}^{(n+1) \times d_{\mathrm{model}}}$,
    \begin{align*}
        \mathsf{FFN}_{\theta_{\mathrm{ffn}}}\!(Z) \!\coloneqq\! Z \!+\! \rho(ZW_1 \!+\! 1_{n+1} v^\top)W_2 \in \mathbb{R}^{(n\!+\!1) \times d_{\mathrm{model}}},
    \end{align*}
    where $\rho$ is applied entrywise.
\end{defn}

If $Z\in\mathbb{R}^{(n+1) \times d_{\mathrm{model}}}$ has rows $Z_1^\top,\ldots,Z_{n+1}^\top\in\mathbb{R}^{1\times d_{\mathrm{model}}}$, then the $i$th row $[\mathsf{FFN}_{\theta_{\mathrm{ffn}}}(Z)]_{i,:}$ of $\mathsf{FFN}_{\theta_{\mathrm{ffn}}}(Z)$ can be written as
\begin{align*}
    [\mathsf{FFN}_{\theta_{\mathrm{ffn}}}(Z)]_{i,:} = \bigl(Z_i + W_2^\top\rho(W_1^\top Z_i + v)\bigr)^\top.
\end{align*}
Thus the FFN layer applies a single-hidden layer neural network to each row of $Z$.

Now, a Transformer block is a composition of an attention layer and an FFN layer.
\begin{defn}[Transformer block]
    Let $\theta_{\mathrm{attn}}$ and $\theta_{\mathrm{ffn}}$ be as in Definitions~\ref{defn:attention-layer} and~\ref{defn:ffn-layer} respectively.  The \emph{Transformer block} with parameters $(\theta_{\mathrm{attn}}, \theta_{\mathrm{ffn}})$ is the function $\mathsf{TFBlock}_{(\theta_{\mathrm{attn}}, \theta_{\mathrm{ffn}})} : \bigcup_{n=1}^{N}\mathbb{R}^{(n+1) \times d_{\mathrm{model}}} \to \bigcup_{n=1}^{N}\mathbb{R}^{(n+1) \times d_{\mathrm{model}}}$ given by
    \begin{align*}
        \mathsf{TFBlock}_{(\theta_{\mathrm{attn}}, \theta_{\mathrm{ffn}})} \coloneqq \mathsf{FFN}_{\theta_{\mathrm{ffn}}} \circ \attn_{\theta_{\mathrm{attn}}}.
    \end{align*}
    We say that the Transformer block $\mathsf{TFBlock}_{(\theta_{\mathrm{attn}}, \theta_{\mathrm{ffn}})}$ has $H$ heads, model dimension $d_{\mathrm{model}}$ and FFN width $d_{\mathrm{ffn}}$.
\end{defn}

Finally, a Transformer is a composition of Transformer blocks.
\begin{defn}[Transformer] \label{defn:transformer}
    Let $\mathcal{F}_{\mathrm{TF}}(L, H, d_{\mathrm{model}}, d_{\mathrm{ffn}})$ be the set of all \emph{Transformers} 
    %with $L$ Transformer blocks, $H$ heads, model dimension $d_{\mathrm{model}}$ and FFN width~$d_{\mathrm{ffn}}$, which are functions 
    $\mathsf{TF}:  \bigcup_{n=1}^{N}\mathbb{R}^{(n+1) \times d_{\mathrm{model}}} \to \bigcup_{n=1}^{N}\mathbb{R}^{(n+1) \times d_{\mathrm{model}}}$ of the form
    \begin{align*}
        \mathsf{TF} = \mathsf{TFBlock}^{(L)} \circ \mathsf{TFBlock}^{(L-1)} \circ \cdots \circ \mathsf{TFBlock}^{(1)},
    \end{align*}
    where $\mathsf{TFBlock}^{(\ell)} \equiv \mathsf{TFBlock}_{(\theta_{\mathrm{attn}}^{(\ell)},\theta_{\mathrm{ffn}}^{(\ell)})}$ is a Transformer block with $H$ heads, model dimension $d_{\mathrm{model}}$ and FFN width $d_{\mathrm{ffn}}$ for all $\ell\in[L]$. 
\end{defn}

\subsection{Proof of Theorem~\ref{thm:universal-approximation}} \label{sec:proof-of-universal-approximation}
We work under the assumptions of Section~\ref{sec:universal-approximation} and assume that $\rho$ satisfies the conditions stated immediately before Theorem~\ref{thm:universal-approximation}. For $g\in\mathcal{G}$, $(X,Y_g)\sim P_g$ and $z\in\mathbb{R}$, let $y \mapsto f_{\xi}(y,z)$ denote the conditional density of $Y_g\,|\,\{g(X)=z\}$ with respect to the dominating measure $\nu$. Further let $\bar{f}_{\xi} \coloneqq f_{\xi}/R'$.  
By Bayes' theorem and~\eqref{eq:alternative-def-posterior-regression-function}, the posterior regression function $g_{\pi}$ can be written as
\begin{align}
    g_{\pi}(x,D_n) = \frac{\int_{\mathcal{G}} g(x) \prod_{i=1}^n \bar{f}_{\xi}\bigl(y_i - g(x_i)\bigr) \,\d \tilde{\pi}(g)}{\int_{\mathcal{G}} \prod_{i=1}^n \bar{f}_{\xi}\bigl(y_i - g(x_i)\bigr) \,\d \tilde{\pi}(g)}. \label{eq:g-pi-with-f-bar}
\end{align}
Throughout this section, we assume that $\bigl((X_1,Y_1),\ldots,(X_N,Y_N),(X,Y)\bigr) \,|\, P \sim P^{\otimes(N+1)}$ with $P\sim\pi$. For $n\in[N]$ and $d_{\mathrm{model}}\geq d+2$, define
\begin{align}
    Z_{\mathrm{in}}^{(n,d_{\mathrm{model}})} \coloneqq \begin{pmatrix}
        X_1^\top & Y_1 & 0_{d_{\mathrm{model}}-d-2}^\top & 0\\
        \vdots & \vdots & \vdots & \vdots\\
        X_n^\top & Y_n & 0_{d_{\mathrm{model}}-d-2}^\top & 0\\
        X^\top & 0 & 0_{d_{\mathrm{model}}-d-2}^\top & 1
    \end{pmatrix} \in \mathbb{R}^{(n+1) \times d_{\mathrm{model}}}. \label{eq:input-matrix-(n,D)}
\end{align}

\begin{lemma} \label{lemma:approximate-integral}
    For any $\epsilon,\delta>0$, there exist $M\in\mathbb{N}$ and $g_1,\ldots,g_M \in \mathcal{G}$ such that with probability at least $1-\delta$, we have 
    \begin{align*}
        \max_{n \in [N]} \biggl| \frac{1}{M} \sum_{m=1}^M g_m(X) \prod_{i=1}^n \bar{f}_{\xi}\bigl(Y_i, g_m(X_i)\bigr) - \int_{\mathcal{G}} g(X) \prod_{i=1}^n \bar{f}_{\xi}\bigl(y_i, g(x_i)\bigr) \,\d \tilde{\pi}(g) \biggr| \leq \epsilon
    \end{align*}
    and
    \begin{align*}
        \max_{n \in [N]} \biggl| \frac{1}{M} \sum_{m=1}^M \prod_{i=1}^n \bar{f}_{\xi}\bigl(Y_i, g_m(X_i)\bigr) - \int_{\mathcal{G}} \prod_{i=1}^n \bar{f}_{\xi}\bigl(y_i, g(x_i)\bigr) \,\d \tilde{\pi}(g) \biggr| \leq \epsilon.
    \end{align*}
\end{lemma}
\begin{proof}
    Let $M\coloneqq 2\epsilon^{-2}(R^2 \vee 1)\log(2N/\delta)$ and let $\tilde{g}_1,\ldots,\tilde{g}_M \overset{\mathrm{iid}}{\sim} \tilde{\pi}$ be independent of $(X_i,Y_i)_{i=1}^N$ and $(X,Y)$. 
    % Then for all $n\in[N], (x_1,y_1),\ldots,(x_n,y_n) \in \mathcal{X}\times\mathcal{Y}$ and $x \in \mathcal{X}$, we have by Hoeffding's inequality that
    % \begin{align*}
    %     &\mathbb{P} \biggl\{\biggl| \frac{1}{M} \sum_{m=1}^M \tilde{g}_m(x) \prod_{i=1}^n f_\xi\bigl(y_i - \tilde{g}_m(x_i)\bigr) - \int_{\mathcal{G}} g(x) \prod_{i=1}^n f_\xi\bigl(y_i, g(X_i)\bigr) \,\d \tilde{\pi}(g) \biggr| > \epsilon \biggr\} \leq \frac{\delta}{2N}
    % \end{align*}
    % and
    % \begin{align*}
    %     \mathbb{P} \biggl\{\biggl| \frac{1}{M} \sum_{m=1}^M \prod_{i=1}^n f_\xi\bigl(y_i - \tilde{g}_m(x_i)\bigr) - \int_{\mathcal{G}} \prod_{i=1}^n f_\xi\bigl(y_i, g(X_i)\bigr) \,\d \tilde{\pi}(g) \biggr| > \epsilon \biggr\} \leq \frac{\delta}{2N}.
    % \end{align*}
    For $n\in[N]$, define events 
    \begin{align*}
        A_n \coloneqq \Biggl\{ \biggl| \frac{1}{M} \sum_{m=1}^M \tilde{g}_m(X) \prod_{i=1}^n \bar{f}_{\xi}\bigl(Y_i, \tilde{g}_m(X_i)\bigr) - \int_{\mathcal{G}} g(X) \prod_{i=1}^n \bar{f}_{\xi}\bigl(y_i, g(x_i)\bigr) \,\d \tilde{\pi}(g) \biggr| \leq \epsilon \Biggr\}
    \end{align*}
    and
    \begin{align*}
        B_n \coloneqq \Biggl\{ \biggl| \frac{1}{M} \sum_{m=1}^M \prod_{i=1}^n \bar{f}_{\xi}\bigl(Y_i, \tilde{g}_m(X_i)\bigr) - \int_{\mathcal{G}} \prod_{i=1}^n \bar{f}_{\xi}\bigl(y_i, g(x_i)\bigr) \,\d \tilde{\pi}(g) \biggr| \leq \epsilon \Biggr\}.
    \end{align*}
    We have by Hoeffding's inequality that
    \begin{align*}
        \mathbb{P}(A_n^{\mathrm{c}} \,|\, X_1,Y_1,\ldots,X_n,Y_n,X) \leq \frac{\delta}{2N} \quad\text{and}\quad
        \mathbb{P}(B_n^{\mathrm{c}} \,|\, X_1,Y_1,\ldots,X_n,Y_n,X) \leq \frac{\delta}{2N}.
    \end{align*}
    Then by a union bound,
    \begin{align*}
        \mathbb{E}\biggl\{ \mathbb{P}\biggl( \bigcup_{n=1}^N (A_n^{\mathrm{c}} \cup B_n^{\mathrm{c}}) &\,\bigg|\, \tilde{g}_1,\ldots,\tilde{g}_M \biggr) \biggr\} = \mathbb{P}\biggl( \bigcup_{n=1}^N (A_n^{\mathrm{c}} \cup B_n^{\mathrm{c}})\biggr)\\
        & = \mathbb{E}\biggl\{ \mathbb{P}\biggl( \bigcup_{n=1}^N (A_n^{\mathrm{c}} \cup B_n^{\mathrm{c}}) \,\bigg|\, X_1,Y_1,\ldots,X_N,Y_N,X,Y \biggr) \biggr\} \leq \delta.
    \end{align*}
    Therefore, 
    \begin{align*}
        \mathbb{P}\biggl\{ \mathbb{P}\biggl( \bigcup_{n=1}^N (A_n^{\mathrm{c}} \cup B_n^{\mathrm{c}}) \,\biggm|\, \tilde{g}_1,\ldots,\tilde{g}_M \biggr) \leq \delta\biggr\} > 0.
    \end{align*}
    This proves the claim.
\end{proof}

\begin{lemma}\label{lemma:denominator-lb}
    For all $n\in[N]$, we have
    \begin{align*}
        \int_{\mathcal{G}} \prod_{i=1}^n \bar{f}_{\xi}\bigl(y_i, g(x_i)\bigr) \,\d \tilde{\pi}(g) > 0
    \end{align*}
    almost surely. 
    %Moreover, for any $\delta\in(0,1)$, there exists $b>0$ such that with probability at least $1-\delta$,
    % \begin{align*}
    %     \min_{n\in[N]} \int_{\mathcal{G}} \prod_{i=1}^n \bar{f}_{\xi}\bigl(y_i, g(x_i)\bigr) \,\d \tilde{\pi}(g) \geq b.
    % \end{align*}
\end{lemma}
\begin{proof}
    The density (with respect to $\nu$) of $(Y_1,\ldots,Y_n)$ conditional on $(X_1,\ldots,X_n)$ is given by
    \begin{align*}
        (y_1,\ldots,y_n) \mapsto \int_{\mathcal{G}} \prod_{i=1}^n f_\xi\bigl(y_i, g(X_i)\bigr) \,\d \tilde{\pi}(g).
    \end{align*}
    Hence, by definition, 
    \begin{align*}
        &\mathbb{P}\biggl( \int_{\mathcal{G}} \prod_{i=1}^n \bar{f}_{\xi}\bigl(y_i, g(x_i)\bigr) \,\d \tilde{\pi}(g) = 0\biggr) = \mathbb{E}\biggl\{ \mathbb{P}\biggl( \int_{\mathcal{G}} \prod_{i=1}^n f_\xi\bigl(y_i, g(X_i)\bigr) \,\d \tilde{\pi}(g) = 0 \,\biggm|\, X_1,\ldots,X_n\biggr) \biggr\} = 0.
    \end{align*}
    This holds for all $n\in[N]$, which proves the claim. 
    % Thus, for each $n\in[N]$, there exists $b_n>0$ such that with probability at least $1-\delta/N$,
    % \begin{align*}
    %     \int_{\mathcal{G}} \prod_{i=1}^n \bar{f}_{\xi}\bigl(y_i, g(x_i)\bigr) \,\d \tilde{\pi}(g) \geq b_n.
    % \end{align*}
    % Taking $b\coloneqq \min_{n\in[N]} b_n > 0$ and applying a union bound proves the second claim.
\end{proof}

\begin{lemma}\label{lemma:universal-approximation-of-NN}
    Let $d_1,d_2 \in \mathbb{N}$, let $f:\mathbb{R}^{d_1} \to \mathbb{R}^{d_2}$ be Borel measurable and let $\mu$ be a Borel probability measure on $\mathbb{R}^{d_1}$.  Then for every $\epsilon>0$ and $\delta\in(0,1)$, there exist $d'\in\mathbb{N}$, $W_1\in\mathbb{R}^{d_1 \times d'}$, $W_2\in\mathbb{R}^{d'\times d_2}$ and $v\in\mathbb{R}^{d'}$ such that 
    \begin{align*}
        \mu\Bigl(\bigl\{x\in\mathbb{R}^{d_1} : \bigl\|f(x) - W_2^\top \rho(W_1^\top x + v)\bigr\|_{\infty} \leq \epsilon\bigr\}\Bigr) \geq 1-\delta.
    \end{align*}
\end{lemma}
\begin{proof}
    We first prove the claim for $d_2=1$.  Let $K_0 \subseteq \mathbb{R}^{d_1}$ be a compact set such that $\mu(K_0) \geq 1-\delta/3$.
    By Lusin's theorem \citep[Theorems~7.8 and~7.10]{folland1999real}, there exists a (compactly supported) continuous function $h:\mathbb{R}^{d_1} \to \mathbb{R}$ such that $\mu\bigl(\{x \in K_0: f(x) \neq h(x)\}\bigr) \leq \delta/3$. Since one-hidden layer neural networks with non-polynomial activation are dense in the set of continuous functions on compact domains equipped with the uniform norm \citep[Theorem~1]{leshno1993multilayer}, there exist $d'\in\mathbb{N}$, $W_1\in\mathbb{R}^{d_1\times d'}$, $W_2\in\mathbb{R}^{d'\times d_2}$ and $v\in\mathbb{R}^{d'}$ such that for all $x \in K_0$,
    \begin{align*}
        \bigl|h(x) - W_2^\top \rho(W_1^\top x + v)\bigr| \leq \frac{\delta\epsilon}{3}.
    \end{align*}
    % Thus
    % \begin{align*}
    %     \int_{K_0} \bigl|h(x) - W_2^\top \rho(W_1^\top x + v)\bigr| \,\d\mu(x) % \leq \frac{\delta\epsilon}{2}.
    % \end{align*}
    It follows by Markov's inequality that
    \begin{align*}
        \mu\Bigl(\bigl\{x \in \mathbb{R}^{d_1}: \bigl|f(x) &- W_2^\top \rho(W_1^\top x + v)\bigr| \geq \epsilon\bigr\}\Bigr)\\
        &\leq \mu\Bigl(\bigl\{x \in K_0: \bigl|f(x) - W_2^\top \rho(W_1^\top x + v)\bigr| \geq \epsilon\bigr\}\Bigr) + \frac{\delta}{3} \\
        &\leq \mu\Bigl(\bigl\{x \in K_0: \bigl|h(x) - W_2^\top \rho(W_1^\top x + v)\bigr| \geq \epsilon\bigr\}\Bigr) + \frac{\delta}{3} + \frac{\delta}{3} \leq \delta.
    \end{align*}
    This proves the claim when $d_2=1$. For $d_2>1$, we may approximate the component functions in parallel and apply a union bound to deduce the final result.
\end{proof}

\begin{lemma}\label{lemma:ffn-layer-approx}
    Let $d_{\mathrm{model}}\geq d+2$ and let $f_1 : \mathbb{R}^d \times \mathbb{R} \to \mathbb{R}^{d_{\mathrm{model}}}$ and $f_2 : \mathbb{R}^d \to \mathbb{R}^{d_{\mathrm{model}}}$ be Borel measurable. For $n\in[N]$, let $Z_{\mathrm{in}}^{(n,d_{\mathrm{model}})}\in\mathbb{R}^{(n+1)\times d_{\mathrm{model}}}$ be defined as in~\eqref{eq:input-matrix-(n,D)} and let
    \begin{align*}
        U^{(n)} \coloneqq \begin{pmatrix}
            f_1^\top(X_1,Y_1)\\
            \vdots\\
            f_1^\top(X_n,Y_n)\\
            f_2^\top(X)
        \end{pmatrix} \in \mathbb{R}^{(n+1)\times d_{\mathrm{model}}}.
    \end{align*}
    Then for $\delta\in(0,1)$ and $\epsilon>0$, there exists an FFN layer $\mathsf{FFN}_{\theta_{\mathrm{ffn}}}$ such that with probability at least $1-\delta$,
    \begin{align*}
        \max_{n\in[N]} \bigl\|\mathsf{FFN}_{\theta_{\mathrm{ffn}}}(Z_{\mathrm{in}}^{(n,d_{\mathrm{model}})}) - U^{(n)}\bigr\|_{\infty} \leq \epsilon.
    \end{align*}
\end{lemma}
\begin{proof}
    Define a Borel measurable function $f:\mathbb{R}^{d_{\mathrm{model}}} \to \mathbb{R}^{d_{\mathrm{model}}}$ by $f(x,y,0_{d_{\mathrm{model}}-d-2},0) \coloneqq f_1(x,y) - (x,y,0_{d_{\mathrm{model}}-d-2},0)$ and $f(x,0,0_{d_{\mathrm{model}}-d-2},1) \coloneqq f_2(x) - (x,0,0_{d_{\mathrm{model}}-d-2},1)$ for $(x,y)\in\mathbb{R}^d\times\mathbb{R}$ and set~$f$ to be zero otherwise. 
    %Choose $R_0>0$ such that $\mathbb{P}(\|X\|_{\infty} \vee |Y| \leq R_0) \geq 1-\delta/2$. 
    Let~$\mu_1$ denote the distribution of $(X,Y)$ and let $\mu_2$ denote the marginal distribution of $X$. Define a distribution $\mu$ on $\mathbb{R}^{d_{\mathrm{model}}}$ by 
    \begin{align*}
        \mu\biggl(\prod_{j=1}^{d_{\mathrm{model}}} A_j\biggr) \coloneqq \frac{\mu_1(\prod_{j=1}^{d+1} A_j) \mathbbold{1}_{\{0 \in A_{d_{\mathrm{model}}}\}} + \mu_2(\prod_{j=1}^{d} A_j)\mathbbold{1}_{\{1 \in A_{d_{\mathrm{model}}}\}}}{2} \cdot \prod_{j=d+2}^{{d_{\mathrm{model}}}-1} \mathbbold{1}_{\{0\in A_j\}},
    \end{align*}
    for all measurable sets $A_1,\ldots,A_{d_{\mathrm{model}}} \subseteq \mathbb{R}$.  By Lemma~\ref{lemma:universal-approximation-of-NN} and the definition of $\mu$ above, there exists an FFN layer $\mathsf{FFN}_{\theta_{\mathrm{ffn}}}$ such that
    \begin{align*}
        \mathbb{P}\Bigl(\bigl\| \mathsf{FFN}_{\theta_{\mathrm{ffn}}}&(X,Y,0_{{d_{\mathrm{model}}}-d-2},0) - f_1(X,Y)\bigr\|_{\infty} > \epsilon \Bigr)\\
        &+ \mathbb{P}\Bigl(\bigl\| \mathsf{FFN}_{\theta_{\mathrm{ffn}}}(X,0,0_{{d_{\mathrm{model}}}-d-2},1) - f_2(X)\bigr\|_{\infty} > \epsilon \Bigr) \leq \frac{\delta}{N+1}.
    \end{align*}
    We conclude the claim by a union bound.
\end{proof}

\begin{proof}[Proof of Theorem~\ref{thm:universal-approximation}]
    It suffices to prove the theorem with $L=3$ and $H=1$.
    Let $\delta\coloneqq \frac{7\epsilon}{80R^2} \wedge 1$.  By Lemma~\ref{lemma:denominator-lb} and a union bound, there exists $b\in(0,1]$ such that the event
    \begin{align*}
        \mathcal{E}_0 \coloneqq \Bigl\{\min_{n\in[N]} \int_{\mathcal{G}} \prod_{i=1}^n \bar{f}_{\xi}\bigl(y_i, g(x_i)\bigr) \,\d \tilde{\pi}(g) \geq b\Bigr\}
    \end{align*}
    has probability at least $1-\delta$.
    Let $\epsilon_1\coloneqq \frac{b}{8} \wedge \frac{b\sqrt{\epsilon}}{64(R\vee 1)}$. By Lemma~\ref{lemma:approximate-integral}, there exist $M\geq d$ and $g_1,\ldots,g_M \in \mathcal{G}$ such that the event
    \begin{align*}
        \mathcal{E}_1 \coloneqq \biggl\{ \max_{n \in [N]} \biggl| \frac{1}{M} \sum_{m=1}^M g_m(X) \prod_{i=1}^n \bar{f}_{\xi}\bigl(Y_i, g_m(X_i)\bigr) - \int_{\mathcal{G}} g(X) \prod_{i=1}^n \bar{f}_{\xi}\bigl(Y_i, g(X_i)\bigr) \,\d \tilde{\pi}(g) \biggr| \leq \epsilon_1 \biggr\}\\
        \bigcap \biggl\{ \max_{n \in [N]} \biggl| \frac{1}{M} \sum_{m=1}^M \prod_{i=1}^n \bar{f}_{\xi}\bigl(Y_i, g_m(X_i)\bigr) - \int_{\mathcal{G}} \prod_{i=1}^n \bar{f}_{\xi}\bigl(Y_i, g(X_i)\bigr) \,\d \tilde{\pi}(g) \biggr| \leq \epsilon_1 \biggr\}
    \end{align*}
    has probability at least $1-\delta$.
    % Define $h = (h_1,\ldots,h_{2M})^\top : \mathbb{R}^{d_{\mathrm{model}}} \to \mathbb{R}^{2M}$ by 
    % \begin{align*}
    %     h_m(x,y,0_{{d_{\mathrm{model}}}-d-2},0) = h_{M+m}(x,y,0_{{d_{\mathrm{model}}}-d-2},0) \coloneqq \log\bigl\{\bar{f}_{\xi}\bigl(y-g_m(x)\bigr)\bigr\} \vee 
    % \end{align*}
    % for $m\in[M]$ and $(x,y)\in[0,1]^d\times\mathbb{R}$;
    % \begin{align*}
    %     h_m(x,0,0_{{d_{\mathrm{model}}}-d-2},1) \coloneqq \log g_m(x)
    % \end{align*}
    % for $m\in[M]$ and $x\in[0,1]^d$; and zero otherwise.
    Now let ${d_{\mathrm{model}}}\coloneqq 4M+1$ and define $f_1=(f_{1,1},\ldots,f_{1,{d_{\mathrm{model}}}}) : \mathbb{R}^d\times\mathbb{R} \to \mathbb{R}^{d_{\mathrm{model}}}$ by 
    \begin{align*}
        f_{1,j}(x,y) \coloneqq \begin{cases}
            \log \bar{f}_{\xi}\bigl(y, g_j(x)\bigr) \quad&\text{if } j\in[M]\\
            \log \bar{f}_{\xi}\bigl(y, g_{j-M}(x)\bigr) \quad&\text{if } j\in[M+1:2M]\\
            0 &\text{if } j\in[2M+1:4M+1].
        \end{cases}
    \end{align*}
    Also define $f_2=(f_{2,1},\ldots,f_{2,{d_{\mathrm{model}}}}) : \mathbb{R}^d \to \mathbb{R}^{d_{\mathrm{model}}}$ by 
    \begin{align*}
        f_{2,j}(x) \coloneqq \begin{cases}
            \log g_j(x) \quad&\text{if } j\in[M]\\
            0 &\text{if } j\in[M+1:4M]\\
            1 &\text{if } j=4M+1.
        \end{cases}
    \end{align*}
    Finally, for $n\in[N]$, define $U^{(n)}\in \mathbb{R}^{(n+1)\times {d_{\mathrm{model}}}}$ by
    \begin{align*}
        U^{(n)} \coloneqq \begin{pmatrix}
            f_1(X_1,Y_1)\\
            \vdots\\
            f_1(X_n,Y_n)\\
            f_2(X)
        \end{pmatrix}.
    \end{align*}
    By Lemma~\ref{lemma:ffn-layer-approx}, there exists an FFN layer $\mathsf{FFN}_{\theta_{\mathrm{ffn}}^{(1)}}$ with $\theta_{\mathrm{ffn}}^{(1)} = (W_1,W_2,v)$, such that writing $\epsilon_2\coloneqq \frac{b}{32(N+1)} \wedge \frac{b\sqrt{\epsilon}}{256R(N+1)}$, the event
    \begin{align*}
        \mathcal{E}_2 \coloneqq \Bigl\{\max_{n\in[N]} \bigl\|\mathsf{FFN}_{\theta_{\mathrm{ffn}}^{(1)}}(Z_{\mathrm{in}}^{(n,{d_{\mathrm{model}}})}) - U^{(n)}\bigr\|_{\infty} \leq \epsilon_2\Bigr\}
    \end{align*}
    has probability at least $1-\delta$.  Since the final column of $W_2$ only affects the final column of $\mathsf{FFN}_{\theta_{\mathrm{ffn}}^{(1)}}(Z_{\mathrm{in}}^{(n,{d_{\mathrm{model}}})})$, we can choose the final column of $W_2$ to be zero, so that
    \begin{align}
    \bigl[\mathsf{FFN}_{\theta_{\mathrm{ffn}}^{(1)}}(Z_{\mathrm{in}}^{(n,{d_{\mathrm{model}}})})\bigr]_{i,{d_{\mathrm{model}}}} = \bigl[Z_{\mathrm{in}}^{(n,{d_{\mathrm{model}}})}\bigr]_{i,{d_{\mathrm{model}}}} = \mathbbold{1}_{\{i=n+1\}} = \bigl[U^{(n)}\bigr]_{i,{d_{\mathrm{model}}}} \label{eq:last-column-ffn-1}
    \end{align}
    for $i\in[n+1]$.
    Next, we define $Q^{(2)}=K^{(2)}\coloneqq 0_{{d_{\mathrm{model}}}\times {d_{\mathrm{model}}}}$ and define $V^{(2)}\in\mathbb{R}^{{d_{\mathrm{model}}}\times {d_{\mathrm{model}}}}$ by $V^{(2)}_{j,2M+j}\coloneqq 1$ for $j\in[2M]$, $V^{(2)}_{4M+1,4M+1}\coloneqq 1$ and zero otherwise. Let $\theta^{(2)}_{\mathrm{attn}}\coloneqq (Q^{(2)},K^{(2)},V^{(2)})$, $\tilde{Z}^{(n)} \coloneqq \mathsf{FFN}_{\theta_{\mathrm{ffn}}^{(1)}}(Z_{\mathrm{in}}^{(n,{d_{\mathrm{model}}})}) \in \mathbb{R}^{(n+1)\times {d_{\mathrm{model}}}}$ and $\bar{Z}^{(n)}\coloneqq \mathsf{Attn}_{\theta^{(2)}_{\mathrm{attn}}}(\tilde{Z}^{(n)})\in\mathbb{R}^{(n+1)\times {d_{\mathrm{model}}}}$. Then by construction and~\eqref{eq:last-column-ffn-1}, the $(n+1)$th row of $\bar{Z}^{(n)}$ is given by
    \begin{align} \label{eq:expression-of-bar-Z}
        \bar{Z}^{(n)}_{n+1,j} = \begin{cases}
            \tilde{Z}^{(n)}_{n+1,j} \quad&\text{if }j\in[2M]\\
            \tilde{Z}^{(n)}_{n+1,j} + \frac{1}{n+1}\sum_{i=1}^{n+1} \tilde{Z}^{(n)}_{i,j-2M} \quad&\text{if }j\in[2M+1:4M]\\
            1+\frac{1}{n+1} \quad&\text{if }j=4M+1.
        \end{cases}
    \end{align}
    Thus, for $n\in[N]$ and $m\in[M]$, we have that on $\mathcal{E}_2$,
    \begin{align}
        \biggl| \bar{Z}^{(n)}_{n+1,2M+m} - \frac{1}{n+1}\sum_{i=1}^{n} &\log \bar{f}_{\xi}\bigl(Y_i,g_m(X_i)\bigr) - \frac{1}{n+1}\log g_m(X) \biggr|\nonumber\\
        &= \biggl| \tilde{Z}^{(n)}_{n+1,2M+m} + \frac{1}{n+1}\sum_{i=1}^{n+1} \tilde{Z}^{(n)}_{i,m} - \frac{1}{n+1} \sum_{i=1}^{n+1} U^{(n)}_{i,m}\biggr|\nonumber\\
        &\leq \bigl|\tilde{Z}^{(n)}_{n+1,2M+m}-0\bigr| + \frac{1}{n+1}\sum_{i=1}^{n+1} \bigl|\tilde{Z}^{(n)}_{i,m} - U^{(n)}_{i,m}\bigr| \leq 2\epsilon_2. \label{eq:approx-sum-of-logs-1}
    \end{align}
    Similarly, for $n\in[N]$ and $m\in[M]$, we have that on $\mathcal{E}_2$,
    \begin{align}
        \biggl| \bar{Z}^{(n)}_{n+1,3M+m} - \frac{1}{n+1}\sum_{i=1}^{n} \log \bar{f}_{\xi}\bigl(Y_i,g_m(X_i)\bigr) \biggr| \leq 2\epsilon_2. \label{eq:approx-sum-of-logs-2}
    \end{align}
    Now define a continuous function $f_3: \mathbb{R}^{4M+1} \to \mathbb{R}$ by
    \begin{align*}
        f_3(z_1,\ldots,z_{4M+1}) \coloneqq \frac{\sum_{m=1}^M \exp\{z_{2M+m}/(z_{4M+1}-1)\}}{\sum_{m=1}^M \exp\{z_{3M+m}/(z_{4M+1}-1)\}}.
    \end{align*}
    For $n\in[N]$, let $\bar{P}^{(n)}$ be the distribution of the last row of $\bar{Z}^{(n)}$.
    By applying Lemma~\ref{lemma:universal-approximation-of-NN} with $\mu = \frac{1}{N} \sum_{n=1}^N \bar{P}^{(n)}$ therein, there exists an FFN layer $\mathsf{FFN}_{\theta_{\mathrm{ffn}}^{(2)}}$ such that the event
    \begin{align*}
        \mathcal{E}_3 \coloneqq \Bigl\{ \max_{n\in[N]}\Bigl| \bigl[\mathsf{FFN}_{\theta_{\mathrm{ffn}}^{(2)}}(\bar{Z}^{(n)})\bigr]_{n+1,d+1} - f_3\bigl(\bar{Z}^{(n)}_{n+1,1},\ldots,\bar{Z}^{(n)}_{n+1,{d_{\mathrm{model}}}}\bigr) \Bigr| \leq \sqrt{\epsilon}/4 \Bigr\}
    \end{align*}
    has probability at least $1-\delta$.
    Since $\|g_m\|_{\infty}\leq R$ and $\|\bar{f}_\xi\|_{\infty} \leq 1$, we have by~\eqref{eq:expression-of-bar-Z} and~\eqref{eq:approx-sum-of-logs-1} that on $\mathcal{E}_1\cap\mathcal{E}_2$, for all $n\in[N]$,
    \begin{align}
        &\biggl|\frac{1}{M}\sum_{m=1}^M \exp\bigl\{ \bar{Z}^{(n)}_{n+1,2M+m}/(\bar{Z}^{(n)}_{n+1,4M+1}-1)\bigr\} - \int_{\mathcal{G}} g(X) \prod_{i=1}^n \bar{f}_{\xi}\bigl(y_i, g(x_i)\bigr) \,\d \tilde{\pi}(g)\biggr|\nonumber\\
        &\leq \biggl|\frac{1}{M}\sum_{m=1}^M \exp\bigl\{ \bar{Z}^{(n)}_{n+1,2M+m}/(\bar{Z}^{(n)}_{n+1,4M+1}-1)\bigr\} - \frac{1}{M}\sum_{m=1}^M g_m(X)\prod_{i=1}^n\bar{f}_{\xi}\bigl(Y_i,g_m(X_i)\bigr)\biggr| + \epsilon_1\nonumber\\
        &\leq \frac{1}{M}\sum_{m\in[M]}\biggl|\exp\bigl\{ (n+1)\bar{Z}^{(n)}_{n+1,2M+m}\bigr\} - g_m(X)\prod_{i=1}^n\bar{f}_{\xi}\bigl(Y_i,g_m(X_i)\bigr)\biggr| + \epsilon_1\nonumber\\
        &\leq R(e^{2(N+1)\epsilon_2}-1) + \epsilon_1 \leq \frac{b\sqrt{\epsilon}}{32} \eqqcolon \epsilon_3 . \label{eq:numerator-error}
    \end{align}
    Similarly, on $\mathcal{E}_1 \cap \mathcal{E}_2$ we have by~\eqref{eq:expression-of-bar-Z} and~\eqref{eq:approx-sum-of-logs-2} that for all $n\in[N]$,
    \begin{align}
        \biggl|\frac{1}{M}\sum_{m=1}^M \exp\bigl\{ \bar{Z}^{(n)}_{n+1,3M+m}/(\bar{Z}^{(n)}_{n+1,4M+1}&-1)\bigr\} - \int_{\mathcal{G}} \prod_{i=1}^n \bar{f}_{\xi}\bigl(y_i, g(x_i)\bigr) \,\d \tilde{\pi}(g)\biggr|\nonumber\\
        &\leq (e^{2(N+1)\epsilon_2}-1) + \epsilon_1 \leq \frac{b\sqrt{\epsilon}}{32R} \wedge \frac{b}{4} \eqqcolon \epsilon_4. \label{eq:denominator-error}
    \end{align}
    Now, on the event $\mathcal{E}_0\cap\mathcal{E}_1$, since $\epsilon_1\leq b/2$, we have that 
    \begin{align}
        \min_{n \in [N]}\frac{1}{M} \sum_{m=1}^M \prod_{i=1}^n \bar{f}_{\xi}\bigl(Y_i, g_m(X_i)\bigr) \geq \frac{b}{2}.\label{eq:denominator-lb}
    \end{align}
    Therefore, on $\mathcal{E}_0\cap \mathcal{E}_1\cap \mathcal{E}_2$, we have by~\eqref{eq:numerator-error},~\eqref{eq:denominator-error},~\eqref{eq:denominator-lb} and~\eqref{eq:g-pi-with-f-bar} that for all $n\in[N]$
    \begin{align} \bigl| f_3&(\bar{Z}^{(n)}_{n+1,1},\ldots,\bar{Z}^{(n)}_{n+1,{d_{\mathrm{model}}}}) - g_\pi(X,\mathcal{D}_n)\bigr|\nonumber\\
        &= \Biggl|\frac{\sum_{m=1}^M \exp\{ \bar{Z}^{(n)}_{n+1,2M+m}/(\bar{Z}^{(n)}_{n+1,4M+1}-1)\}}{\sum_{m=1}^M \exp\{ \bar{Z}^{(n)}_{n+1,3M+m}/(\bar{Z}^{(n)}_{n+1,4M+1}-1)\}} - \frac{\int_{\mathcal{G}} g(X) \prod_{i=1}^n \bar{f}_{\xi}\bigl(y_i, g(x_i)\bigr) \,\d \tilde{\pi}(g)}{\int_{\mathcal{G}} \prod_{i=1}^n \bar{f}_{\xi}\bigl(y_i, g(x_i)\bigr) \,\d \tilde{\pi}(g)}\Biggr|\nonumber\\
        &\leq \biggl(\frac{1}{1-2\epsilon_4/b}-1\biggr) \Biggl|\frac{\int_{\mathcal{G}} g(X) \prod_{i=1}^n \bar{f}_{\xi}\bigl(y_i, g(x_i)\bigr) \,\d \tilde{\pi}(g)}{\int_{\mathcal{G}} \prod_{i=1}^n \bar{f}_{\xi}\bigl(y_i, g(x_i)\bigr) \,\d \tilde{\pi}(g)}\Biggr| + \frac{\epsilon_3}{b/2 - \epsilon_4} \nonumber\\
        &\leq \biggl(\frac{1}{1-2\epsilon_4/b}-1\biggr)R + \frac{2\epsilon_3}{b-2\epsilon_4} \leq \frac{4R\epsilon_4}{b} + \frac{4\epsilon_3}{b} \leq \frac{\sqrt{\epsilon}}{4}, \label{eq:f3-bound}
    \end{align}
    where in the penultimate inequality, we used the fact that $\frac{1}{1-x}-1 \leq 2x$ for $x\in[0,1/2]$ and that $\epsilon_4\leq b/4$, and the final inequality uses the definitions of $\epsilon_3$ and $\epsilon_4$.  If we set the first attention layer $\mathsf{Attn}_{\theta^{(1)}_{\mathrm{attn}}}$ to be the identity map (i.e.~set the parameters of the attention layer to zero), and let $\check{Z}^{(n)} \coloneqq \mathsf{FFN}_{\theta_{\mathrm{ffn}}^{(2)}} \circ\mathsf{Attn}_{\theta^{(2)}_{\mathrm{attn}}} \circ \mathsf{FFN}_{\theta_{\mathrm{ffn}}^{(1)}} \circ \mathsf{Attn}_{\theta^{(1)}_{\mathrm{attn}}} (Z_{\mathrm{in}}^{(n,{d_{\mathrm{model}}})})$, then by~\eqref{eq:f3-bound} and the triangle inequality, we have on $\mathcal{E}_0\cap \mathcal{E}_1\cap \mathcal{E}_2 \cap \mathcal{E}_3$ that for all $n\in[N]$
    \begin{align}
        \Bigl| \check{Z}^{(n)}_{n+1,d+1} - g_\pi(X,\mathcal{D}_n) \Bigr| = \Bigl| \bigl[\mathsf{FFN}_{\theta_{\mathrm{ffn}}^{(2)}}(\bar{Z}^{(n)})\bigr]_{n+1,d+1} - g_\pi(X,\mathcal{D}_n) \Bigr| \leq \frac{\sqrt{\epsilon}}{2}. \label{eq:tf-block-2}
    \end{align}
    For $n\in[N]$, let $\check{P}^{(n)}$ be the distribution of the last row of $\check{Z}^{(n)}$.  By applying Lemma~\ref{lemma:universal-approximation-of-NN} with $\mu = \frac{1}{N} \sum_{n=1}^N \check{P}^{(n)}$ therein, there exists an FFN layer $\mathsf{FFN}_{\theta_{\mathrm{ffn}}^{(3)}}$ such that the event 
    \begin{align*}
        \mathcal{E}_4 \coloneqq \Bigl\{ \max_{n\in[N]} \Bigl| \bigl[\mathsf{FFN}_{\theta_{\mathrm{ffn}}^{(3)}}(\check{Z}^{(n)})\bigr]_{n+1,d+1} - \mathsf{clip}_R(\check{Z}^{(n)}_{n+1,d+1}) \Bigr| \leq \sqrt{\epsilon}/4\Bigr\}
    \end{align*}
    has probability at least $1-\delta$. Finally, let $\mathsf{Attn}_{\theta^{(3)}_{\mathrm{attn}}}$ be the identity map and let 
    \[
    \mathsf{TF}\coloneqq \mathsf{FFN}_{\theta_{\mathrm{ffn}}^{(3)}}  \circ\mathsf{Attn}_{\theta^{(3)}_{\mathrm{attn}}} \circ \mathsf{FFN}_{\theta_{\mathrm{ffn}}^{(2)}}  \circ\mathsf{Attn}_{\theta^{(2)}_{\mathrm{attn}}} \circ \mathsf{FFN}_{\theta_{\mathrm{ffn}}^{(1)}} \circ\mathsf{Attn}_{\theta^{(1)}_{\mathrm{attn}}}.
    \]
    Since $\|g_\pi\|_{\infty} \leq R$, we have by~\eqref{eq:tf-block-2} that on $\mathcal{E} \coloneqq \mathcal{E}_0\cap \mathcal{E}_1\cap \mathcal{E}_2 \cap \mathcal{E}_3 \cap \mathcal{E}_4$, for all $n\in[N]$,
    \begin{align*}
    &\bigl|\mathsf{Read}\circ\mathsf{TF}(Z_{\mathrm{in}}^{(n,{d_{\mathrm{model}}})}) - g_\pi(X,\mathcal{D}_n)\bigr|\\
    &\leq \Bigl| \bigl[\mathsf{FFN}_{\theta_{\mathrm{ffn}}^{(3)}}(\check{Z}^{(n)})\bigr]_{n+1,d+1} - \mathsf{clip}_R(\check{Z}^{(n)}_{n+1,d+1}) \Bigr| + \Bigl| \mathsf{clip}_R(\check{Z}^{(n)}_{n+1,d+1}) - g_\pi(X,\mathcal{D}_n)\Bigr| \leq \frac{3\sqrt{\epsilon}}{4}.
    \end{align*}
    By a union bound, the event $\mathcal{E}$ has probability at least $1-5\delta$.  We conclude that 
    \begin{align*}
        \max_{n\in[N]}\mathbb{E}_{P\sim\pi, (\mathcal{D}_n,X,Y)|P \sim P^{\otimes (n+1)}}\Bigl\{\bigl(\mathsf{Read}\circ\mathsf{TF}(X,\mathcal{D}_n) - g_\pi(X,\mathcal{D}_n)\bigr)^2\Bigr\} \leq \frac{9\epsilon}{16} + 5R^2\delta \leq \epsilon,
    \end{align*}
    as required.
\end{proof}

\subsection{Proof of Proposition~\ref{prop:ERM-approx-posterior}} \label{sec:proof-of-learnability}

The main idea behind the proof of Proposition~\ref{prop:ERM-approx-posterior} is to use empirical process theory to upper bound $\mathcal{E}(\tilde{f}_T)$ by the approximation error and a complexity measure of $\mathcal{F}$. In the literature, the complexity measure is often taken to be the covering number with respect to the uniform norm, and upper bounds on this covering number for Transformers are available, provided the input matrix and parameters are bounded. There, the proof strategy is to argue that for any input matrix from a bounded set, the output of the Transformer with bounded parameters is uniformly Lipschitz in its parameters. However, this argument does not work in our setting, since our input matrix and parameters are unbounded. Instead, we use the pseudo-dimension of $\mathcal{F}$ (see the definition below) as a complexity measure and show that the pseudo-dimension of Transformers is finite.  This result (Lemma~\ref{lemma:finite-pdim} below) may be of independent interest.
\begin{defn}[VC-dimension and pseudo-dimension]
    Let $X$ be a set and let $\mathcal{C}$ be a collection of subsets of $X$. For $S\coloneqq\{x_1,\ldots,x_n\} \subseteq X$, we say that $\mathcal{C}$ shatters $S$ if for every $T\subseteq S$, there exists $C\in\mathcal{C}$ such that $C\cap S = T$. The \emph{VC-dimension} of $\mathcal{C}$ is the cardinality of the largest subset of $X$ that can be shattered by $\mathcal{C}$.

    Now let $\mathcal{G}$ be a set of functions from $\mathcal{X}$ to $\mathbb{R}$. The \emph{pseudo-dimension} $\mathrm{Pdim}(\mathcal{G})$ of $\mathcal{G}$ is the VC-dimension of its subgraph class $\bigl\{ \{(x,y) \in \mathcal{X}\times\mathbb{R} : y\leq g(x)\} : g\in\mathcal{G}\bigr\}$.
\end{defn}

Lemma~\ref{lemma:approx-error-plus-pdim} below controls $\mathcal{E}(\tilde{f}_T)$ in terms of the approximation error and pseudo-dimension of $\mathcal{F}$. 

\begin{lemma} \label{lemma:approx-error-plus-pdim}
    Under the setting of Proposition~\ref{prop:ERM-approx-posterior}, there exists $C>0$, depending only on $\|\xi\|_{\psi_2}$ and $R$, such that for $T\geq 2$,
    \begin{align*}
        \mathcal{E}(\tilde{f}_T) \leq 2\inf_{f\in\mathcal{F}} \mathbb{E}_{P\sim\pi,\, (\mathcal{D}_n,X,Y)|P \sim P^{\otimes(n+1)}} \Bigl\{ \bigl(f(X,\mathcal{D}_n) - g_{\pi}(X,\mathcal{D}_n)\bigr)^2 \Bigr\} + \frac{C\log^3 (T) \cdot \mathrm{Pdim}(\mathcal{F})}{T}.
    \end{align*}
\end{lemma}
\begin{proof}
    Let $P\sim\pi$, $(\mathcal{D}_n, X, Y) |P \sim P^{\otimes(n+1)}$ and let $(\mathcal{D}_n^{(t)}, X^{(t)}, Y^{(t)})_{t\in[T]}$ be independent copies of $(\mathcal{D}_n, X, Y)$.
    For $t\in[T]$, let $Z^{(t)} \coloneqq (X^{(t)},\mathcal{D}_n^{(t)})$ and let $Z\coloneqq (X,\mathcal{D}_n)$.  
    Now note that $g_{\pi}(Z) = \mathbb{E}(Y \,|\, Z)$,
    \begin{align*}
        \hat{f}_T \in \argmin_{f\in\mathcal{F}} \frac{1}{T}\sum_{t=1}^T \bigl(Y^{(t)} - f(Z^{(t)})\bigr)^2
    \end{align*}
    and
    \begin{align*}
        \mathcal{E}(\tilde{f}_T) = \mathbb{E}\bigl\{\bigl(Y - \mathsf{clip}_R \hat{f}_T(Z)\bigr)^2\bigr\} - \mathbb{E}\bigl\{\bigl(Y - g_{\pi}(Z)\bigr)^2\bigr\}.
    \end{align*}
    We have therefore reduced our problem to bounding the excess risk of a nonparametric regression estimator in a problem with bounded regression function and sub-Gaussian noise.  The conclusion now follows by applying, for example, \citet[Theorem~1]{ma2025deep} with the following minor modifications: (i) we ignore their $\Omega$ vectors since our data do not have missing values, (ii) by inspecting their proof, we may truncate $\hat{f}_T$ at level $R$ since this is a known bound on the regression function, (iii) the optimisation error term (the first term in their upper bound) is zero since we consider the empirical risk minimiser, and (iv) by inspecting their proof and using \citet[Theorem~9.4]{gyorfi2002distribution}, we can replace the third term in their upper bound by $\frac{C\log^3(T) \cdot \mathrm{Pdim}(\mathcal{F})}{T}$, where $C>0$ depends only on $\|\xi\|_{\psi_2}$ and $R$.
\end{proof}

Next we will show that the class of Transformers with fixed architecture has finite pseudo-dimension. The proof uses terminology and results from model theory, a branch of mathematical logic.  Since this field will probably be unfamiliar to most readers, we now provide a minimal glossary (sufficient for our proof) of the relevant terms; we refer the readers to \citet{marker2002model} for a formal introduction to model theory.  We will be concerned with an object called a \emph{structure}, which is a set equipped with a collection of distinguished functions, relations and elements that together give meaning to the symbols of a formal language.  Our structure of interest is $\mathbb{R}_{\exp,\Phi}\coloneqq (\mathbb{R}; 0,1,+,-,\cdot,<,\exp,\Phi)$, where $\cdot$ refers to multiplication, $\exp:\mathbb{R}\to(0,\infty)$ refers to the exponential function and $\Phi:\mathbb{R}\to[0,1]$ refers to the distribution function of the standard Gaussian distribution. Our goal is to show that Transformers can be defined within $\mathbb{R}_{\exp,\Phi}$, in a sense that we will make precise below.  

Given a countable set $\mathcal{V}$ of variable symbols, which we intrepret as syntactic placeholders, the set of \emph{terms} in the structure $\mathbb{R}_{\exp,\Phi}\coloneqq (\mathbb{R}; 0,1,+,-,\cdot,<,\exp,\Phi)$ is the smallest set~$\mathcal{T}$ such that (i) $0,1\in\mathcal{T}$, (ii) $\mathcal{V} \subseteq \mathcal{T}$, (iii) if $t_1,t_2\in\mathcal{T}$, then $t_1+t_2,\, t_1-t_2,\, t_1\cdot t_2 \in \mathcal{T}$, and (iv) if $t\in\mathcal{T}$, then $\exp(t),\Phi(t) \in \mathcal{T}$. We say that $\varphi$ is an \emph{atomic formula} in $\mathbb{R}_{\exp,\Phi}$ if $\varphi$ is either $t_1=t_2$ or $t_1<t_2$, for some $t_1,t_2\in\mathcal{T}$. The set of \emph{(first-order) formulas} $\mathcal{W}$ in (the language of) $\mathbb{R}_{\exp,\Phi}$ is the smallest set such that (i) every atomic formula is in $\mathcal{W}$, (ii) if $\varphi,\psi\in\mathcal{W}$, then $(\neg \varphi), (\varphi\vee\psi), (\varphi \wedge \psi) \in \mathcal{W}$, where the logic symbolds $\neg,\vee,\wedge$ stand for negation, disjunction and conjunction, and (iii) if $\varphi\in\mathcal{W}$ and $v \in \mathcal{V}$, then $\exists v \;\varphi$ and $\forall v \;\varphi$ are in $\mathcal{W}$. For a formula $\varphi\in\mathcal{W}$, we say that a variable $v$ is \emph{free} in $\varphi$ if it is not inside a $\exists v$ or $\forall v$ quantifier; e.g.~$v_1$ is free in the formula $(v_1=0) \vee (v_1>0)$, but it is not free in the formula $\exists v_1 \; v_2\cdot v_2 = v_1$. When a formula $\varphi$ has free variables $v_1,\ldots,v_n$, we will often write $\varphi(v_1,\ldots,v_n)$ to make the free variables explicit. Now, letting $\varphi(v_1,\ldots,v_n)\in\mathcal{W}$ be a formula and $a=(a_1,\ldots,a_n)\in\mathbb{R}^n$, we write $\mathbb{R}_{\exp,\Phi} \models \varphi(a)$ (we read `$\mathbb{R}_{\exp,\Phi}$ entails $\varphi(a)$' or `$\mathbb{R}_{\exp,\Phi}$ models $\varphi(a)$') if $\varphi(a)$ is true in $\mathbb{R}_{\exp,\Phi}$, and otherwise we write $\mathbb{R}_{\exp,\Phi} \not\models \varphi(a)$. For example, if $\varphi(v_1,v_2)$ is the formula $(v_1=v_2) \vee (v_1+v_2=0)$, then $\mathbb{R}_{\exp,\Phi} \models \varphi(1,1)$ and $\mathbb{R}_{\exp,\Phi} \models \varphi(-1,1)$, but $\mathbb{R}_{\exp,\Phi} \not\models \varphi(1,2)$.

For $d\in\mathbb{N}$ and a set $A\subseteq \mathbb{R}^d$, we say that $A$ is \emph{definable} (in $\mathbb{R}_{\exp,\Phi}$) if there exist $m\in\mathbb{N}$, a formula $\varphi(v_1,\ldots,v_d,w_1,\ldots,w_m)$ in $\mathbb{R}_{\exp,\Phi}$ and $b\in\mathbb{R}^m$ such that $A=\{a\in\mathbb{R}^d : \mathbb{R}_{\exp,\Phi} \models \varphi(a,b)\}$; we will sometimes abbreviate the notation and write $A=\{a\in\mathbb{R}^d : \varphi(a,b)\}$ instead. For $d_1,d_2\in\mathbb{N}$ and a function $f:\mathbb{R}^{d_1} \to \mathbb{R}^{d_2}$, we say that $f$ is \emph{definable} (in $\mathbb{R}_{\exp,\Phi}$) if its \emph{graph} $\{(x,y)\in \mathbb{R}^{d_1}\times \mathbb{R}^{d_2} : y=f(x)\}$ is a definable set in $\mathbb{R}_{\exp,\Phi}$. Moreover, for a collection $\mathcal{C}$ of subsets of $\mathbb{R}^d$, we say that $\mathcal{C}$ is \emph{uniformly definable} (in $\mathbb{R}_{\exp,\Phi}$) if there exists a formula $\varphi(v,w)$, where $v=(v_1,\ldots,v_d)$ and $w=(w_1,\ldots,w_m)$ for some $m\in\mathbb{N}$, such that for every $C\in\mathcal{C}$, there exists $b_C\in\mathbb{R}^m$ with $C=\{a\in\mathbb{R}^d : \varphi(a,b_C)\}$. Similarly, for a collection $\mathcal{F}$ of functions from $\mathbb{R}^{d_1}$ to $\mathbb{R}^{d_2}$, we say that $\mathcal{F}$ is \emph{uniformly definable} (in $\mathbb{R}_{\exp,\Phi}$) if the collection of graphs of functions in $\mathcal{F}$ is uniformly definable.

\begin{lemma}\label{lemma:properties-of-definable-functions} 
\begin{itemize}
    \item[(a)] Let $\mathcal{F}$ be a collection of uniformly definable functions from $\mathbb{R}^{d_1}$ to $\mathbb{R}$, and let $\mathcal{G}$ be a collection of uniformly definable functions from $\mathbb{R}^{d_2}$ to $\mathbb{R}$. Then $\bigl\{ \{(x,y)\in\mathbb{R}^{d_1} \times \mathbb{R} : y\leq f(x)\} : f\in\mathcal{F} \bigr\}$, $\{(x,z) \mapsto f(x)+g(z) : f\in\mathcal{F}, g\in\mathcal{G}\}$ and $\{(x,z) \mapsto f(x)\cdot g(z) : f\in\mathcal{F}, g\in\mathcal{G}\}$ are all uniformly definable.

    \item[(b)] Let $\mathcal{F}$ be a collection of uniformly definable functions from $\mathbb{R}^{d}$ to $\mathbb{R}\setminus \{0\}$. Then $\{x\mapsto 1/f(x) : f\in\mathcal{F}\}$ is uniformly definable.

    \item[(c)] Let $p\in\mathbb{N}$ and for $j\in[p]$, let $\mathcal{G}_j$ be a collection of uniformly definable functions from $\mathbb{R}^{d_j}$ to $\mathbb{R}$ for some $d_j\in\mathbb{N}$. Further let $\mathcal{F}$ be a collection of uniformly definable functions from $\mathbb{R}^{p}$ to $\mathbb{R}$. Then $\bigl\{(x_1,\ldots,x_p) \mapsto f\bigl(g_1(x_1), \ldots, g_p(x_p)\bigr) : f\in\mathcal{F}, g_j\in\mathcal{G}_j \,\forall j\in[p]\bigr\}$ is uniformly definable.
\end{itemize}
\end{lemma}
\begin{proof}
    \emph{(a)} Suppose that for each $f\in\mathcal{F}$, the graph of $f$ is $\{(x,y)\in \mathbb{R}^{d_1}\times\mathbb{R} : \varphi(x,y,b_f)\}$ for some formula $\varphi$, some $b_f\in\mathbb{R}^m$ and some $m\in\mathbb{N}$; also suppose that for each $g\in\mathcal{G}$, the graph of $g$ is $\{(z,y)\in \mathbb{R}^{d_2}\times\mathbb{R} : \psi(z,y,c_g)\}$ for some formula $\psi$, some $c_g\in\mathbb{R}^n$ and some $n\in\mathbb{N}$. The subgraph $\{(x,y)\in \mathbb{R}^{d_1}\times\mathbb{R} : y\leq f(x)\}$ of $f\in\mathcal{F}$ can be written as $\bigl\{(x,y) \in \mathbb{R}^{d_1}\times\mathbb{R} : \exists t\; \bigl(\varphi(x,t,b_f) \wedge (y<t \vee y=t)\bigr)\bigr\}$. Hence the collection of subgraphs of $f\in\mathcal{F}$ is uniformly definable. For $f\in\mathcal{F}$ and $g\in\mathcal{G}$, the graph of the function $(x,z) \mapsto f(x)+g(z)$ can be written as $\bigl\{(x,z,y) \in \mathbb{R}^{d_1} \times \mathbb{R}^{d_2} \times \mathbb{R} : \exists u,v \; \bigl(\varphi(x,u,b_f) \wedge \psi(z,v,c_g) \wedge (y=u+v)\bigr)\bigr\}$.  Finally, the graph of the function $(x,z) \mapsto f(x)\cdot g(z)$ can be written as $\bigl\{(x,z,y) \in \mathbb{R}^{d_1} \times \mathbb{R}^{d_2} \times \mathbb{R} : \exists u,v \; \bigl(\varphi(x,u,b_f) \wedge \psi(z,v,c_g) \wedge (y=u\cdot v)\bigr) \bigr\}$. 

    \medskip
    \emph{(b)} Again, suppose that for each $f\in\mathcal{F}$, the graph of $f$ is $\{(x,y)\in \mathbb{R}^{d}\times\mathbb{R} : \varphi(x,y,b_f)\}$ for some formula $\varphi$, some $b_f\in\mathbb{R}^m$ and some $m\in\mathbb{N}$. Then, for $f\in\mathcal{F}$, the graph of the function $x\mapsto 1/f(x)$ can be written as $\bigl\{(x,y)\in \mathbb{R}^{d}\times\mathbb{R} : \exists u\; \bigl(\varphi(x,u,b_f) \wedge (y\cdot u=1)\bigr)\bigr\}$. This proves the claim.

    \medskip
    \emph{(c)} Suppose that for each $f\in\mathcal{F}$, the graph of $f$ is $\{(x,y)\in \mathbb{R}^{p}\times\mathbb{R} : \varphi(x,y,b_f)\}$ for some formula $\varphi$, some $b_f\in\mathbb{R}^m$ and some $m\in\mathbb{N}$; also suppose that for $j\in[p]$ and $g_j\in\mathcal{G}_j$, the graph of $g_j$ is $\{(z,y)\in \mathbb{R}^{d_j}\times\mathbb{R} : \psi_j(z,y,c_{g_j})\}$ for some formula $\psi_j$, some $c_{g_j}\in\mathbb{R}^{n_j}$ and some $n_j\in\mathbb{N}$. Then the graph of the function $(x_1,\ldots,x_p) \mapsto f\bigl(g_1(x_1), \ldots, g_p(x_p)\bigr)$ can be written as 
    \begin{align*}
        \biggl\{(x_1,\ldots,x_p,y) \in \mathbb{R}^{p+1} : \exists v_1,\ldots,v_p \biggl( \bigwedge_{j=1}^p \psi_j(x_j,v_j,c_{g_j}) \wedge \varphi(v_1,\ldots,v_p,y,b_f) \biggr)\biggr\}.
    \end{align*}
    This proves the claim.
\end{proof}

\begin{lemma}\label{lemma:finite-pdim}
    For any $L,H,{d_{\mathrm{model}}},d_{\mathrm{ffn}}\in\mathbb{N}$, the class $\mathcal{F}' \coloneqq \mathsf{Read} \circ \mathcal{F}_{\mathrm{TF}}(L,H,{d_{\mathrm{model}}},d_{\mathrm{ffn}})$ %\mathsf{Read} \circ \mathcal{F}_{\mathrm{TF}}(L,H,{d_{\mathrm{model}}},d_{\mathrm{ffn}})$, as functions from $\mathbb{R}^{(n+1)\times {d_{\mathrm{model}}}}$ to $\mathbb{R}$, 
    has finite pseudo-dimension. 
\end{lemma}
\begin{remark}
    It is possible to obtain upper bounds on the pseudo-dimension of~$\mathcal{F}'$ with polynomial dependence on $L$, $H$, ${d_{\mathrm{model}}}$ and $d_{\mathrm{ffn}}$ via a more detailed analysis, but a finite pseudo-dimension is already enough for our purpose.
\end{remark}
\begin{proof}
    Note that $\mathcal{F}'$ is defined on the space $\bigcup_{n=1}^N \mathbb{R}^{(n+1)\times {d_{\mathrm{model}}}}$, equipped  with the disjoint union topology. Thus, it suffices to show that, for each $n \in [N]$, the restriction $\mathcal{F}'_n$ of the functions in $\mathcal{F}'$ to $\mathbb{R}^{(n+1)\times {d_{\mathrm{model}}}}$ has finite pseudo-dimension. 
%Let $V\in\mathbb{N}$ be the total number of parameters of a Transformer in $\mathcal{F}_{\mathrm{TF}}(L,H,{d_{\mathrm{model}}},d_{\mathrm{ffn}})$. Let $\mathsf{TF}_{\theta}$ be the Transformer in $\mathcal{F}_{\mathrm{TF}}(L,H,{d_{\mathrm{model}}},d_{\mathrm{ffn}})$ with parameters given by $\theta$ and let $f_{\theta} \coloneqq \mathsf{Read} \circ \mathsf{TF}_{\theta}$. Then we have $\mathcal{F}' = \{f_\theta: \theta\in\mathbb{R}^V\}$.
    Now let 
    \begin{align*}
        \mathcal{C}_n \coloneqq \Bigl\{ \bigl\{(Z,y) \in \mathbb{R}^{(n+1)\times {d_{\mathrm{model}}}} \times \mathbb{R} : y \leq f(Z)\bigr\} : f\in \mathcal{F}'_n \Bigr\} 
    \end{align*}
    denote the collection of subgraphs of $f \in \mathcal{F}'_n$. Then by definition, the pseudo-dimension of $\mathcal{F}'_n$ is equal to the VC-dimension of $\mathcal{C}_n$.  We will now show that $\mathcal{C}_n$ is uniformly definable in $\mathbb{R}_{\exp,\Phi}$.  The graph $\{(x,y)\in\mathbb{R}^2 : y=\relu(x)\}$ of $\relu$ can be written as
    \begin{align*}
        \bigl\{(x,y)\in\mathbb{R}^2 : \bigl(\neg(x<0) \wedge y=x\bigr) \vee (x<0 \wedge y=0)\bigr\},
    \end{align*}
    so $\relu : \mathbb{R} \to \mathbb{R}$ is definable on $\mathbb{R}_{\exp,\Phi}$. By Lemma~\ref{lemma:properties-of-definable-functions}, sums, products, reciprocals and compositions of definable functions are definable. Thus, for any $d,p\in\mathbb{N}$, the set of polynomials on $\mathbb{R}^d$ with degree at most $p$ is uniformly definable in $\mathbb{R}_{\exp,\Phi}$. Moreover, since $\exp$ and $\Phi$ are definable in $\mathbb{R}_{\exp,\Phi}$ by definition, we deduce that $\mathsf{GELU}$, $\mathsf{SiLU}$, and $\softmax$ are also definable in $\mathbb{R}_{\exp,\Phi}$. Each function $f\in\mathcal{F}'_n$ comprises compositions, sums and products of $\rho\in\{\mathsf{ReLU}, \mathsf{GELU}, \mathsf{SiLU}\}$, $\softmax$ and polynomials of degree at most two, we deduce by Lemma~\ref{lemma:properties-of-definable-functions} that $\mathcal{F}'_n$ is uniformly definable in $\mathbb{R}_{\exp,\Phi}$. Further, by Lemma~\ref{lemma:properties-of-definable-functions}\emph{(a)}, we have that the subgraph class of $\mathcal{F}'_n$ is also uniformly definable on $\mathbb{R}_{\exp,\Phi}$ and
    \begin{align*}
        \mathcal{C}_n \subseteq \Bigl\{\bigl\{a\in\mathbb{R}^{(n+1)\times {d_{\mathrm{model}}}} \times \mathbb{R} : \mathbb{R}_{\exp,\Phi} \models \varphi(a,b)\bigr\} : b\in\mathbb{R}^m\Bigr\}.
    \end{align*}
    By \citet[Therorem~3.4.37]{marker2002model}, the structure $\mathbb{R}_{\exp}\coloneqq (\mathbb{R}; 0,1,+,-,\cdot,<,\exp,\Phi)$ is \emph{o-minimal}, i.e.~every definable set is a finite union of open intervals and singleton points.  Further, the derivative $\Phi'$ given by $\Phi'(x) = \frac{1}{\sqrt{2\pi}}e^{-x^2/2}$ is definable in $\mathbb{R}_{\exp}$, so by \citet{Speissegger1999pfaffian}, there exists an o-minimal extension $\tilde{\mathbb{R}}$ of $\mathbb{R}_{\exp}$ in which $\Phi$ is definable. Now, every definable set in $\mathbb{R}_{\exp,\Phi}$ is definable in $\tilde{\mathbb{R}}$, so $\mathbb{R}_{\exp,\Phi}$ is o-minimal.
    We may therefore apply \citet[Theorem~3.2]{laskowski1992vapnik} \citep[see also][Theorem~2]{karpinski1995polynomial} to deduce that the VC-dimension of $\mathcal{C}_n$ is finite.
\end{proof}

\begin{proof}[Proof of Proposition~\ref{prop:ERM-approx-posterior}]
    By Theorem~\ref{thm:universal-approximation}, there exist $d_{\mathrm{model}}^\circ\geq d+2$ and $d_{\mathrm{ffn}}^\circ\in\mathbb{N}$ such that if ${d_{\mathrm{model}}}\geq d_{\mathrm{model}}^\circ$, $d_{\mathrm{ffn}} \geq d_{\mathrm{ffn}}^\circ$ and $\mathcal{F} \coloneqq \mathsf{Read} \circ \mathcal{F}_{\mathrm{TF}}(L,H,{d_{\mathrm{model}}},d_{\mathrm{ffn}})$, then 
    \begin{align*}
        \inf_{f\in\mathcal{F}} \mathbb{E}_{P\sim\pi,\, (\mathcal{D}_n,X,Y)|P \sim P^{\otimes(n+1)}} \Bigl\{ \bigl(f(X,\mathcal{D}_n) - g_{\pi}(X,\mathcal{D}_n)\bigr)^2 \Bigr\} \leq \frac{\epsilon}{4},
    \end{align*}
    Now by Lemma~\ref{lemma:finite-pdim}, we have that $\mathrm{Pdim}(\mathcal{F})$ is finite. Thus, for all $T$ sufficiently large, we have $\frac{C\log^3(T) \cdot \mathrm{Pdim}(\mathcal{F})}{T} \leq \frac{\epsilon}{2}$, where $C>0$ is taken from Lemma~\ref{lemma:approx-error-plus-pdim}. The claim then follows from Lemma~\ref{lemma:approx-error-plus-pdim}. 
\end{proof}

\section{Function spaces and wavelets} \label{sec:function-spaces-and-wavelets}
\subsection{Function spaces} \label{sec:function-spaces}

In this section, we restrict the domain of the functions to $[0,1]^d$.  By Jensen's inequality, $\|f\|_{L^{p_1}} \leq \|f\|_{L^{p_2}}$ for all $1\leq p_1 \leq p_2 \leq \infty$.
For $f,g\in L^2([0,1]^d)$, define the inner product
\begin{align*}
    \langle f,g \rangle \coloneqq \int_{[0,1]^d} f(x)g(x) \,\d x.
\end{align*}

Let $C([0,1]^d)$ denote the set of continuous functions from $[0,1]^d$ to $\mathbb{R}$. For $m=(m_1,\ldots,m_d)^\top \in\mathbb{N}_0^d$ with $k\coloneqq \|m\|_1$, and a $k$-times differentiable $f:[0,1]^d \rightarrow \mathbb{R}$, we use the multi-index notation $\partial^m f \coloneqq \partial^{m_1}\cdots\partial^{m_d}f$ for partial derivatives. 
\begin{defn}[H\"older space]\label{defn:holder-space}
    For $\alpha>0$, let $\alpha_0\coloneqq \lceil \alpha \rceil - 1$ be the largest integer strictly smaller than $\alpha$ and define the \emph{H\"older space} $H^{\alpha}([0,1]^d)$ to be the set of $\alpha_0$-times differentiable $f:[0,1]^d \to \mathbb{R}$ with
    \begin{align*}
    \sup_{x\neq y} \sum_{m\in \mathbb{N}_0^d : \|m\|_1 = \alpha_0} \frac{\partial^m f(x) - \partial^m f(y)}{\|x-y\|_2^{\alpha-\alpha_0}} < \infty. 
    \end{align*}
\end{defn}

For $r\in\mathbb{N}$, define the \emph{$r$th order difference operator} $\Delta_h^r:L^1([0,1]^d) \rightarrow L^1([0,1]^d)$ by
\begin{align*}
    \Delta_h^r(f)(x) = \mathbbold{1}_{\{x+rh\in[0,1]^d\}}\sum_{k=0}^r \binom{r}{k} (-1)^{r-k} f(x+kh). 
\end{align*}
Note that $\Delta_h^1(f)(x)=\{f(x+h) - f(x)\}\mathbbold{1}_{\{x+h\in[0,1]^d\}}$ and $\Delta_h^r(f)(x) = \Delta_h^1 \bigl(\Delta_h^{r-1}(f)(x)\bigr)$ if $x+rh\in[0,1]^d$.  Now, for $p\in[1,\infty]$, $f\in L^p([0,1]^d)$ and $t \geq 0$, define the \emph{$r$th modulus of smoothness}
\begin{align*}
    \omega_{r,p}(f,t) \coloneqq \sup_{\|h\|_2 \leq t} \|\Delta_h^r(f)\|_{L^p}.
\end{align*}

\begin{defn}[Besov space]\label{defn:besov-space}
    For $p,q \in [1,\infty]$ and $\alpha>0$, let $r\coloneqq \lfloor \alpha \rfloor+1$ and define the \emph{Besov space}  $B_{p,q}^{\alpha}([0,1]^d)$ by
    \begin{align*}
        B_{p,q}^{\alpha}([0,1]^d) \coloneqq \begin{cases}
            {\displaystyle \biggl\{f\in L^p([0,1]^d) : \biggl[\int_0^{\infty} \bigl(t^{-\alpha} \omega_{r,p}(f,t)\bigr)^q \frac{\d t}{t}\biggr]^{1/q} < \infty\biggr\}} \quad & \text{for } q \in [1,\infty) \\[1em] 
            {\displaystyle \Bigl\{f\in L^p([0,1]^d) : \sup_{t>0} t^{-\alpha} \omega_{r,p}(f,t) < \infty \Bigr\}} \quad & \text{for } q=\infty.
        \end{cases}
    \end{align*}
\end{defn}
For non-integer smoothness levels, H\"older spaces are special cases of Besov spaces in the sense that $H^{\alpha}([0,1]^d) = B^{\alpha}_{\infty,\infty}([0,1]^d)$ for all $\alpha\in (0,\infty)\setminus \mathbb{N}$, while for $\alpha\in\mathbb{N}$, we have $H^{\alpha}([0,1]^d) \subseteq B^{\alpha}_{\infty,\infty}([0,1]^d)$, and in fact, $B_{\infty,\infty}^{\alpha}([0,1]^d)$ coincides with Zygmund spaces when $\alpha\in\mathbb{N}$; see e.g.~\citet[p.~351]{gine2021mathematical} or \citet[Chapter~2.5.7]{triebel1983theory}.

\subsection{Wavelets}
\label{subsec:wavelets}
We start with the univariate setting where the notation is simpler.  For $S\in\mathbb{N}$, denote the Cohen–Daubechies–Vial (CDV) wavelet basis on $[0,1]$ \citep{cohen1993wavelets}, constructed from $S$-regular and $S$-times continuously differentiable father and mother Daubechies wavelets \citep{daubechies1988orthonormal}, by
\begin{align}
    \bigl\{\phi_{k}: k\in[0:2^{\ell_0}-1]\bigr\} \bigcup \bigl\{\psi_{\ell,k}: \ell\in[{\ell_0}:\infty),\, k\in[0:2^{\ell}-1]\bigr\}, \label{eq:one-dim-CDV}
\end{align}
where $\ell_0\in\mathbb{N}$ depends only on $S$.  The precise definition and construction of this wavelet basis is not important for us, but the functions in~\eqref{eq:one-dim-CDV} are $S$-times continuously differentiable and form an orthonormal basis for $L^2([0,1])$ \citep[Theorem~4.4]{cohen1993wavelets}. Moreover, $\phi_k$ is supported on $\bigl[\frac{k}{2^{\ell_0}}, \frac{k+1}{2^{\ell_0}}\bigr]$ for $k\in[0:2^{\ell_0}-1]$, and $\psi_{\ell,k}$ is supported on $\bigl[\frac{k}{2^{\ell}}, \frac{k+1}{2^{\ell}}\bigr]$ for $\ell\in[{\ell_0}:\infty),\, k\in[0:2^{\ell}-1]$.
The one-dimensional CDV wavelets in~\eqref{eq:one-dim-CDV} can be extended to an orthonormal basis for $L^2([0,1]^d)$ via tensor products. For $k=(k_1,\ldots,k_d)^\top \in [0:2^{\ell_0}-1]^d$ and $x=(x_1,\ldots,x_d)^\top \in [0,1]^d$, define
\begin{align}
    \Phi_k(x) \coloneqq \prod_{j=1}^d \phi_{k_j}(x_j). \label{eq:tensor-product-phi}
\end{align}
Further, for $\tau=(\tau_1,\ldots,\tau_d)^\top \in\{0,1\}^d\setminus \{0_d\}$, $\ell\in[{\ell_0}:\infty)$, $k=(k_1,\ldots,k_d)^\top \in [0:2^\ell-1]^d$ and $x=(x_1,\ldots,x_d)^\top \in [0,1]^d$, define
\begin{align}
    \Psi_{\ell,k}^{(\tau)}(x) \coloneqq \prod_{j=1}^d \psi^{(\tau_j)}_{\ell,k_j}(x_j), \label{eq:tensor-product-psi}
\end{align}
where $\psi^{(1)}_{\ell,k'} = \psi_{\ell,k'}$ for all $k'\in[0:2^{\ell}-1]$, and $(\psi^{(0)}_{\ell,k'})_{k'=0}^{2^{\ell}-1}$ are the boundary-corrected ($S$-regular and $S$-times continuously differentiable) father Daubechies wavelets at resolution~$\ell$. Again, the precise definition and construction are not important for us, but we will exploit the following properties. To simplify notation, define
\begin{align}
    K\coloneqq [0:2^{\ell_0}-1]^d \quad\text{and}\quad \Gamma_\ell \coloneqq [0:2^\ell-1]^d \times \bigl(\{0,1\}^d\setminus\{0_d\}\bigr) \text{ for } \ell\in[\ell_0:\infty), 
\end{align}
and for $\gamma = (k,\tau) \in \Gamma_{\ell}$, we write $\Psi_{\ell,\gamma} \coloneqq \Psi_{\ell,k}^{(\tau)}$. Note that
\begin{align}
    |K| = 2^{\ell_0 d} \quad\text{and}\quad |\Gamma_{\ell}| = 2^{\ell d} (2^d-1) < 2^{(\ell+1)d} \text{ for }\ell\in[\ell_0:\infty). \label{eq:number-of-wavelets}
\end{align}
Each term in the products in~\eqref{eq:tensor-product-phi} and~\eqref{eq:tensor-product-psi} is $S$-times continuously differentiable, and the functions
\begin{align*}
    \bigl\{\Phi_k : k\in K\bigr\} \cup \bigl\{\Psi_{\ell,\gamma} : \ell\in[{\ell_0}:\infty),\, \gamma\in\Gamma_{\ell}\bigr\}
\end{align*}
form an orthonormal basis for $L^2([0,1]^d)$; see e.g. \citet[Chapter~4.3.6]{gine2021mathematical} and \citet[Section~4]{dahmen1997wavelet}. By construction, there exist $C,C'>0$ (depending only on $S$ and $d$) such that
\begin{align}
    \|\Phi_k\|_{\infty} \leq C 2^{\ell_0 d/2} \quad\text{and}\quad \|\Psi_{\ell,\gamma}\|_{\infty} \leq C 2^{\ell d/2} \label{eq:ell-infty-norm-wavelet-functions}
\end{align}
for $k\in K$, $\ell\in[\ell_0:\infty)$ and $\gamma\in\Gamma_{\ell}$, and that 
\begin{align}
    \sum_{k\in K} \mathbbold{1}_{\{\Phi_{k}(x) \neq 0\}} \vee \sup_{\ell\in[\ell_0:\infty)} \sum_{\gamma\in\Gamma_{\ell}} \mathbbold{1}_{\{\Psi_{\ell,\gamma}(x) \neq 0\}} \leq C' \quad\text{for all }x\in[0,1]^d. \label{eq:number-of-non-zero-wavelets}
\end{align}
For $f\in L^2([0,1]^d)$, we have the wavelet decomposition
\begin{align}
    f = \sum_{k\in K}a_k \Phi_k + \sum_{\ell=\ell_0}^{\infty} \sum_{\gamma\in\Gamma_{\ell}} b_{\ell,\gamma} \Psi_{\ell,\gamma}, \label{eq:wavelet-decomposition-d}
\end{align}
where $a_k\coloneqq \langle f,\Phi_k\rangle$ and $b_{\ell,\gamma}\coloneqq \langle f,\Psi_{\ell,\gamma}\rangle$. Define
\begin{align}
    a\coloneqq (a_k)_{k\in K} \in \mathbb{R}^{K} \quad\text{and}\quad b_{\ell}\coloneqq (b_{\ell,\gamma})_{\gamma\in\Gamma_{\ell}} \in \mathbb{R}^{\Gamma_{\ell}} \text{ for }\ell\in[\ell_0:\infty). \label{eq:def-a-b-d}
\end{align}
For $f\in L^2([0,1]^d)$, $p\in[2, \infty]$, $q \in [1,\infty]$ and $\alpha\in(0,S)$, define its Besov norm by
\begin{align*}
    \|f\|_{B_{p,q}^{\alpha}} \coloneqq 
    \begin{cases}
        \displaystyle 2^{\ell_0(d/2-d/p)}\|a\|_p + \biggl\{\sum_{\ell={\ell_0}}^{\infty} \bigl(2^{\ell(\alpha+d/2-d/p)}\|b_{\ell}\|_p\bigr)^q \biggr\}^{1/q} \quad &\text{for } q \in [1,\infty)\\
        \displaystyle 2^{\ell_0(d/2-d/p)}\|a\|_p + \sup_{\ell\in[\ell_0:\infty)} 2^{\ell(\alpha+d/2-d/p)}\|b_{\ell}\|_p \quad &\text{for } q=\infty.
    \end{cases}
\end{align*}
For $p\in[2, \infty]$, $q \in [1,\infty]$ and $\alpha\in(0,S)$, the Besov space $B_{p,q}^{\alpha}([0,1]^d)$ in Definition~\ref{defn:besov-space} can be equivalently characterised by wavelet coefficients in the sense that $f\in B_{p,q}^{\alpha}([0,1]^d)$ if and only if $f\in L^2([0,1]^d)$ and $\|f\|_{B_{p,q}^{\alpha}} < \infty$ \citep[e.g.][Chapter~4.3.6]{gine2021mathematical}.  Therefore, we can define a closed ball $B_{p,q}^{\alpha}([0,1]^d, R)$ of radius $R>0$ in the space $B_{p,q}^{\alpha}([0,1]^d)$ by 
\begin{align}
    B_{p,q}^{\alpha}([0,1]^d,R) \coloneqq \{f\in L^2([0,1]^d) : \|f\|_{B_{p,q}^{\alpha}}\leq R\}. \label{eq:besov-ball}
\end{align}
As a special case, we have $\|f\|_{B_{\infty,\infty}^{\alpha}} = 2^{\ell_0 d/2}\|a\|_{\infty} + \sup_{\ell\in[\ell_0:\infty)} 2^{\ell(\alpha+d/2)}\|b_{\ell}\|_{\infty}$.
%, and we can define an $\alpha$-H\"older ball $H^{\alpha}([0,1]^d,R) \coloneqq B_{\infty,\infty}^{\alpha}([0,1]^d,R)$.

\section{Posterior contraction and adaptation}
\subsection{General theory}

This section is primarily based on material from \citet{ghosal2008nonparametric}; see also \citet[Section~10.2]{ghosal2017fundamentals}. Here, we make some modifications in order to obtain finite-sample posterior contraction at an exponential rate.

Let $\mathcal{Z}$ be a measurable space and let $\mathcal{A}$ be a countable index set. For $\alpha\in \mathcal{A}$, let~$\mathcal{P}_{\alpha}$ be a space of distributions on $\mathcal{Z}$ dominated by a $\sigma$-finite measure $\nu$, and let $\mathcal{P} \coloneqq \bigcup_{\alpha\in\mathcal{A}} \mathcal{P}_{\alpha}$, which is a metric space when equipped with the Hellinger metric $d_{\mathrm{H}}$.  For example, $\alpha>0$ might index smoothness, in which case~$\mathcal{P}_{\alpha}$ might denote the set of distributions of $(X,Y)$ for which the regression function $x \mapsto \mathbb{E}(Y\,|\,X=x)$ is $\alpha$-smooth. For $\alpha\in\mathcal{A}$, let $\pi_{\alpha}$ be a distribution on $\mathcal{P}_{\alpha}$ and let $\lambda=(\lambda_{\alpha})_{\alpha\in\mathcal{A}}$ be such that $\lambda_{\alpha}>0$ for $\alpha\in\mathcal{A}$ and $\sum_{\alpha\in\mathcal{A}}\lambda_{\alpha}=1$. We consider a prior distribution $\pi$ on $\mathcal{P}$ of the form
\begin{align*}
    \pi \coloneqq \sum_{\alpha\in\mathcal{A}} \lambda_{\alpha} \pi_{\alpha}.
\end{align*}
For $z_1,\ldots,z_n \in \mathcal{Z}$, we may define a distribution on $\mathcal{P}$ by 
\begin{align*}
    \pi(B \,|\, z_1,\ldots,z_n) \coloneqq \frac{\int_B \prod_{i=1}^n \frac{\d P}{\d \nu}(z_i) \,\d \pi(P)}{\int_{\mathcal{P}} \prod_{i=1}^n \frac{\d P}{\d \nu}(z_i) \,\d \pi(P)},
\end{align*}
for any measurable $B\subseteq \mathcal{P}$.  Thus, $\pi(\cdot \,|\, Z_1,\ldots,Z_n)$ would be the posterior distribution when $P$ has prior distribution $\pi$ and $Z_1,\ldots,Z_n\,|\,P \stackrel{\mathrm{iid}}{\sim} P$.  In the sequel, however, we will assume instead that $Z_1,\ldots,Z_n \stackrel{\mathrm{iid}}{\sim} P_0$ for some fixed distribution $P_0$ on $\mathcal{Z}$, and seek bounds on the expected mass assigned by $\pi(\cdot \,|\, Z_1,\ldots,Z_n)$ to (Hellinger) balls around~$P_0$.

For $\alpha\in\mathcal{A}$, a distribution $P_0$ on $\mathcal{Z}$ dominated by $\nu$, and for $\epsilon > 0$, define
\begin{align}
    B_{\alpha}(P_0,\epsilon) \coloneqq \Bigl\{P\in\mathcal{P}_{\alpha}: \mathrm{KL}(P_0,P) \leq \epsilon^2,\, \mathrm{V}_2(P_0,P) \leq \epsilon^2 \Bigr\} \label{eq:B-ball}
\end{align}
and
\begin{align}
    B'_{\alpha}(P_0,\epsilon) \coloneqq \Bigl\{P\in\mathcal{P}_{\alpha}: d_{\mathrm{H}}(P_0,P) \leq \epsilon \Bigr\}. \label{eq:C-ball}
\end{align}
Let $(\epsilon_{n,\alpha})_{\alpha\in\mathcal{A}}$ be a positive sequence, thought of as the rate of convergence in Hellinger distance on each $\mathcal{P}_{\alpha}$. Fix $\beta\in\mathcal{A}$ and $H\geq 1$, and define
\begin{align*}
    \mathcal{A}_{\geq\beta} &\coloneqq \bigl\{\alpha\in\mathcal{A} : \epsilon_{n,\alpha}^2 \leq H\epsilon_{n,\beta}^2 \bigr\}\\
    \mathcal{A}_{<\beta} &\coloneqq \bigl\{\alpha\in\mathcal{A} : \epsilon_{n,\alpha}^2 > H\epsilon_{n,\beta}^2 \bigr\},
\end{align*}
to be the sets of indices that are `at least as regular as $\beta$' and `less regular than $\beta$', respectively. Here $\beta$ should be thought of as the index of the `best model' for $P_0$.

We now state five conditions for our posterior contraction theory: 
\begin{enumerate}[label=(C\arabic*)]
    \item \label{eq:entropy-condition} Covering number bound: There exists a positive sequence $(E_{\alpha})_{\alpha\in\mathcal{A}}$ such that
    \begin{align*}
    \sup_{\epsilon\geq\epsilon_{n,\alpha}} \log\mathcal{N}\bigl(\epsilon/3,\, B'_{\alpha}(P_0,2\epsilon),\, d_{\mathrm{H}}\bigr) \leq E_{\alpha}n\epsilon_{n,\alpha}^2 \quad\text{for all }\alpha\in\mathcal{A}. 
    \end{align*}
    \item \label{eq:prior-mass-condition-beta} Prior mass lower bound for $\beta$: There exists $F>0$ such that
    \begin{align*}
    \pi_{\beta}\bigl(B_{\beta}(P_0,\epsilon_{n,\beta})\bigr) \geq e^{-Fn\epsilon_{n,\beta}^2}. 
    \end{align*}
    \item \label{eq:mixture-weight-condition} Mixture weights: There exists a positive sequence $(\mu_{n,\alpha})_{\alpha\in\mathcal{A}}$ such that
    \begin{align*}
    \frac{\lambda_{\alpha}}{\lambda_{\beta}} \leq \mu_{n,\alpha}\, e^{n(H^{-1}\epsilon_{n,\alpha}^2 \vee \epsilon_{n,\beta}^2)} \quad\text{for all }\alpha\in\mathcal{A}.
    \end{align*}
    \item \label{eq:prior-mass-condition-alpha} Prior mass upper bound for $\mathcal{A}_{<\beta}$: There exist $M,G>0$ such that
    \begin{align*}
    \sum_{\alpha\in\mathcal{A}_{<\beta}} \frac{\lambda_{\alpha}}{\lambda_{\beta}} \cdot \pi_{\alpha}\bigl(B'_{\alpha}(P_0,M\epsilon_{n,\alpha})\bigr) \leq Ge^{-(F+3)n\epsilon_{n,\beta}^2}. 
    \end{align*}
    \item \label{eq:E-condition} Uniformity of covering number constants: There exists $E \in (0,\infty)$ such that
    \begin{align*}
    E\geq\sup_{\alpha\in\mathcal{A}_{<\beta}} E_{\alpha} \vee \sup_{\alpha\in\mathcal{A}_{\geq\beta}} \frac{E_{\alpha}\epsilon_{n,\alpha}^2}{\epsilon_{n,\beta}^2}. 
    \end{align*}
\end{enumerate}
The set $\mathcal{A}_{<\beta}$ may be empty; in that case, condition~\ref{eq:prior-mass-condition-alpha} is automatically satisfied, and the supremum over the empty set in~\ref{eq:E-condition} is defined as $-\infty$.

The following theorem is a modification of \citet[Theorem~2.1]{ghosal2008nonparametric}, where we use a slightly larger constant in the exponent in~\ref{eq:prior-mass-condition-alpha} and impose an additional sub-Gamma assumption on the log-likelihood ratios in order to obtain exponential posterior contraction (as opposed to the weaker conclusion that the left-hand side of~\eqref{eq:exponential-posterior-contraction} converges to zero as $n\to\infty$).  In our in-context learning problem, this is crucial to derive rates of convergence for the posterior regression function. The sub-Gamma assumption is verified in Lemma~\ref{lemma:properties-of-gaussian-regression}\emph{(c)} under a nonparametric regression setting with independent Gaussian noise. 
%We also combine Theorem~\ref{thm:exponential-posterior-contraction} and Lemma~\ref{lemma:properties-of-gaussian-regression} to give convergence rates for the posterior regression function in a nonparametric regression setting in Corollary~\ref{cor:posterior-mean-concentration}.

\begin{prop}\label{prop:exponential-posterior-contraction}
    Let $P_0$ be a distribution on $\mathcal{Z}$, let $\beta\in\mathcal{A}$ and let $Z_1,\ldots,Z_n \overset{\mathrm{iid}}{\sim} P_0$.
    Assume that the conditions \ref{eq:entropy-condition}--\ref{eq:E-condition} hold with $M^2 > 243(F+1)(HE+1)$ and $\sum_{\alpha\in\mathcal{A}} \sqrt{\mu_{n,\alpha}} \leq e^{n\epsilon_{n,\beta}^2}$. For $i\in[n]$ and $P\in\mathcal{P}_{\beta}$ with with $P \ll P_0$, let 
    \[
    W(P,Z_i) \coloneqq \log \frac{\d P}{\d P_0}(Z_i) - \mathbb{E}\log \frac{\d P}{\d P_0}(Z_i).
    \]
    Assume that there exist $a,c>0$ such that for all $P\in B_{\beta}(P_0,\epsilon_{n,\beta})$, we have that $W(P,Z_1)$ is sub-Gamma in the left tail with variance parameter $a\epsilon_{n,\beta}^2$ and scale parameter $c$. Then there exist $C_1,c_2,c_3>0$, not depending on $n$, such that if $n\epsilon_{n,\beta}^2 \geq c_3$, then
    \begin{align}
        \mathbb{E}_{(Z_1,\ldots,Z_n)\sim P_0^{\otimes n}} \bigl[ \pi\bigl(\{P\in\mathcal{P} : d_{\mathrm{H}}(P,P_0)\geq M \epsilon_{n,\beta}\} \,\big|\, Z_1,\ldots,Z_n\bigr) \bigr] \leq C_1e^{-c_2n\epsilon_{n,\beta}^2}.  \label{eq:exponential-posterior-contraction}
    \end{align}
\end{prop}
\begin{proof}
    Our assumptions ensure that the conditions of \citet[Theorem~2.1]{ghosal2008nonparametric} hold\footnote{As we seek a finite-sample bound, we do not require their condition that $n\epsilon_{n,\beta}^2 \to \infty$ as $n\to\infty$.} with $K=1/9$, $L=K/3=1/27$, $I=\sqrt{27(F+1)}$ and $B=M/I=M/\sqrt{27(F+1)}$ therein. Moreover, our conditions~\ref{eq:prior-mass-condition-beta} and~\ref{eq:prior-mass-condition-alpha} ensure that
    \begin{align}
        \sum_{\alpha\in\mathcal{A}_{<\beta}} \frac{\lambda_{\alpha}}{\lambda_{\beta}} \cdot \frac{\pi_{\alpha}\bigl(B'_{\alpha}(P_0,M\epsilon_{n,\alpha})\bigr)}{\pi_{\beta}\bigl(B_{\beta}(P_0,\epsilon_{n,\beta})\bigr)} \leq G e^{-3n\epsilon_{n,\beta}^2}. \label{eq:prior-mass-ratio}
    \end{align}
    Now, by following the proof of \citet[Theorem~2.1]{ghosal2008nonparametric}, replacing the use of their~Eq.~(2.4) by our~\eqref{eq:prior-mass-ratio}, and the use of their Lemma~6.3 by our Lemma~\ref{lemma:evidence-lower-bound} (which applies because of our sub-Gamma assumption on $W(P,Z_1)$), we deduce the desired conclusion.
\end{proof}

\subsection{Bayesian nonparametric regression} \label{sec:bayesian-nonparametric-regression}
Let $\mathcal{X}\subseteq \mathbb{R}^d$ and $\mathcal{Y}=\mathbb{R}$. In this section, we specialise to the setting of regression and take $\mathcal{Z}=\mathcal{X}\times\mathcal{Y}$.
Let $P_X$ be a Borel distribution on $\mathcal{X}$, let $X\sim P_X$ and let $\xi\sim N(0,\sigma^2)$ be independent of $X$. Let $R>0$, and for $\alpha\in\mathcal{A}$, let $\mathcal{G}_{\alpha}$ be a space of measurable functions from $\mathcal{X}$ to $[-R,R]$. For any Borel measurable function $g:\mathcal{X} \to \mathbb{R}$, we write $P_g$ for the distribution of $(X,Y_g)$ where $Y_g=g(X)+\xi$. For $\alpha\in\mathcal{A}$, let $\tilde{\pi}_{\alpha}$ be a distribution on~$\mathcal{G}_{\alpha}$, let $\tilde{g}_{\alpha} \sim \tilde{\pi}_{\alpha}$ and let $\pi_{\alpha}$ denote the distribution of the random measure $P_{\tilde{g}_{\alpha}}$, where the randomness comes from $\tilde{g}_{\alpha} \sim \tilde{\pi}_{\alpha}$.  Finally, let $\tilde{\pi}\coloneqq \sum_{\alpha\in\mathcal{A}} \lambda_{\alpha} \tilde{\pi}_{\alpha}$ and $\pi\coloneqq \sum_{\alpha\in\mathcal{A}} \lambda_{\alpha} \pi_{\alpha}$, where $\lambda_\alpha \geq 0$ for all $\alpha \in \mathcal{A}$ and $\sum_{\alpha \in \mathcal{A}} \lambda_\alpha = 1$, let $\mathcal{P}$ be the support of $\pi$ and $\mathcal{G}\coloneqq \bigcup_{\alpha\in\mathcal{A}} \mathcal{G}_{\alpha}$.

We first provide an alternative expression for the $\pi$-posterior regression function $g_{\pi}$, defined in~\eqref{eq:def-posterior-regression-function}.  For $P\in\mathcal{P}$, let $g_P:\mathcal{X}\to\mathcal{Y}$ be the regression function under~$P$, so that~$g_P(x) \coloneqq \mathbb{E}_{(X,Y)\sim P}(Y \,|\, X=x)$ for $x\in\mathcal{X}$. 
%Assume that
% \begin{align}
%     \mathbb{E}_{P\sim\pi,\, (\mathcal{D}_n,X,Y)|P\sim P^{\otimes (n+1)}} \bigl\{g_P(X) \,\big|\, X=x,\mathcal{D}_n\bigr\} = \mathbb{E}_{P\sim\pi,\, \mathcal{D}_n|P\sim P^{\otimes n}} \bigl\{g_P(x) \,\big|\, \mathcal{D}_n\bigr\} \label{eq:assumption-on-prior}
% \end{align}
% for all $x\in\mathcal{X}$. Intuitively, this means that knowing the test point $X$ does not provide any information on the regression function $g_P$, which is a very natural assumption. Indeed, this assumption holds if the distribution of $X$ and the conditional distribution of $Y\,|\,X$ are drawn independently. 
Then the $\pi$-posterior regression function~$g_{\pi}$ satisfies
\begin{align}
    g_\pi (x,D_n) &= \mathbb{E}_{P\sim\pi,\, (\mathcal{D}_n,X,Y)|P\sim P^{\otimes (n+1)}} \bigl\{\mathbb{E}(Y \,|\, X=x,\, P) \,|\, X=x,\, \mathcal{D}_n=D_n\bigr\} \nonumber \\
    &= \mathbb{E}_{P\sim\pi,\, (\mathcal{D}_n,X,Y)|P\sim P^{\otimes (n+1)}} \bigl(g_P(x) \,|\, \mathcal{D}_n=D_n\bigr) \nonumber \\
    &= \int_{\mathcal{P}} g_P(x) \,\d \pi(P\,|\, D_n) = \int_{\mathcal{G}} g(x) \,\d \tilde{\pi}(g\,|\, D_n), \label{eq:alternative-def-posterior-regression-function}
\end{align}
so $g_{\pi}(\cdot,D_n)$ can be thought of as the posterior mean of the regression function.

\begin{cor}\label{cor:posterior-regression-function-concentration}
    Let $\beta\in\mathcal{A}$, let $g^0: \mathcal{X}\to[-R,R]$ be measurable and write $P_0 \equiv P_{g^0}$.  Assume that the conditions \ref{eq:entropy-condition}--\ref{eq:E-condition} hold with $M^2 > 243(F+1)(HE+1)$ and $\sum_{\alpha\in\mathcal{A}} \sqrt{\mu_{n,\alpha}} \leq e^{n\epsilon_{n,\beta}^2}$. 
    There exist $C_1,C_2,c_3,c_4>0$, not depending on $n$, such that if $n\epsilon_{n,\beta}^2 \geq c_4$, then
    \begin{align*}
        \mathbb{E}_{(\mathcal{D}_n,X,Y)\sim P_0^{\otimes(n+1)}} \bigl[\bigl\{g_\pi(X,\mathcal{D}_n) - g^0(X)\bigr\}^2\bigr] \leq C_1\epsilon_{n,\beta}^2 + C_2e^{-c_3n\epsilon_{n,\beta}^2}.
    \end{align*}
\end{cor}
\begin{proof}
    We write $\mathbb{E}_{P_0}$ for $\mathbb{E}_{(\mathcal{D}_n,X,Y)\sim P_0^{\otimes(n+1)}}$ and $C_1' \coloneqq 2R^2/\bigl(1-e^{-R^2/(2\sigma^2)}\bigr)$. We have by~\eqref{eq:alternative-def-posterior-regression-function} that
    \begin{align*}
        \mathbb{E}_{P_0} \bigl[\bigl\{g_\pi(X&,\mathcal{D}_n) - g^0(X)\bigr\}^2\bigr]
        = \mathbb{E}_{P_0} \biggl[ \biggl\{ \int_{\mathcal{P}}g_P(X) \,\d\pi(P\,|\,\mathcal{D}_n) - g^0(X) \biggr\}^2 \biggr]\\
        \overset{(i)}&{\leq} \mathbb{E}_{P_0} \biggl[ \int_{\mathcal{P}} \bigl\{ g_P(X)  - g^0(X) \bigr\}^2 \,\d\pi(P\,|\,\mathcal{D}_n) \biggr]\\
        \overset{(ii)}&{=}  \mathbb{E}_{\mathcal{D}_n\sim P_0^{\otimes n}} \biggl( \int_{\mathcal{P}} \mathbb{E}_{(X,Y)\sim P_0}\bigl[\bigl\{ g_P(X)  - g^0(X) \bigr\}^2\bigr] \,\d\pi(P\,|\,\mathcal{D}_n) \biggr) \\
        \overset{(iii)}&{\leq} C_1' \mathbb{E}_{\mathcal{D}_n\sim P_0^{\otimes n}} \biggl[ \int_{\mathcal{P}} d_{\mathrm{H}}^2(P,P_0) \,\d\pi(P\,|\,\mathcal{D}_n) \biggr]\\
        \overset{(iv)}&{\leq} C_1' M^2 \epsilon_{n,\beta}^2 + 2C_1' \mathbb{E}_{\mathcal{D}_n\sim P_0^{\otimes n}} \bigl[ \pi\bigl(\{P\in\mathcal{P} : d_{\mathrm{H}}(P,P_0)\geq M\epsilon_{n,\beta}\} \,\big|\, \mathcal{D}_n\bigr) \bigr]\\
        \overset{(v)}&{\leq} C_1\epsilon_{n,\beta}^2 + C_2e^{-c_3n\epsilon_{n,\beta}^2}.
    \end{align*}
    Here, $(i)$ follows from Jensen's inequality, $(ii)$ follows by Fubini's theorem, $(iii)$ follows from Lemma~\ref{lemma:properties-of-gaussian-regression}\emph{(d)}, $(iv)$ follows by splitting the integral into regions where $d_{\mathrm{H}}(P,P_0)< M\epsilon_{n,\beta}$ and $d_{\mathrm{H}}(P,P_0)\geq M\epsilon_{n,\beta}$ and $(v)$ follows from Proposition~\ref{prop:exponential-posterior-contraction} where the sub-Gamma assumption is verified in Lemma~\ref{lemma:properties-of-gaussian-regression}\emph{(c)}.
\end{proof}

\section{Proofs for Section~\ref{sec:adaptive-icl}} \label{sec:proofs-of-main-results}

\subsection{Definition of our random Besov function} \label{sec:defn-random-besov}
Since Besov functions can be characterised by their wavelet decomposition, we will consider regression functions with random wavelet coefficients.  Let $L^2([0,1]^d)$ denote the set of square integrable functions on $[0,1]^d$, let $S\in\mathbb{N}$ and let
\begin{align}
    \bigl\{\Phi_k : k\in K\bigr\} \cup \bigl\{\Psi_{\ell,\gamma} : \ell\in[{\ell_0}:\infty),\, \gamma\in\Gamma_{\ell}\bigr\} \label{eq:cdv-wavelet}
\end{align}
denote the tensor product Cohen--Daubechies--Vial (CDV) wavelet basis for $L^2([0,1]^d)$, constructed from $S$-regular and $S$-times continuously differentiable wavelets. The precise definitions of these wavelet functions are not important for us, and we will only use some basic properties of this wavelet basis, summarised in Appendix~\ref{subsec:wavelets}.  Let $c_0\in(0,1]$ and $C_0\geq 1$.
For $\alpha\in(0,S)$, define the random function
\begin{align}
    &\tilde{g}_{\alpha}\coloneqq \sum_{k\in K} C_0 2^{-\ell_0 d/2} a_k^{(\alpha)}\Phi_k + \sum_{\ell=\ell_0}^{\infty} \sum_{\gamma\in\Gamma_{\ell}} C_0 2^{-\ell(\alpha+d/2)} b_{\ell,\gamma}^{(\alpha)} \Psi_{\ell,\gamma}, \label{eq:random-g-alpha}
\end{align}
where $\bigl(a_k^{(\alpha)}, b_{\ell,\gamma}^{(\alpha)} : k\in K, \ell\in[\ell_0:\infty), \gamma\in\Gamma_\ell\bigr)$ are independent random variables supported on $[-1,1]$ with Lebesgue density bounded between $c_0/2$ and $c_0^{-1}/2$. Further, define
$\tilde{\pi}_{\alpha} \coloneqq \mathsf{Law}(\tilde{g}_{\alpha})$ as the distribution of $\tilde{g}_{\alpha}$.
The scaling of the wavelet coefficients in~\eqref{eq:random-g-alpha} is chosen to ensure that $\tilde{g}_{\alpha}$ belongs to the Besov space $B_{\infty,\infty}^{\alpha}([0,1]^d)$ with smoothness $\alpha$. Indeed, a function $g \coloneqq \sum_{k\in K} a_k' \Phi_k + \sum_{\ell=\ell_0}^{\infty} \sum_{\gamma\in\Gamma_{\ell}} b_{\ell,\gamma}' \Psi_{\ell,\gamma}$, with deterministic wavelet coefficients, belongs to the Besov ball $B_{\infty,\infty}^\alpha([0,1]^d,C)$ of radius $C>0$ if and only if $2^{\ell_0 d/2}\max_{k} |a_k'| + \sup_{\ell,\gamma} 2^{\ell(\alpha+d/2)}|b_{\ell,\gamma}'| \leq C$; see Appendix~\ref{sec:function-spaces-and-wavelets}. Thus, writing $\mathcal{G}_{\alpha}\coloneqq \mathrm{supp}(\tilde{\pi}_{\alpha})$, we have
\begin{align}
    B_{\infty,\infty}^\alpha([0,1]^d,C_0) \subseteq \mathcal{G}_{\alpha} \subseteq B_{\infty,\infty}^\alpha([0,1]^d,2C_0). \label{eq:inclusion-G-alpha}
\end{align}
We remark that the Besov space $B_{\infty,\infty}^\alpha([0,1]^d)$ coincides with the H\"older space $H^{\alpha}([0,1]^d)$ for $\alpha\notin\mathbb{N}$, and it contains $H^{\alpha}([0,1]^d)$ for $\alpha\in\mathbb{N}$; see Appendix~\ref{sec:function-spaces} for more details. Moreover, there exists $C_{\alpha,S,d}>0$ such that $\sup_{g\in\mathcal{G}_{\alpha}} \|g\|_{\infty} \leq C_{\alpha,S,d}$; see~\eqref{eq:ell-infty-norm-g-circ}. See Figure~\ref{fig:random-besov} for realisations of random Besov functions with $\alpha=2$ and~$4$.

\begin{assumption} \label{assumption:P-X-besov}
    Let $P_X$ be a probability distribution on $[0,1]^d$ with the property that there exists a hypercube $A \subseteq[0,1]^d$ such that $P_X(A_0) \geq c_0\mathrm{Vol}_d(A_0)$ for all measurable $A_0\subseteq A$.  
\end{assumption}

\subsection{Definition of our random multi-index model} \label{sec:defn-random-multi-index}
For $p\in[d]$, let $V_p(\mathbb{R}^d)\coloneqq \bigl\{U\in\mathbb{R}^{d\times p} : U^\top U = I_p\bigr\}$ denote the set of all projection matrices from $\mathbb{R}^d$ to $\mathbb{R}^p$ (also called the \emph{Stiefel manifold}). We will consider regression functions on $\mathbb{B}^d$ of the form $x\mapsto g\bigl((U^\top x + 1_p)/2\bigr)$, where $U\in V_p(\mathbb{R}^d)$ for some \emph{effective dimension} $p\in[d]$ and Borel measurable $g:[0,1]^p \to \mathbb{R}$. Note that we translate and scale $U^\top x$ to ensure that it belongs to $[0,1]^p$.

For $\alpha\in(0,S)$ and $p\in[d]$, let $\tilde{g}_{\alpha}^{(p)}: [0,1]^p \to \mathbb{R}$ be the random function (with smoothness~$\alpha$) defined as in~\eqref{eq:random-g-alpha}, but with $d$ replaced by $p$ in the wavelet coefficients, and using the tensor product CDV wavelet basis for $L^2([0,1]^p)$ instead of $L^2([0,1]^d)$. Let $\tilde{\pi}_{\alpha,p}$ be the distribution of the random function $\tilde{g}_{\alpha,p} : \mathbb{B}^d \to \mathbb{R}$ defined by
\begin{align}
    \tilde{g}_{\alpha,p}(x) \coloneqq \tilde{g}_{\alpha}^{(p)}\biggl(\frac{(U^{(p)})^\top x+1_p}{2}\biggr), \label{eq:random-g-alpha-p}
\end{align}
where $U^{(p)}$ is uniformly distributed\footnote{One way to generate such a random matrix is to set $U^{(p)}\coloneqq Z(Z^\top Z)^{-1/2}$, where $Z$ is a $d\times p$ random matrix with independent $N(0,1)$ entries \citep[e.g.][Theorem~2.2.1(iii)]{chikuse2003statistics}.} on the Stiefel manifold $V_p(\mathbb{R}^d)$, independently of $\tilde{g}_{\alpha}^{(p)}$.  We write $\mathcal{G}_{\alpha,p}$ for the support of $\tilde{\pi}_{\alpha,p}$.

Let $P_X$ be a distribution on $\mathbb{B}^d$ with the property that there exist $c_0>0$ and a non-empty Euclidean ball $A \subseteq \mathbb{B}^d$ such that $P_X(A_0) \geq c_0\mathrm{Vol}_d(A_0)$ for all measurable $A_0\subseteq A$. Given $g\in\mathcal G_{\alpha,p}$, let $P_{g}$ denote the distribution of $(X,Y_g)$ where $X\sim P_X$,
\begin{align*}
    Y_g=g(X)+\xi, \quad \xi\sim N(0,\sigma^2) \quad\text{and}\quad \xi \indep X.
\end{align*}
We then define~$\pi_{\alpha,p}$ to be the distribution of the random measure $P_{\tilde{g}_{\alpha,p}}$, where the randomness comes from the random regression function $\tilde{g}_{\alpha,p} \sim \tilde{\pi}_{\alpha,p}$, and let~$\mathcal{P}_{\alpha,p}\coloneqq \mathrm{supp}(\pi_{\alpha,p})$.

\subsection{Proof of Proposition~\ref{prop:posterior-regression-function-wavelet-prior}} \label{sec:proof-of-posterior-contraction-holder}

We begin with two preliminary lemmas in the setting of Section~\ref{sec:holder-regression}.  Suppose that $g^\circ \in \mathcal{G}_{\beta}$ has wavelet decomposition
\begin{align}
    g^\circ = \sum_{k\in K} C_0 2^{-\ell_0 d/2} a_k^\circ \Phi_k + \sum_{\ell=\ell_0}^{\infty} \sum_{\gamma\in\Gamma_{\ell}} C_0 2^{-\ell(\beta + d/2)} b_{\ell,\gamma}^\circ \Psi_{\ell,\gamma}, \label{eq:def-g0}
\end{align}
where $(a_k^\circ)_{k\in K}$ and $(b_{\ell,\gamma}^\circ)_{\ell\in[\ell_0:\infty),\gamma\in\Gamma_{\ell}}$ are deterministic and satisfy $|a^\circ_k| \leq1$ and $|b^\circ_{\ell,\gamma}| \leq 1$ for $k\in K$, $\ell\in[\ell_0:\infty)$ and $\gamma\in\Gamma_{\ell}$.
Lemma~\ref{lemma:prior-mass-lb} below provides a lower bound on the prior mass; similar results for the one-dimensional case exist in the literature, e.g.~\citet[Proposition~1]{gine2011rates} and \citet[Lemma~4.1]{reiss2020posterior}.

\begin{lemma} \label{lemma:prior-mass-lb}
    Let $n\in\mathbb{N}$ and assume that~\eqref{eq:def-g0} holds. Then there exists $C>0$, depending only on $\beta,d,c_0,C_0$ and $S$, such that
    \begin{align}
        \tilde{\pi}_{\beta} \bigl(\bigl\{g\in \mathcal{G}_{\beta} : \|g-g^{\circ}\|_{\infty}^2 \leq C n^{-2\beta/(2\beta+d)} \bigr\}\bigr) \geq \exp\bigl(-C n^{d/(2\beta+d)} \bigr). \label{eq:prior-mass-lb-func}
    \end{align}
    Moreover, there exists $C' > 0$, depending only on $\beta,d,c_0,C_0,S$ and $\sigma$, such that 
    \begin{align}
        \pi_{\beta}\bigl(B_{\beta}(P_{g^{\circ}}, \epsilon_{n,\beta})\bigr) \geq \exp\bigl(-C' n \epsilon_{n,\beta}^2 \bigr), \label{eq:prior-mass-lb-dist}
    \end{align}
    where $\epsilon_{n,\beta} \coloneqq C'n^{-\beta/(2\beta+d)}$ and $B_{\beta}(P_{g^{\circ}}, \epsilon_{n,\beta})$ is defined by~\eqref{eq:B-ball}.
\end{lemma}
\begin{proof}
    Throughout the proof, $C_1,C_2,\ldots$ will denote positive quantities that may depend only on $\beta,d,c_0,C_0$ and $S$, while $C_1',C_2',\ldots$ may in addition depend on $\sigma$.  Let $\tilde{g}_{\beta}$ be given by~\eqref{eq:random-g-alpha} so that $\tilde{g}_{\beta}\sim\tilde{\pi}_{\beta}$, and write $a_k\equiv a_k^{(\beta)}$ and $b_{\ell,\gamma} \equiv b_{\ell,\gamma}^{(\beta)}$ where the random variables $(a_k^{(\beta)})$ and $(b_{\ell,\gamma}^{(\beta)})$ satisfy the conditions below~\eqref{eq:random-g-alpha}. By~\eqref{eq:ell-infty-norm-wavelet-functions} and~\eqref{eq:number-of-non-zero-wavelets}, writing $L\coloneqq \bigl\lceil \frac{\log_2 n}{2\beta+d} \bigr\rceil$ (and assuming for now that $L \geq \ell_0$), we have
    \begin{align}
        \|\tilde{g}_{\beta}-g^{\circ}\|_{\infty} 
        &\leq C_1 \max_{k\in K} |a_k-a^\circ_k| + C_1\sum_{\ell=\ell_0}^{L} 2^{-\ell\beta} \max_{\gamma\in\Gamma_{\ell}} |b_{\ell,\gamma} - b_{\ell,\gamma}^{\circ}| + 2C_0C_1\sum_{\ell=L+1}^{\infty} 2^{-\ell\beta}\nonumber\\
        &\leq C_1 \max_{k\in K} |a_k-a^\circ_k| + C_1\sum_{\ell=\ell_0}^{L} 2^{-\ell\beta} \max_{\gamma\in\Gamma_{\ell}} |b_{\ell,\gamma} - b_{\ell,\gamma}^{\circ}| + C_3 n^{-\beta/(2\beta+d)}. \label{eq:ub-g-g0}
    \end{align}
    Define the events
    \begin{align*}
        \mathcal{E}_1 &\coloneqq \bigcap_{k\in K} \Bigl\{ |a_k-a^\circ_k| \leq 2 \cdot n^{-\beta/(2\beta+d)} \Bigr\},\\
        \mathcal{E}_2 &\coloneqq \bigcap_{\ell=\ell_0}^L \bigcap_{\gamma\in\Gamma_{\ell}} \Bigl\{ |b_{\ell,\gamma} - b_{\ell,\gamma}^{\circ}| \leq 2 \cdot 2^{-2(L-\ell)\beta}\Bigr\}.
    \end{align*}
    Then, on the event $\mathcal{E}_1\cap \mathcal{E}_2$, we have by~\eqref{eq:ub-g-g0} that
    \begin{align*}
        \|\tilde{g}_{\beta}-g^{\circ}\|_{\infty} &\leq 2C_1 n^{-\beta/(2\beta+d)} + 2C_1 2^{-L\beta} \sum_{\ell=\ell_0}^L 2^{-(L-\ell)\beta} + C_3 n^{-\beta/(2\beta+d)} \leq C_4 n^{-\beta/(2\beta+d)}.
    \end{align*}
    Therefore,
    \begin{align}
        \mathbb{P}\bigl( \|\tilde{g}_{\beta}-g^{\circ}\|_{\infty} \leq C_4 n^{-\beta/(2\beta+d)} \bigr) &\geq \mathbb{P}(\mathcal{E}_1 \cap \mathcal{E}_2) = \mathbb{P}(\mathcal{E}_1) \mathbb{P}(\mathcal{E}_2). \label{eq:lb-pi-bar-alpha}
    \end{align}
    Now, by~\eqref{eq:number-of-wavelets}, we have
    \begin{align}
        \mathbb{P}(\mathcal{E}_1) \geq \bigl(c_0 n^{-\beta/(2\beta+d)} \bigr)^{2^{\ell_0 d}} = \exp(-C_5 \log n). \label{eq:P-E1}
    \end{align}
    Similarly,
    \begin{align}
       \mathbb{P}(\mathcal{E}_2) >  \prod_{\ell=\ell_0}^L \bigl( c_0 2^{-2(L-\ell)\beta} \bigr)^{2^{(\ell+1)d}} &= \exp\biggl( (\log c_0) \sum_{\ell=\ell_0}^L 2^{(\ell+1)d} -(2^{d+1}\log 2) \beta \sum_{\ell=\ell_0}^L (L-\ell)2^{\ell d} \biggr)\nonumber\\
       &\geq \exp(-C_6 n^{d/(2\beta+d)}), \label{eq:P-E2}
    \end{align}
    where the final inequality follows since if we set $A\coloneqq \sum_{\ell=\ell_0}^L (L-\ell)2^{\ell d}$, then $2^dA-A \leq \sum_{\ell=0}^L 2^{\ell d}\leq 2^{(L+1)d}/(2^d-1)$, so $A\leq 2^{(L+1)d}/(2^d-1)^2$.
    The conclusion~\eqref{eq:prior-mass-lb-func} for $n$ large enough that $L \geq \ell_0$ follows from~\eqref{eq:lb-pi-bar-alpha}, \eqref{eq:P-E1} and~\eqref{eq:P-E2}.  By increasing $C > 0$, depending only on $\beta,d,c_0,C_0$ and $S$, we can then ensure that~\eqref{eq:prior-mass-lb-func} holds for all $n \in \mathbb{N}$.

    For~\eqref{eq:prior-mass-lb-dist}, note that by Lemma~\ref{lemma:properties-of-gaussian-regression}\emph{(a)} and~\emph{(b)}, we have that if $g\in\mathcal{G}_{\beta}$ satisfies $\|g - g^{\circ}\|_{\infty}^2 \leq Cn^{-2\beta/(2\beta+d)}$, then 
    \begin{align*}
        \mathrm{KL}(P_{g^{\circ}}, P_g) \leq C_1'Cn^{-2\beta/(2\beta+d)} \quad\text{and}\quad \mathrm{V}_2(P_{g^{\circ}}, P_g) \leq C_1'Cn^{-2\beta/(2\beta+d)}.
    \end{align*}
    Writing $\epsilon_{n,\beta} \coloneqq \sqrt{C_1'C} \cdot n^{-\beta/(2\beta+d)}$, we deduce from~\eqref{eq:prior-mass-lb-func} that
    \begin{align*}
        \pi_{\beta}\bigl(B_{\beta}(P_{g^\circ}, \epsilon_{n,\beta})\bigr) \geq \exp\bigl(-Cn^{d/(2\beta+d)}\bigr) \geq \exp(-C_2'n \epsilon_{n,\beta}^2).
    \end{align*}
    Thus~\eqref{eq:prior-mass-lb-dist} holds with $C'\coloneqq \sqrt{C_1'C} \vee C_2'$.
\end{proof}

Lemma~\ref{lemma:prior-mass-ub} below provides an upper bound on the prior mass.
\begin{lemma} \label{lemma:prior-mass-ub}
    Let $\alpha\in\mathcal{A}$, $m>0$ and assume that~\eqref{eq:def-g0} holds. Suppose that $P_X$ is a Borel distribution on $[0,1]^d$ with the property that there exist $c_0>0$ and a hypercube $A\subseteq[0,1]^d$ of side length $\tau>0$ such that $P_X(A_0) \geq c_0\mathrm{Vol}_d(A_0)$ for all measurable $A_0\subseteq A$. Then there exist $c \equiv c(\alpha,d,c_0,C_0,\tau) > 0$ and $n_0 \equiv n_0(\alpha,d,c_0,C_0,S,m,\tau) \in \mathbb{N}$  such that 
    \begin{align}
        \tilde{\pi}_{\alpha}\Bigl(\Bigl\{g\in \mathcal{G}_{\alpha} : \|g-g^{\circ}\|_{L^2(P_X)}^2 \leq \frac{c_0}{4}m^2n^{-2\alpha/(2\alpha+d)} \Bigr\}\Bigr) \leq \exp\bigl( -c m^{-d/\alpha} n^{d/(2\alpha+d)} \bigr)\label{eq:prior-mass-ub-func}
    \end{align}
    for $n \geq n_0$.  Moreover, there exists $c'>0$, depending only on $\alpha,\beta,d,c_0,C_0,S,\tau$ and $\sigma$, such that for $n\geq n_0$,
    \begin{align}
        \pi_{\alpha}\bigl(B'_{\alpha}(P_{g^{\circ}}, c'mn^{-\alpha/(2\alpha+d)})\bigr) \leq \exp\bigl( -c m^{-d/\alpha} n^{d/(2\alpha+d)} \bigr) \label{eq:prior-mass-ub-dist}
    \end{align}
    where $B'_{\alpha}(P_{g^{\circ}}, c'mn^{-\alpha/(2\alpha+d)})$ is defined as in~\eqref{eq:C-ball}.
\end{lemma}
\begin{proof}
    Throughout the proof, $C_1,C_2,\ldots$ will denote positive quantities that may depend only on $\alpha,d,c_0,C_0$ and $\tau$, while $C_1',C_2',\ldots$ may depend in addition on $\beta,S$ and $\sigma$. Let $\epsilon\coloneqq n^{-\alpha/(2\alpha+d)}$.  Fix $g'\in\mathcal{G}_{\alpha}$ such that $\|g'-g^\circ\|_{L^2(P_X)}^2 \leq \frac{c_0}{4}m^2 \epsilon^2$; we may assume without loss of generality that such a~$g'$ exists, since otherwise the left-hand side of~\eqref{eq:prior-mass-ub-func} would be zero. Suppose that $g'$ has wavelet decomposition
    \begin{align*}
        g' = \sum_{k\in K} C_0 2^{-\ell_0 d/2} a_k' \Phi_k + \sum_{\ell=\ell_0}^{\infty} \sum_{\gamma\in\Gamma_{\ell}} C_0 2^{-\ell(\alpha + d/2)} b_{\ell,\gamma}' \Psi_{\ell,\gamma},
    \end{align*}
    where $|a_k'| \leq 1$ and $|b_{\ell,\gamma}'| \leq 1$. Note that if $g\in \mathcal{G}_{\alpha}$ satisfies $\|g-g^{\circ}\|_{L^2(P_X)}^2 \leq \frac{c_0}{4}m^2\epsilon^2$, then $\|g-g'\|_{L^2(P_X)}^2 \leq 2\|g-g^\circ\|_{L^2(P_X)}^2 + 2\|g'-g^\circ\|_{L^2(P_X)}^2 \leq c_0m^2\epsilon^2$. Thus,
    \begin{align}
        \tilde{\pi}_{\alpha}\Bigl(\Bigl\{g\in \mathcal{G}_{\alpha} : \|g &-g^{\circ}\|_{L^2(P_X)}^2 \leq \frac{c_0}{4}m^2\epsilon^2 \Bigr\}\Bigr) \leq \tilde{\pi}_{\alpha}\Bigl(\Bigl\{g\in \mathcal{G}_{\alpha} : \|g-g'\|_{L^2(P_X)}^2 \leq c_0m^2\epsilon^2 \Bigr\}\Bigr). \label{eq:center-at-g'}
    \end{align}
    Let $\tilde{g}_{\alpha}\sim\tilde{\pi}_{\alpha}$, and write $a_k\equiv a_k^{(\alpha)}$ and $b_{\ell,\gamma} \equiv b_{\ell,\gamma}^{(\alpha)}$ with representation~\eqref{eq:random-g-alpha}, where the random variables $(a_k^{(\alpha)})$ and $(b_{\ell,\gamma}^{(\alpha)})$ satisfy the conditions below~\eqref{eq:random-g-alpha}.  Further, let $C_1\coloneqq \lceil \log_2(2/\tau) \rceil$. Then, by the assumption on $P_X$, there exists $C_2\in[0:2^{C_1}-1]$ such that $\bigl[\frac{C_2}{2^{C_1}}, \frac{C_2+1}{2^{C_1}}\bigr]^d \subseteq A$.   Choose $n_0 \equiv n_0(\alpha,d,c_0,C_0,S,m,\tau) \in \mathbb{N}$ large enough that $L\coloneqq \bigl\lfloor \frac{1}{2\alpha} \log_2\bigl(\frac{c_0C_0^2}{6m^2\epsilon^2 2^{C_1d}}\bigr) \bigr\rfloor$ is at least $\ell_0\vee C_1$ for $n \geq n_0$.  Assume henceforth that $n \geq n_0$.  Let $\Gamma_L' \coloneqq \bigl\{\gamma\in\Gamma_L : \mathrm{supp}(\Psi_{L,\gamma}) \subseteq \bigl[\frac{C_2}{2^{C_1}}, \frac{C_2+1}{2^{C_1}}\bigr]^d\bigr\}$, $b_L \coloneqq (b_{L,\gamma})_{\gamma\in\Gamma_L'}$ and $b_L' \coloneqq (b_{L,\gamma}')_{\gamma\in\Gamma_L'}$. Since $\mathrm{supp}(\Psi_{L,\gamma})$ is contained in a dyadic cube of side length $2^{-L}$ (see Section~\ref{subsec:wavelets}), we deduce that $|\Gamma_L'| \geq 2^{(L-C_1)d}$ and that
    \begin{align*}
        C_0^2 2^{-L(2\alpha + d)}\mathbb{E}\bigl(\|b_L - b_L'\|_2^2\bigr) &\geq C_0^2 2^{-L(2\alpha + d)} \mathrm{tr}\bigl(\mathrm{Cov}(b_L)\bigr)\\
        &\geq |\Gamma_L'| C_0^2 2^{-L(2\alpha+d)} \cdot \frac{c_0}{3} \geq \frac{c_0C_0^2 2^{-2\alpha L - C_1d}}{3} \geq 2m^2\epsilon^2.
    \end{align*}
    Therefore, by the assumption on $P_X$ and Lemma~\ref{lemma:L2-subset-lb}, we have
    \begin{align}
        \mathbb{P}\bigl(\|\tilde{g}_{\alpha}-g' &\|_{L^2(P_X)}^2  \leq c_0m^2\epsilon^2 \bigr) \leq \mathbb{P}\bigl(c_0\|\tilde{g}_{\alpha}-g'\|_{L^2([C_2 2^{-C_1}, (C_2+1)2^{-C_1}]^d)}^2 \leq c_0m^2\epsilon^2 \bigr)\nonumber\\
        &\leq \mathbb{P}\biggl\{\|b_L - b_L'\|_2^2 - \mathbb{E}\bigl(\|b_L - b_L'\|_2^2\bigr) \leq \frac{m^2 \epsilon^2}{C_0^2 2^{-L(2\alpha + d)}} - \mathbb{E}\bigl(\|b_L - b_L'\|_2^2\bigr) \biggr\} \nonumber\\
        &\leq \mathbb{P}\biggl\{\|b_L - b_L'\|_2^2 - \mathbb{E}\bigl(\|b_L - b_L'\|_2^2\bigr) \leq -\frac{m^2 \epsilon^2}{C_0^2 2^{-L(2\alpha + d)}} \biggr\}. \label{eq:prior-mass-ub-1}
    \end{align}
    Since $(b_{L,\gamma} - b_{L,\gamma}')^2 \in [0, 4]$ for $\gamma\in\Gamma_L'$, we have by Hoeffding's inequality and~\eqref{eq:number-of-wavelets} that
    \begin{align}
        \mathbb{P}\biggl\{\|b_L &- b_L'\|_2^2 - \mathbb{E}\bigl(\|b_L - b_L'\|_2^2\bigr) \leq -\frac{m^2 \epsilon^2}{C_0^2 2^{-L(2\alpha + d)}} \biggr\} \leq \exp\biggl(-\frac{m^4\epsilon^4}{8C_0^4 2^{-(4\alpha+d) L - C_1d}}\biggr) \nonumber\\
        %&\leq \exp\biggl(-\frac{c_0^2 2^{Ld}}{288 \cdot 2^{4\alpha+C_1d}}\biggr) 
        &\leq \exp\biggl\{ -\frac{c_0^2}{288 \cdot 2^{4\alpha+(C_1+1)d}}\biggl(\frac{c_0C_0^2}{6m^2\epsilon^2 2^{C_1d}}\biggr)^{d/(2\alpha)} \biggr\} \eqqcolon \exp\bigl( -c m^{-d/\alpha} n^{d/(2\alpha+d)} \bigr), \label{eq:prior-mass-ub-2}
    \end{align}
    where $c>0$ depends only on $\alpha,d,c_0,C_0$ and $\tau$. Combining~\eqref{eq:center-at-g'},~\eqref{eq:prior-mass-ub-1} and~\eqref{eq:prior-mass-ub-2} yields~\eqref{eq:prior-mass-ub-func}.

    For the second claim, note that by~\eqref{eq:ell-infty-norm-wavelet-functions} and~\eqref{eq:number-of-non-zero-wavelets},
    \begin{align}
        \|g^\circ\|_{\infty} \leq C_3 \max_{k\in K} |a_k^\circ| + C_3\sum_{\ell=\ell_0}^\infty 2^{-\ell\beta} \max_{\gamma\in\Gamma_\ell} |b_{\ell,\gamma}^\circ| \leq C_1', \label{eq:ell-infty-norm-g-circ}
    \end{align}
    and similarly, $\|\tilde{g}_{\alpha}\|_{\infty} \leq C_2'$ almost surely. Thus, if $d_{\mathrm{H}}(P_{g^{\circ}},P_g) \leq c'mn^{-\alpha/(2\alpha+d)}$ for some $c'>0$, then by Lemma~\ref{lemma:properties-of-gaussian-regression}\emph{(d)}, 
    \begin{align*}
        \|g-g^{\circ}\|_{L^2(P_X)}^2 \leq  C_3' d_{\mathrm{H}}^2(P_{g^{\circ}},P_g) \leq C_3'(c')^2 m^2 n^{-2\alpha/(2\alpha+d)}.
    \end{align*}
    Therefore, taking $c'\coloneqq \sqrt{c_0 /(4C_3')}$ and applying~\eqref{eq:prior-mass-ub-func} proves~\eqref{eq:prior-mass-ub-dist}.
\end{proof}

We are now in a position to prove Proposition~\ref{prop:posterior-regression-function-wavelet-prior}.
\begin{proof}[Proof of Proposition~\ref{prop:posterior-regression-function-wavelet-prior}]
    Throughout the proof, $C_1,C_2,\ldots$ will denote positive quantities that do not depend on $n$. By an argument similar to~\eqref{eq:ell-infty-norm-g-circ}, there exists $C_1>0$ such that $\|g^\circ\|_{\infty} \leq C_1$ and $\|g\|_{\infty} \leq C_1$ for all $g\in\mathcal{G}$.
    We now verify the assumptions of Corollary~\ref{cor:posterior-regression-function-concentration} with $H=1$. For each $\alpha\in\mathcal{A}$, let $C'>0$ (which does not depend on $n$) be the quantity in Lemma~\ref{lemma:prior-mass-lb} such that~\eqref{eq:prior-mass-lb-dist} holds, and let $\epsilon_{n,\alpha}\coloneqq C'n^{-\alpha/(2\alpha+d)}$ for $\alpha\in\mathcal{A}$. Then by Lemma~\ref{lemma:prior-mass-lb}, condition~\ref{eq:prior-mass-condition-beta} holds with $F=C'$. By Lemma~\ref{lemma:properties-of-gaussian-regression}\emph{(d)},~\eqref{eq:inclusion-G-alpha} and \citet[Eq.~(4.184)]{gine2021mathematical}, for each $\alpha\in\mathcal{A}$, we have
    \begin{align*}
        \sup_{\epsilon\geq\epsilon_{n,\alpha}} \log\mathcal{N}\bigl(\epsilon/3,\, B'_{\alpha}(P_0,2\epsilon),\, d_{\mathrm{H}}\bigr) &\leq \log \mathcal{N}\bigl(\epsilon_{n,\alpha}/3,\, \mathcal{P}_{\alpha},\, d_{\mathrm{H}}\bigr)\\
        &\leq \log \mathcal{N}\bigl(2\sigma \epsilon_{n,\alpha}/3,\, \mathcal{G}_{\alpha},\, \|\cdot\|_{\infty}\bigr) \\
        &\leq \log \mathcal{N}\bigl(2\sigma \epsilon_{n,\alpha}/3,\, B^{\alpha}_{\infty,\infty}([0,1]^d,2C_0),\, \|\cdot\|_{\infty}\bigr) \\
        &\leq E_{\alpha}n\epsilon_{n,\alpha}^2,
    \end{align*}
    for some $E_{\alpha}>0$ not depending on $n$; this verifies condition~\ref{eq:entropy-condition}. Moreover, since $\mathcal{A}$ is finite, condition~\ref{eq:mixture-weight-condition} is satisfied with $\mu_{n,\alpha}=1$ for all $\alpha\in\mathcal{A}$ when $n$ is sufficiently large. Similarly, condition~\ref{eq:E-condition} is satisfied for some $E>0$ not depending on $n$.  Finally, take $M>0$ such that $M^2 > 243(F+1)(HE+1)$. Letting $c,c'>0$ be the quantities in Lemma~\ref{lemma:prior-mass-ub} and applying Lemma~\ref{lemma:prior-mass-ub} with $m\coloneqq MC'/c'$, we deduce that for $n$ sufficiently large,
    \begin{align*}
        \sum_{\alpha\in\mathcal{A}_{<\beta}} \frac{\lambda_{\alpha}}{\lambda_{\beta}} \cdot \pi_{\alpha}\bigl(B'_{\alpha}(P_0,M\epsilon_{n,\alpha})\bigr) \leq \sum_{\alpha\in\mathcal{A}: \alpha<\beta} \frac{\lambda_{\alpha}}{\lambda_{\beta}} \cdot \exp\bigl(-c m^{-d/\alpha}n^{d/(2\alpha+d)}\bigr) \leq e^{-(F+3)n\epsilon_{n,\beta}^2},
    \end{align*}
    which verifies condition~\ref{eq:prior-mass-condition-alpha}. Therefore, by Corollary~\ref{cor:posterior-regression-function-concentration}, the claim follows for sufficiently large $n$, since $e^{-c_3n\epsilon_{n,\beta}^2} \leq \epsilon_{n,\beta}^2$ for $n$ large enough.  Moreover, since $g_\pi$ and $g^\circ$ are bounded, the conclusion now follows for all $n \in \mathbb{N}$ by increasing $C > 0$ in the statement if necessary.
\end{proof}

\subsection{Statement and proof of Proposition~\ref{prop:posterior-regression-function-multi-index}} \label{sec:posterior-contraction-multi-index}
We work in the setting of Section~\ref{sec:multi-index-regression}.  For $(\alpha,p)\in\mathcal{A}'$, we define $\mathcal{G}_{\alpha}^{(p)}$ similarly to the definition of $\mathcal{G}_{\alpha}$ in Section~\ref{sec:holder-regression} but replacing the dimensionality $d$ by $p$, i.e.~we let $\mathcal{G}_{\alpha}^{(p)}$ be the set of all functions $h_{\alpha}^{(p)}:[0,1]^p \to\mathbb{R}$ such that $h$ has wavelet decomposition 
    \begin{align*}
    h_{\alpha}^{(p)} = \sum_{k\in K} C_0 2^{-\ell_0 r/2} a_k \Phi_k^{(p)} + \sum_{\ell=\ell_0}^{\infty} \sum_{\gamma\in\Gamma_{\ell}} C_0 2^{-\ell(\beta + r/2)} b_{\ell,\gamma} \Psi_{\ell,\gamma}^{(p)}, 
    \end{align*}
where $(\Phi_k^{(p)})$ and $(\Psi_{\ell,\gamma}^{(p)})$ is the tensor product CDV wavelet basis for $L^2([0,1]^p)$, $|a_k| \leq 1$ and $|b_{\ell,\gamma}| \leq 1$. Thus, similarly to~\eqref{eq:inclusion-G-alpha}, 
\begin{align*}
    B_{\infty,\infty}^{\alpha}([0,1]^p,C_0) \subseteq \mathcal{G}_{\alpha}^{(p)} \subseteq B_{\infty,\infty}^{\alpha}([0,1]^p,2C_0).
\end{align*}
The proof of Proposition~\ref{prop:posterior-regression-function-multi-index} relies on three preliminary lemmas that we now formulate. We may write 
\begin{align}
    g^\circ(x) = h^\circ\biggl(\frac{(U^\circ)^\top x+ 1_r}{2}\biggr), \label{eq:def-g0-multi-index}
\end{align}
where $U^\circ \in V_r(\mathbb{R}^d)$ and $h^\circ\in\mathcal{G}_{\beta}^{(r)}$ has wavelet decomposition
\begin{align}
    h^\circ = \sum_{k\in K} C_0 2^{-\ell_0 r/2} a_k^\circ \Phi_k^{(r)} + \sum_{\ell=\ell_0}^{\infty} \sum_{\gamma\in\Gamma_{\ell}} C_0 2^{-\ell(\beta + r/2)} b_{\ell,\gamma}^\circ \Psi_{\ell,\gamma}^{(r)}, \label{eq:def-h0}
\end{align}
where $|a^\circ_k| \leq1$ and $|b^\circ_{\ell,\gamma}| \leq 1$. 

\begin{lemma} \label{lemma:prior-mass-lb-multi-index}
    Let $n\in\mathbb{N}$ and assume that~\eqref{eq:def-g0-multi-index} and~\eqref{eq:def-h0} hold. Then there exists $C>0$, depending only on $\beta,d,r,c_0,C_0$ and $S$, such that
    \begin{align}
        \tilde{\pi}_{\beta,r} \bigl(\bigl\{g\in \mathcal{G}_{\beta,r} : \|g-g^{\circ}\|_{\infty}^2 \leq C n^{-2\beta/(2\beta+r)} \bigr\}\bigr) \geq \exp\bigl(-C n^{r/(2\beta+r)} \bigr). \label{eq:prior-mass-lb-func-multi-index}
    \end{align}
    Moreover, there exists $C' > 0$, depending only on $\beta,d,r,c_0,C_0,S$ and $\sigma$, such that if $\epsilon_{n,\beta,r} \coloneqq C'n^{-\beta/(2\beta+r)}$, then
    \begin{align}
        \pi_{\beta,r}\bigl(B_{\beta,r}(P_{g^{\circ}}, \epsilon_{n,\beta,r})\bigr) \geq \exp\bigl(-C' n\epsilon_{n,\beta,r}^2 \bigr), \label{eq:prior-mass-lb-dist-multi-index}
    \end{align}
    where $B_{\beta,r}(P_{g^{\circ}}, \epsilon_{n,\beta,r}) \coloneqq \bigl\{ P\in\mathcal{P}_{\beta,r} : \mathrm{KL}(P_{g^{\circ}},P) \leq \epsilon_{n,\beta,r}^2,\, \mathrm{V}_2(P_{g^{\circ}},P) \leq \epsilon_{n,\beta,r}^2 \bigr\}$, similarly to~\eqref{eq:B-ball}.
\end{lemma}
\begin{proof}
    Throughout the proof, $C_1,C_2,\ldots$ will denote positive quantities that may depend only on $\beta,d,r,c_0,C_0$ and $S$, while $C_1',C_2',\ldots$ may in addition depend on $\sigma$.
    Let $\tilde{g}_{\beta,r}(x) = \tilde{g}_{\beta}^{(r)} \bigl(\frac{(U^{(r)})^\top x+1_r}{2}\bigr)$ be defined as in~\eqref{eq:random-g-alpha-p} so that $\tilde{g}_{\beta,r}\sim \tilde{\pi}_{\beta,r}$. Observe that $B_{\infty,\infty}^{\beta}([0,1]^r,2C_0) \subseteq B_{\infty,\infty}^{\beta\wedge 1}([0,1]^r,2C_0)$ \citep[e.g.][Proposition~4.3.10(ii)]{gine2021mathematical}. When $\beta<1$, the space $B_{\infty,\infty}^{\beta\wedge 1}$ is equal to the H\"older space $H^{\beta}$ with equivalent norms \citep[e.g.][Proposition~4.3.23]{gine2021mathematical}, so for all $x,y\in[0,1]^r$ and $g\in B_{\infty,\infty}^{\beta\wedge 1}([0,1]^r,2C_0)$, we have $|g(x) - g(y)| \leq C_1\|x-y\|_2^{\beta}$; when $\beta\geq 1$, $B_{\infty,\infty}^{\beta\wedge 1}$ is equal to the Zygmund space of order $1$ with equivalent norms \citep[p.~113]{triebel1983theory}, so by \citet[Proposition~2.4]{anderson1989probabilistic} \citep[see also][Chapter~II, Theorem~3.4]{zygmund2002trigonometric}, for all $x,y\in[0,1]^r$ and $g\in B_{\infty,\infty}^{\beta\wedge 1}([0,1]^r,2C_0)$, we have $|g(x) - g(y)| \leq C_2\|x-y\|_2 \bigl\{\log(1/\|x-y\|_2) \vee 1\bigr\}$. Therefore, for all $x,y\in[0,1]^r$ and $g\in B_{\infty,\infty}^{\beta}([0,1]^r,2C_0)$,
    \begin{align}
        |g(x) - g(y)| \leq C_3\|x-y\|_2^{\beta\wedge1} \bigl\{\log(1/\|x-y\|_2) \vee 1\bigr\}. \label{eq:lipschitzness-g-beta-r}
    \end{align} 
    On the event $E\coloneqq \bigl\{\|U^{(r)}-U^\circ\|_{\mathrm{op}} \leq \bigl(n^{-\beta/(2\beta+r)}/\log (en)\bigr)^{1/(\beta\wedge 1)}\bigr\}$, we have for all $x\in\mathbb{B}^d$ that
    \begin{align}
        &|\tilde{g}_{\beta,r}(x) - g^\circ(x)|\nonumber\\
        &\leq \biggl|\tilde{g}_{\beta}^{(r)}\biggl(\frac{(U^{(r)})^\top x+1_r}{2}\biggr) - \tilde{g}_{\beta}^{(r)}\biggl(\frac{(U^\circ)^\top x+1_r}{2}\biggr) \biggr| + \biggl| \tilde{g}_{\beta}^{(r)}\biggl(\frac{(U^\circ)^\top x+1_r}{2}\biggr) - h^\circ\biggl(\frac{(U^\circ)^\top x+1_r}{2}\biggr) \biggr|\nonumber\\
        &\leq C_4 n^{-\beta/\{(2\beta+r)\}} + \|\tilde{g}_{\beta}^{(r)} - h^\circ\|_{\infty}, \label{eq:g-beta-r-ub}
    \end{align}
    where the second inequality follows from~\eqref{eq:lipschitzness-g-beta-r}, since $\tilde{g}_{\beta}^{(r)}\in \mathcal{G}_{\beta}^{(r)} \subseteq B_{\infty,\infty}^{\beta}([0,1]^r,2C_0)$ and $\bigl\|\frac{(U^{(r)})^\top x+1_r}{2} - \frac{(U^\circ)^\top x+1_r}{2}\bigr\|_2 \leq \frac{1}{2} \bigl(n^{-\beta/(2\beta+r)}/\log (en)\bigr)^{1/(\beta\wedge 1)}$ on $E$. Now by Lemma~\ref{eq:small-ball-prob-stiefel-manifold}, we have \begin{align*}
        \mathbb{P}(E) \geq C_5 \biggl(\frac{n^{-\beta/(2\beta+r)}}{\log (en)} \wedge \frac{1}{2}\biggr)^{dr/(\beta\wedge 1)} \geq \exp(-C_6n^{r/(2\beta+r)}),
    \end{align*}
    and by~\eqref{eq:prior-mass-lb-func} (with $d$ replaced by~$r$ therein), we have \begin{align*}
        \mathbb{P}\bigl(\|\tilde{g}_{\beta}^{(r)} - h^\circ\|_{\infty} \leq C_7n^{-\beta/(2\beta+r)}\bigr) \geq \exp(-C_7n^{r/(2\beta+r)}).
    \end{align*}
    Hence, 
    \begin{align*}
        \mathbb{P}\bigl(\|\tilde{g}_{\beta,r} -g^\circ\|_{\infty}^2 &\leq (C_4+C_7)^2 n^{-2\beta/(2\beta+r)}\bigr)\\
        &\geq \mathbb{P}\bigl(\|\tilde{g}_{\beta,r}-g^\circ\|_{\infty} \leq (C_4+C_7) n^{-\beta/(2\beta+r)} \,\big|\, E\bigr) \mathbb{P}(E)\\
        &\geq \mathbb{P}\bigl(\|\tilde{g}_{\beta}^{(r)} - h^\circ\|_{\infty} \leq C_7n^{-\beta/(2\beta+r)}\bigr)\mathbb{P}(E) \geq \exp\bigl\{-(C_6+C_7)n^{r/(2\beta+r)}\bigr\},
    \end{align*}
    where the second inequality follows from~\eqref{eq:g-beta-r-ub} and the fact that $\tilde{g}_{\beta}^{(r)}$ and $U^{(r)}$ are independent. This proves~\eqref{eq:prior-mass-lb-func-multi-index} with $C\coloneqq (C_4+C_7)^2 \vee (C_6+C_7)$.

    The proof of~\eqref{eq:prior-mass-lb-dist-multi-index} follows the same argument as the proof of~\eqref{eq:prior-mass-lb-dist}. By Lemma~\ref{lemma:properties-of-gaussian-regression}\emph{(a)} and~\emph{(b)}, if $g\in\mathcal{G}_{\beta,r}$ satisfies $\|g - g^{\circ}\|_{\infty}^2 \leq Cn^{-2\beta/(2\beta+r)}$, then 
    \begin{align*}
        \mathrm{KL}(P_{g^{\circ}}, P_g) \leq C_1'Cn^{-2\beta/(2\beta+r)} \quad\text{and}\quad \mathrm{V}_2(P_{g^{\circ}}, P_g) \leq C_2'Cn^{-2\beta/(2\beta+r)}.
    \end{align*}
    Writing $C'\coloneqq C^{1/3}\vee\sqrt{(C_1'\vee C_2')C}$ and $\epsilon_{n,\beta,r} \coloneqq C' n^{-\beta/(2\beta+r)}$, we deduce from~\eqref{eq:prior-mass-lb-func-multi-index} that
    \begin{align*}
        \pi_{\beta,r}\bigl(B_{\beta,r}(P_{g^\circ}, \epsilon_{n,\beta,r})\bigr) \geq \exp\bigl(-Cn^{r/(2\beta+r)}\bigr) \geq \exp(-C'n \epsilon_{n,\beta,r}^2).
    \end{align*}
    This proves the claim.
\end{proof}

\begin{lemma} \label{lemma:prior-mass-ub-multi-index}
    Let $(\alpha,p)\in\mathcal{A}$, $m>0$ and assume that~\eqref{eq:def-g0-multi-index} and~\eqref{eq:def-h0} hold. Suppose that $P_X$ is a Borel distribution on $\mathbb{B}^d$ with the property that there exist $c_0>0$ and a closed Euclidean ball $A\subseteq\mathbb{B}^d$ of radius $\tau>0$ such that $P_X(A_0) \geq c_0\mathrm{Vol}_d(A_0)$ for all measurable $A_0\subseteq A$. Then there exist $c_0'\equiv c_0'(c_0,\tau,d,p)$, $c \equiv c(\alpha,d,p,c_0,C_0,\tau) > 0$ and $n_0 \equiv n_0(\alpha,d,p,c_0,C_0,S,m,\tau) \in \mathbb{N}$,  such that 
    \begin{align}
        \tilde{\pi}_{\alpha,p}\Bigl(\Bigl\{g\in \mathcal{G}_{\alpha,p} : \|g-g^{\circ}\|_{L^2(P_X)}^2 \leq c_0'm^2n^{-2\alpha/(2\alpha+p)} \Bigr\}\Bigr) \leq \exp\bigl( -c m^{-p/\alpha} n^{p/(2\alpha+p)} \bigr)\label{eq:prior-mass-ub-func-multi-index}
    \end{align}
    for $n \geq n_0$.  Moreover, there exists $c'>0$, depending only on $\alpha,\beta,d,p,c_0,C_0,S,\tau$ and $\sigma$, such that for $n\geq n_0$,
    \begin{align}
        \pi_{\alpha,p}\bigl(B'_{\alpha,p}(P_{g^{\circ}}, c'mn^{-\alpha/(2\alpha+p)})\bigr) \leq \exp\bigl( -c m^{-p/\alpha} n^{p/(2\alpha+p)} \bigr) \label{eq:prior-mass-ub-dist-multi-index}
    \end{align}
    where $B'_{\alpha,p}(P_{g^{\circ}}, c'mn^{-\alpha/(2\alpha+p)}) \coloneqq \bigl\{P\in\mathcal{P}_{\alpha,p} : d_{\mathrm{H}}(P_{g^{\circ}},P) \leq c'mn^{-\alpha/(2\alpha+p)} \bigr\}$ is defined as in~\eqref{eq:C-ball}.
\end{lemma}
\begin{proof}
    Throughout the proof, $C_1,C_2,\ldots$ will denote positive quantities that may depend only on $\alpha,d,p,c_0,C_0$ and $\tau$. Let $c_0'\equiv c_0'(c_0,\tau,d,p)$ be the quantity in Lemma~\ref{lemma:density-lb}.
    % Let $Q$ be the distribution of $\frac{(U^\circ)^\top X+1}{2}$ with $X\sim P_X$ independent of $U^{(p)}$. By \red{Lemma XXX}, there exists a non-empty open set $A'\subseteq[0,1]^p$ on which the restriction of $Q$ to $A'$ is absolutely continuous with respect to Lebesgue measure with Radon–Nikodym derivative bounded below by some $c_0'>0$.
    Let $\tilde{g}_{\alpha,p}$ be defined as in~\eqref{eq:random-g-alpha-p} so that $\tilde{g}_{\alpha,p}\sim\tilde{\pi}_{\alpha,p}$. 
    Then
    \begin{align}
        &\tilde{\pi}_{\alpha,p}\Bigl(\Bigl\{g\in \mathcal{G}_{\alpha,p} : \|g-g^{\circ}\|_{L^2(P_X)}^2 \leq c_0'm^2n^{-2\alpha/(2\alpha+p)} \Bigr\}\Bigr)\nonumber\\
        &= \mathbb{P}\biggl\{\biggl\| \tilde{g}_{\alpha}^{(p)} \biggl(\frac{(U^{(p)})^\top x+1_p}{2}\biggr) - h^\circ\biggl(\frac{(U^\circ)^\top x+1_p}{2}\biggr) \biggr\|_{L^2(P_X)}^2 \leq c_0'm^2n^{-2\alpha/(2\alpha+p)} \biggr\}\nonumber\\
        &= \mathbb{E}\biggl[ \mathbb{P}\biggl\{\biggl\| \tilde{g}_{\alpha}^{(p)} \biggl(\frac{(U^{(p)})^\top x+1_p}{2}\biggr) - h^\circ\biggl(\frac{(U^\circ)^\top x+1_p}{2}\biggr) \biggr\|_{L^2(P_X)}^2 \leq c_0'm^2n^{-2\alpha/(2\alpha+p)} \,\bigg|\, U^{(p)} \biggr\}\biggr]. \label{eq:g-alpha-p-tower-property}
    \end{align}
    % We further define $\mathcal{G}_{\alpha}^{(p)}$ similarly to the definition of $\mathcal{G}_{\alpha}$ in Section~\ref{sec:holder-regression} but replacing the dimensionality $d$ by $p$, i.e.~we let $\mathcal{G}_{\alpha}^{(p)}$ be the set of all functions $g^*:[0,1]^p \to\mathbb{R}$ such that $g^*$ has wavelet decomposition 
    % \begin{align*}
    % g^* = \sum_{k\in K} C_0 2^{-\ell_0 r/2} a_k^* \Phi_k^{(p)} + \sum_{\ell=\ell_0}^{\infty} \sum_{\gamma\in\Gamma_{\ell}} C_0 2^{-\ell(\beta + r/2)} b_{\ell,\gamma}^* \Psi_{\ell,\gamma}^{(p)}, 
    % \end{align*}
    % where $(\Phi_k^{(p)})$ and $(\Psi_{\ell,\gamma}^{(p)})$ is the tensor product CDV wavelet basis for $L^2([0,1]^p)$, $|a^*_k| \leq 1$ and $|b^*_{\ell,\gamma}| \leq 1$. Note that in particular, $\mathcal{G}_{\alpha}^{(p)}$ is the support of the random function~$\tilde{g}_{\alpha}^{(p)}$.
    Now conditional on $U^{(p)}$, we fix $g'\in\mathcal{G}_{\alpha}^{(p)}$ such that $\bigl\|g'\bigl(\frac{(U^{(p)})^\top x+1_p}{2}\bigr)-h^\circ\bigl(\frac{(U^{\circ})^\top x+1_p}{2}\bigr)\bigr\|_{L^2(P_X)}^2 \leq c_0'm^2n^{-2\alpha/(2\alpha+p)}$. We may assume without loss of generality that such $g'$ exists, since otherwise the conditional probability in the expectation in~\eqref{eq:g-alpha-p-tower-property} would be zero. Using the same argument as that in~\eqref{eq:center-at-g'}, we deduce that
    \begin{align}
        \mathbb{P}&\biggl\{\biggl\| \tilde{g}_{\alpha}^{(p)} \biggl(\frac{(U^{(p)})^\top x+1}{2}\biggr) - h^\circ\biggl(\frac{(U^\circ)^\top x+1_p}{2}\biggr) \biggr\|_{L^2(P_X)}^2 \leq c_0'm^2n^{-2\alpha/(2\alpha+p)} \,\bigg|\, U^{(p)} \biggr\}\nonumber\\
        &\leq \mathbb{P}\biggl\{\biggl\| \tilde{g}_{\alpha}^{(p)} \biggl(\frac{(U^{(p)})^\top x+1_p}{2}\biggr) - g'\biggl(\frac{(U^{(p)})^\top x+1_p}{2}\biggr) \biggr\|_{L^2(P_X)}^2 \leq 4c_0' m^2n^{-2\alpha/(2\alpha+p)} \,\bigg|\, U^{(p)} \biggr\}\nonumber\\
        &= \mathbb{P}\bigl(\|\tilde{g}_{\alpha}^{(p)} - g'\|_{L^2(Q(U^{(p)}))}^2 \leq 4c_0' m^2n^{-2\alpha/(2\alpha+p)} \,\big|\, U^{(p)}\bigr)\label{eq:ub-prior-mass-g-alpha-p-1},
    \end{align}
    where $Q(U^{(p)})$ denotes the conditional distribution of $\frac{(U^{(p)})^\top X+1_p}{2}$ given $U^{(p)}$, with $X\sim P_X$. By Lemma~\ref{lemma:density-lb}, for every $U\in V_p(\mathbb{R}^d)$, there exists a hypercube $A'\subseteq [0,1]^p$ of side length~$\frac{\tau}{2\sqrt{p}}$ such that $Q(U)(A_0') \geq c_0'\mathrm{Vol}_p(A_0')$ for all measurable $A_0'\subseteq A'$. 
    % Moreover, note that $\tilde{g}_{\alpha}^{(p)}$ is distributed as $\tilde{g}_{\alpha}$ in~\eqref{eq:random-g-alpha}, but with $d$ replaced by $p$ in the wavelet coefficients as well as in the tensor product CDV wavelet basis.
    We may thus apply~\eqref{eq:prior-mass-ub-func}, with $d$ replaced by $p$, $m$ replaced by $4m$ and $c_0$ replaced by $c_0'$ therein, to deduce that
    \begin{align*}
        \mathbb{P}\bigl\{\bigl\| \tilde{g}_{\alpha}^{(p)} - g'\bigr\|_{L^2(Q(U^{(p)}))}^2 \leq 4c_0' m^2n^{-2\alpha/(2\alpha+p)} \,\big|\, U^{(p)} \bigr\} \leq \exp\bigl( -c m^{-p/\alpha} n^{p/(2\alpha+p)} \bigr),
    \end{align*}
    for $n\geq n_0$.
    Combining this with~\eqref{eq:g-alpha-p-tower-property} and~\eqref{eq:ub-prior-mass-g-alpha-p-1} proves~\eqref{eq:prior-mass-ub-func-multi-index}.

    Finally, the proof of~\eqref{eq:prior-mass-ub-dist-multi-index} follows a very similar argument to that in~\eqref{eq:prior-mass-ub-dist}.
\end{proof}

\begin{lemma} \label{lemma:entropy-multi-index}
    Let $(\alpha,p)\in\mathcal{A}$, $C>0$, $P_0\in\mathcal{P}_{\alpha,p}$ and $\epsilon_{n,\alpha,p} \coloneqq Cn^{-\alpha/(2\alpha+p)}$. Then there exists $C'$ depending only on $\alpha,p,d,\sigma,C$ and $C_0$ such that
    \begin{align*}
        \sup_{\epsilon \geq \epsilon_{n,\alpha,p}} \log\mathcal{N}\bigl(\epsilon/3, B_{\alpha,p}'(P_0,2\epsilon), d_{\mathrm{H}}\bigr) \leq C' n \epsilon_{n,\alpha,p}^2.
    \end{align*}
\end{lemma}
\begin{proof}
    Throughout this proof, $C_1,C_2,\ldots$ will denote positive quantities depending only on $\alpha,p,d,\sigma,C$ and $C_0$.
    By Lemma~\ref{lemma:properties-of-gaussian-regression}\emph{(d)},
    \begin{align*}
        \sup_{\epsilon\geq\epsilon_{n,\alpha,p}} \log\mathcal{N}\bigl(\epsilon/3,\, B'_{\alpha,p}(P_0,2\epsilon),\, d_{\mathrm{H}}\bigr) &\leq \log \mathcal{N}\bigl(\epsilon_{n,\alpha,p}/3,\, \mathcal{P}_{\alpha,p},\, d_{\mathrm{H}}\bigr)\\
        &\leq \log \mathcal{N}\bigl(2\sigma \epsilon_{n,\alpha,p}/3,\, \mathcal{G}_{\alpha,p},\, \|\cdot\|_{\infty}\bigr).
    \end{align*}
    Recall that $\mathcal{G}_{\alpha}^{(p)} \subseteq B_{\infty,\infty}^{\alpha}([0,1]^p,2C_0)$ and $\mathcal{G}_{\alpha,p} = \mathcal{G}_{\alpha}^{(p)} \circ \bigl\{x\mapsto \frac{U^\top x + 1_p}{2} : U\in V_p(\mathbb{R}^d)\bigr\}$. Now let $\mathcal{N}_1$ be a $(\sigma\epsilon_{n,\alpha,p}/3)$-cover of $B_{\infty,\infty}^{\alpha}([0,1]^p,2C_0)$ with respect to $\|\cdot\|_{\infty}$ such that $|\mathcal{N}_1| \leq \exp(C_1n\epsilon_{n,\alpha,p}^2)$ for some $C_1>0$, which exists by, e.g.~\citet[][Eq.~(4.184)]{gine2021mathematical}. Recall from~\eqref{eq:lipschitzness-g-beta-r} that for all $x,y\in[0,1]^p$,
    \begin{align}
        |g(x) - g(y)| \leq C_2\|x-y\|_2^{\alpha\wedge1} \bigl\{\log(1/\|x-y\|_2) \vee 1\bigr\}.  \label{eq:lipschitz-weak-form}
    \end{align} 
    The function $z \mapsto z^{\alpha \wedge 1}\log(1/z)$ is increasing on $(0,e^{-1/(\alpha \wedge 1)}]$, so let $\delta\coloneqq \frac{1}{\sqrt{pd}}\bigl(\frac{\sigma\epsilon_{n,\alpha,p}}{3C_3\log (en)}\bigr)^{1/(\alpha\wedge 1)} \wedge \frac{1}{\sqrt{pd}} e^{-1/(\alpha \wedge 1)}$ for $C_3>0$ large enough such that 
    \begin{align}
        C_2(\delta\sqrt{pd})^{\alpha\wedge 1} \biggl\{\log\biggl(\frac{1}{\delta\sqrt{pd}}\biggr)\vee 1\biggr\} \leq \frac{\sigma\epsilon_{n,\alpha,p}}{3}. \label{eq:choice-of-C3}
    \end{align}
    Let $\Delta\coloneqq \{i\delta : i\in\mathbb{Z}, |i\delta|\leq 1\}$ and $\mathcal{N}_2\coloneqq \Delta^{d\times p} \subseteq \mathbb{R}^{d\times p}$. Then for any $g\in\mathcal{G}_{\alpha}^{(p)} \subseteq B_{\infty,\infty}^{\alpha}([0,1]^p,2C_0)$ and $U\in V_p(\mathbb{R}^d)$, there exist $g'\in\mathcal{N}_1$ and $U'\in \mathcal{N}_2$ such that $\|g'-g\|_{\infty} \leq \sigma\epsilon_{n,\alpha,p}/3$ and $\|U'-U\|_{\mathrm{op}} \leq \sqrt{pd}\|U'-U\|_{\infty} \leq \delta\sqrt{pd}$. Thus, for all $x\in\mathbb{B}^d$, we have $\|U^\top x - (U')^\top x\|_2 \leq \delta\sqrt{pd}$, so by~\eqref{eq:lipschitz-weak-form} and~\eqref{eq:choice-of-C3}, we deduce that
    \begin{align*}
        \biggl|g\biggl(&\frac{U^\top x + 1_p}{2}\biggr) - g'\biggl(\frac{(U')^\top x + 1_p}{2}\biggr)\biggr|\\
        &\leq \biggl|g\biggl(\frac{U^\top x + 1_p}{2}\biggr) - g\biggl(\frac{(U')^\top x + 1_p}{2}\biggr)\biggr| + \biggl|g\biggl(\frac{(U')^\top x+ 1_p}{2}\biggr) - g'\biggl(\frac{(U')^\top x+ 1_p}{2}\biggr)\biggr|\\
        &\leq \frac{2\sigma\epsilon_{n,\alpha,p}}{3}.
    \end{align*}
    Therefore, $\mathcal{N}_1 \circ \mathcal{N}_2$ is a $(2\sigma\epsilon_{n,\alpha,p}/3)$-cover of $\mathcal{G}_{\alpha,p}$ with respect to $\|\cdot\|_{\infty}$. Hence, we have 
    \begin{align}
        \log \mathcal{N}\bigl(2\sigma \epsilon_{n,\alpha,p}/3,\, \mathcal{G}_{\alpha,p},\, \|\cdot\|_{\infty}\bigr) \leq \log|\mathcal{N}_1| + \log|\mathcal{N}_2| \leq C'n\epsilon_{n,\alpha,p}^2, \label{eq:ell-infty-entropy-multi-index}
    \end{align}
    which proves the claim.
\end{proof}

\begin{prop} \label{prop:posterior-regression-function-multi-index}
    For any $(\beta,r)\in\mathcal{A}'$ and $g^\circ\in\mathcal{G}_{\beta,r}$, write $P_0 \equiv P_{g^\circ}$, let $\pi$ be the prior distribution defined in~\eqref{eq:mixture-training-prior-distribution-multi-index-model} and let $g_{\pi}$ denote the posterior regression function with respect to~$\pi$. Then there exists $C>0$, not depending on $n$, such that 
    \begin{align*}
    \mathbb{E}_{P_0} \bigl[\bigl\{g_\pi(X,\mathcal{D}_n) - g^\circ(X)\bigr\}^2\bigr] \leq Cn^{-2\beta/(2\beta+r)}
    \end{align*}
    for all $n \in \mathbb{N}$.
\end{prop}
%\red{Maybe remark that our proof works for compositions of H\"older functions as well, but we consider multi-index model to keep the notation simple.}
\begin{proof}[Proof of Proposition~\ref{prop:posterior-regression-function-multi-index}]
    The proof of Proposition~\ref{prop:posterior-regression-function-multi-index} follows the same arguments as the proof of Proposition~\ref{prop:posterior-regression-function-wavelet-prior}, but instead we apply Lemmas~\ref{lemma:prior-mass-lb-multi-index}, \ref{lemma:prior-mass-ub-multi-index} and~\ref{lemma:entropy-multi-index} to verify the assumptions of Corollary~\ref{cor:posterior-regression-function-concentration}.
\end{proof}

\subsection{Proofs of Theorems~\ref{thm:holder-adaptive-icl} and~\ref{thm:multi-index-adaptive-icl}}

\begin{proof}[Proof of Theorem~\ref{thm:holder-adaptive-icl}]
    For any test distribution $\mu$ on $\mathcal{P}_\beta$ such that $\chi^2(\mu,\pi_{\beta})\leq\kappa$, we have
    \begin{align*}
        \chi^2(\mu,\pi) = \mathbb{E}_{\pi}\biggl\{\biggl(\frac{\d\mu}{\d\pi}\biggr)^2\biggr\} - 1 \leq \max_{\alpha\in\mathcal{A}} \frac{1}{\lambda_\alpha} \cdot \mathbb{E}_{\pi_{\beta}}\biggl\{\biggl(\frac{\d\mu}{\d\pi_{\beta}}\biggr)^2\biggr\} - 1 \leq (\kappa+1) \max_{\alpha\in\mathcal{A}} \frac{1}{\lambda_\alpha} - 1 \eqqcolon \kappa'.
    \end{align*}
    Now, by Proposition~\ref{prop:R_mu-decomposition}, then by taking $d_{\mathrm{model}}^\circ \geq d+2$, $d_{\mathrm{ffn}}^\circ \in \mathbb{N}$ and $T \in \mathbb{N}$ large enough that Proposition~\ref{prop:ERM-approx-posterior} holds with $\epsilon = 1/\{16R^2n^2(\kappa'+1)\}$ and applying Proposition~\ref{prop:posterior-regression-function-wavelet-prior}, we have 
    \begin{align*}
        \mathcal{R}_{\mu}^{\mathrm{ICL}}(\tilde f_T) &\leq 4R\sqrt{\bigl\{\chi^2(\mu,\pi)+1\bigr\}\mathcal{E}(\tilde f_T)} + 2\mathbb{E}_\mu\mathbb{E}_P\bigl[\bigl\{g_\pi(X,\mathcal{D}_n)-g^\circ(X)\bigr\}^2\bigr]\\
        &\leq \frac{1}{n} + C'n^{-2\beta/(2\beta+d)} \leq Cn^{-2\beta/(2\beta+d)},
    \end{align*}
    where $C',C>0$ depend on neither $n$ nor $\kappa$.
\end{proof}

\begin{proof}[Proof of Theorem~\ref{thm:multi-index-adaptive-icl}]
    As in the proof of Theorem~\ref{thm:holder-adaptive-icl}, for any test distribution $\mu$ on $\mathcal{P}_{\beta,r}$ such that $\chi^2(\mu, \pi_{\beta,r}) \leq \kappa$, we have $\chi^2(\mu, \pi) \leq (\kappa+1) \max_{(\alpha,p) \in \mathcal{A}'} \lambda_{\alpha,p}^{-1} - 1 \eqqcolon \kappa'$.
    The result therefore follows by the same argument as in the proof of Theorem~\ref{thm:holder-adaptive-icl}, except that we apply Proposition~\ref{prop:posterior-regression-function-multi-index} in place of Proposition~\ref{prop:posterior-regression-function-wavelet-prior} to obtain the improved bound $\mathcal{R}_{\mu}^{\mathrm{ICL}}(\tilde f_T) \leq Cn^{-2\beta/(2\beta+r)}$ for sufficiently large $T$, where $C>0$ depends on neither $n$ nor $\kappa$.
    % The fact that $C$ does not depend on $\kappa$ follows from the fact that the first term in the right-hand side of~\eqref{eq:ICL-excess-risk-ub} can be made smaller than $1/n$ by choosing $d_{\mathrm{model}}^\circ,d_{\mathrm{ffn}}^\circ,T$ sufficiently large such that $\mathcal{E}(\tilde{f}_T)\leq \bigl\{ (4nR)^2 (\kappa'+1) \bigr\}^{-1}$, and the second term in the right-hand side of~\eqref{eq:ICL-excess-risk-ub} does not depend on $\gamma$, by Proposition~\ref{prop:posterior-regression-function-multi-index}.
\end{proof}

\subsection{Proof of Theorem~\ref{thm:bayes-risk-lb}}
\begin{proof}[Proof of Theorem~\ref{thm:bayes-risk-lb}]
    We use the Bayes risk lower bound techniques developed by \citet{chen2016bayes}. Let $(\mathcal{P}_*,\mathcal{G}_*)$ be either $(\mathcal{P}_{\beta}, \mathcal{G}_{\beta})$ defined as in Section~\ref{sec:holder-regression} or $(\mathcal{P}_{\beta,r}, \mathcal{G}_{\beta,r})$ defined as in Section~\ref{sec:multi-index-regression}, and let $\tilde{\mu}$ denote the distribution on $\mathcal{G}_*$ induced by $\mu$ on $\mathcal{P}_*$, so that~$\tilde{\mu}$ is the distribution of the random function $x\mapsto\mathbb{E}_P(Y\,|\,X=x)$, where the randomness comes from $P\sim\mu$. By \citet[Corollary~12\emph{(i)} and Eq.~(48)]{chen2016bayes}, we have
    \begin{align}
        \inf_{\hat{g}_n\in\hat{\mathcal{G}}_n} \mathcal{R}_{\mu}^{\mathrm{ICL}}(\hat{g}_n) &=
        \inf_{\hat{g}_n\in\hat{\mathcal{G}}_n} \mathbb{E}_{P\sim \mu}\mathbb{E}_{\mathcal{D}_n|P \sim P^{\otimes n}} \bigl\{\|\hat{g}_n(\cdot, \mathcal{D}_n) - g_P\|_{L^2(P_X)}^2\bigr\}\nonumber\\
        &= \inf_{\hat{g}_n\in\hat{\mathcal{G}}_n} \mathbb{E}_{\tilde{g}\sim \tilde{\mu}}\mathbb{E}_{\mathcal{D}_n|\tilde{g} \sim P_{\tilde{g}}^{\otimes n}} \bigl\{\|\hat{g}_n(\cdot, \mathcal{D}_n) - \tilde{g}\|_{L^2(P_X)}^2\bigr\}\nonumber\\
        &\geq \frac{1}{2} \sup \biggl\{t>0 : \sup_{g^\circ\in\mathcal{G}_*} \tilde{\mu}\Bigl( \bigl\{g \in\mathcal{G}_* : \|g-g^\circ\|_{L^2(P_X)}^2 < t \bigr\}\Bigr) < \frac{1}{4}e^{-2I^{\mathrm{up}}_{\mathrm{KL}}}\biggr\}, \label{eq:general-bayes-risk-lb}
    \end{align}
    where 
    \begin{align*}
        I^{\mathrm{up}}_{\mathrm{KL}} \coloneqq \inf_{\eta>0} \bigl\{\log \mathcal{N}(\eta^2/n,\mathcal{P}_*,\mathrm{KL}) + \eta^2\bigr\}
    \end{align*}
    and we recall that $\mathcal{N}(\eta^2/n,\mathcal{P}_*,\mathrm{KL})$ is the $(\eta^2/n)$-covering number of $\mathcal{P}_*$ with respect to the Kullback--Leibler divergence. By Lemma~\ref{lemma:properties-of-gaussian-regression}\emph{(a)}, we deduce that
    \begin{align*}
        I^{\mathrm{up}}_{\mathrm{KL}} \leq \inf_{\eta>0} \bigl\{\log \mathcal{N}\bigl(\sigma\eta\sqrt{2/n},\mathcal{G}_*,\|\cdot\|_{\infty}\bigr) + \eta^2\bigr\}.
    \end{align*}
    From now on, $C_1,C_2,\ldots$ will denote positive quantities that do not depend on $n$.

    \medskip
    \emph{(a)} By choosing $\eta\coloneqq n^{d/(2\beta+d)}$ and \citet[Eq.~(4.184)]{gine2021mathematical}, we deduce that
    \begin{align*}
        I^{\mathrm{up}}_{\mathrm{KL}} \leq C_1 n^{d/(2\beta+d)}.
    \end{align*}
    %By assumption, there exists $C_2>0$ such that $\d\tilde{\mu}/\d\tilde{\pi}_{\beta}$ is uniformly bounded by $C_2$. 
    By the measurable bijection between $\mathcal{G}_{\beta}$ and $\mathcal{P}_{\beta}$ and our assumption on $\mu$, we have that  $\sup_{g\in\mathcal{G}_{\beta}}\frac{\d\tilde{\mu}}{\d\tilde{\pi}_{\beta}}(g) < \infty$.
    Thus, by Lemma~\ref{lemma:prior-mass-ub}, there exists $C_2>0$ small enough that for all $g^\circ\in\mathcal{G}_{\beta}$,
    \begin{align*}
        \tilde{\mu}\Bigl( \bigl\{g\in\mathcal{G}_{\beta} &: \|g -g^\circ\|_{L^2(P_X)}^2 < C_2n^{-2\beta/(2\beta+d)} \bigr\}\Bigr)\\
        & \leq \biggl\{ \sup_{g\in\mathcal{G}_{\beta}}\frac{\d\tilde{\mu}}{\d\tilde{\pi}_{\beta}}(g)\biggr\} \tilde{\pi}_{\beta} \Bigl( \bigl\{g\in\mathcal{G}_{\beta} : \|g-g^\circ\|_{L^2(P_X)}^2 < C_2n^{-2\beta/(2\beta+d)} \bigr\}\Bigr) < \frac{1}{4}e^{-2I^{\mathrm{up}}_{\mathrm{KL}}}.
    \end{align*}
    The claim then follows from~\eqref{eq:general-bayes-risk-lb}.

    \medskip
    \emph{(b)} By choosing $\eta\coloneqq n^{d/(2\beta+r)}$ and applying~\eqref{eq:ell-infty-entropy-multi-index} with $(\beta,r)$ replacing $(\alpha,p)$, we deduce that
    \begin{align*}
        I^{\mathrm{up}}_{\mathrm{KL}} \leq C_3 n^{r/(2\beta+r)}.
    \end{align*}
    Similarly, by Lemma~\ref{lemma:prior-mass-ub-multi-index}, there exists $C_4>0$ small enough that for all $g^\circ\in\mathcal{G}_{\beta,r}$,
    \begin{align*}
        \tilde{\mu}\Bigl( \bigl\{g&\in\mathcal{G}_{\beta,r} : \|g-g^\circ\|_{L^2(P_X)}^2 < C_4n^{-2\beta/(2\beta+r)} \bigr\}\Bigr)\\
        & \leq \biggl\{ \sup_{g\in\mathcal{G}_{\beta,r}}\frac{\d\tilde{\mu}}{\d\tilde{\pi}_{\beta,r}}(g)\biggr\}\tilde{\pi}_{\beta,r} \Bigl( \bigl\{g\in\mathcal{G}_{\beta,r} : \|g-g^\circ\|_{L^2(P_X)}^2 < C_4n^{-2\beta/(2\beta+r)} \bigr\}\Bigr) < \frac{1}{4}e^{-2I^{\mathrm{up}}_{\mathrm{KL}}}.
    \end{align*}
    The claim then follows again from~\eqref{eq:general-bayes-risk-lb}.
\end{proof}

\section{Auxiliary lemmas}
The lemma below is a modification of \citet[Lemma~6.3]{ghosal2008nonparametric}, where we assume that the log-likelihood ratio is sub-Gamma in the left tail with uniform sub-Gamma parameters in order to obtain an exponential concentration.

\begin{lemma}\label{lemma:evidence-lower-bound}
    Let $\mathcal{P}$ be a measurable space of distributions on a measurable space $\mathcal{Z}$, let $P_0$ be a distribution on $\mathcal{Z}$ and let $Z_1,\ldots,Z_n \overset{\mathrm{iid}}{\sim} P_0$. For $\epsilon>0$, define 
    \begin{align*}
        B(P_0,\epsilon) \coloneqq \Bigl\{P\in\mathcal{P}: \mathrm{KL}(P_0,P) \leq \epsilon^2,\, \mathrm{V}_2(P_0,P) \leq \epsilon^2 \Bigr\}.
    \end{align*}
    For $i\in[n]$ and $P\in\mathcal{P}$ with $P \ll P_0$, let 
    \[
    W(P,Z_i) \coloneqq \log \frac{\d P}{\d P_0}(Z_i) - \mathbb{E}\log \frac{\d P}{\d P_0}(Z_i).
    \]
    Assume that there exist $a,c>0$ such that for all $P\in B(P_0,\epsilon)$, we have that $W(P,Z_1)$ is sub-Gamma in the left tail with variance parameter $a\epsilon^2$ and scale parameter $c$.
    Then, for any probability measure $\Pi$ on $\mathcal{P}$ and any $\epsilon,D>0$, we have with $P_0^{\otimes n}$-probability at most $\exp\bigl(-\frac{D^2n\epsilon^2}{2(a+cD)}\bigr)$ that
    \begin{align}
    \label{eq:evidence-lb}
        \int_{B(P_0,\epsilon)} \prod_{i=1}^n \frac{\mathrm{d}P}{\mathrm{d}P_0}(Z_i) \,\mathrm{d}\Pi(P) < \Pi\bigl(B(P_0,\epsilon)\bigr)e^{-(1+D)n\epsilon^2}.
    \end{align}
\end{lemma}
\begin{proof}
    % Writing $B\coloneqq B(P_0,\epsilon)$, it suffices to show that
    % \begin{align*}
    %     \frac{1}{\Pi(B)} \int_{B} \prod_{i=1}^n \frac{\mathrm{d}P}{\mathrm{d}P_0}(Z_i) \,\mathrm{d}\Pi(P) \geq e^{-(1+D)n\epsilon^2}.
    % \end{align*}
    % The left-hand side is the expectation with respect to the distribution of $\Pi$ conditional on $B$. Thus we may assume without loss of generality that $\Pi$ is supported on $B$, and we show that
    % \begin{align}
    %     \int \prod_{i=1}^n \frac{\mathrm{d}P}{\mathrm{d}P_0}(Z_i) \,\mathrm{d}\Pi(P) \geq e^{-(1+D)n\epsilon^2} \label{eq:reduced-evidence-lower-bound}
    % \end{align}
    % for all $\Pi$ such that $\Pi(B)=1$. 
    Writing $B\coloneqq B(P_0,\epsilon)$, we observe that the conclusion is clear if $\Pi(B) = 0$ because the right-hand side of~\eqref{eq:evidence-lb} is zero.  Moreover, by replacing $\Pi$ with its conditional distribution on $B$ if necessary, we may assume without loss of generality that $\Pi(B) = 1$.   Now, by Jensen's inequality, 
    \begin{align}
        \frac{1}{n}\log \int_{B} \prod_{i=1}^n \frac{\mathrm{d}P}{\mathrm{d}P_0}(Z_i) \,\mathrm{d}\Pi(P) &\geq \frac{1}{n}\sum_{i=1}^n \int_{B} \log \frac{\d P}{\d P_0}(Z_i) \,\d \Pi(P)\nonumber\\
        &= \frac{1}{n}\sum_{i=1}^n \int_{B} W(P,Z_i) \,\d \Pi(P) - \int_{B} \mathrm{KL}(P_0,P) \,\d \Pi(P)\nonumber\\
        &\geq \frac{1}{n}\sum_{i=1}^n \int_{B} W(P,Z_i) \,\d \Pi(P) - \epsilon^2, \label{eq:log-int-ll}
    \end{align}
    where the final inequality follows from the definition of $B$.  Moreover, for $\lambda \in (-1/c,0]$, by Jensen's inequality and Fubini's theorem,
    \begin{align*}
        \mathbb{E}_{Z_1} \exp\biggl\{ \lambda \int_{B} W(P,Z_1) \,\d \Pi(P)\biggr\} &\leq \mathbb{E}_{Z_1} \biggl[ \int_{B} \exp\bigl\{\lambda  W(P,Z_1)\bigr\} \,\d \Pi(P)\biggr]\\
        &=  \int_{B} \mathbb{E}_{Z_1} \exp\bigl\{  \lambda W(P,Z_1)\bigr\} \,\d \Pi(P) \leq \exp\biggl(\frac{a\epsilon^2\lambda^2}{2(1-c\lambda)}\biggr),
    \end{align*}
    where the final inequality follows from the sub-Gamma assumption on $W(P,Z_1)$.  This shows that $\int_{B} W(P,Z_1) \,\d \Pi(P)$ is sub-Gamma in the left tail with variance parameter $a\epsilon^2$ and scale parameter $c$, so $n^{-1}\sum_{i=1}^n \int_{B} W(P,Z_i) \,\d \Pi(P)$ is also sub-Gamma in the left tail with variance parameter $a\epsilon^2/n$ and scale parameter $c/n$. Hence, by~\eqref{eq:log-int-ll},
    \begin{align*}
        \mathbb{P}\biggl( \frac{1}{n}\log \int_{B} \prod_{i=1}^n &\frac{\mathrm{d}P}{\mathrm{d}P_0}(Z_i) \,\mathrm{d}\Pi(P) < -(1+D)\epsilon^2 \biggr)\\
        &\leq \mathbb{P}\biggl( \frac{1}{n} \sum_{i=1}^n \int_{B} W(P,Z_i) \,\d \Pi(P) < -D\epsilon^2 \biggr) \leq \exp\biggl(-\frac{D^2n\epsilon^2}{2(a+cD)}\biggr),
    \end{align*}
    where the final inequality follows from Bernstein's inequality \citep[Corollary 2.11]{boucheron2013concentration}. 
\end{proof}

We next derive some properties about the distributions of nonparametric regression models with Gaussian noise.
\begin{lemma}\label{lemma:properties-of-gaussian-regression}
    Let $P_X$ be a Borel distribution on $\mathcal{X}\subseteq \mathbb{R}^d$, let $X\sim P_X$ and let $\xi\sim N(0,\sigma^2)$ be independent of $X$. Further, let $R>0$, $g_1,g_2:\mathcal{X} \to [-R,R]$ be Borel measurable, let $Y_1=g_1(X)+\xi$ and $Y_2=g_2(X)+\xi$. Finally, let $P_1$ and $P_2$ denote the distributions of $(X,Y_1)$ and $(X,Y_2)$ respectively. Then
    \begin{itemize}
        \item[(a)] $\mathrm{KL}(P_1,P_2) = \frac{1}{2\sigma^2} \mathbb{E}\bigl\{\bigl(g_1(X)-g_2(X)\bigr)^2\bigr\}$.
        \item[(b)] $\mathrm{V}_2(P_1,P_2) = \frac{1}{2\sigma^2} \Var\bigl\{\bigl(g_1(X)-g_2(X)\bigr)^2\bigr\} + 2\mathbb{E}\bigl\{ \bigl(g_1(X)-g_2(X)\bigr)^2 \bigr\}$.
        \item[(c)] If $\mathrm{KL}(P_1,P_2)\leq\epsilon^2$ and $\mathrm{V}_2(P_1,P_2)\leq\epsilon^2$, then $W\coloneqq \log \frac{\mathrm{d}P_2}{\mathrm{d}P_1}(X,Y_1) - \mathbb{E} \log \frac{\mathrm{d}P_2}{\mathrm{d}P_1}(X,Y_1)$ is sub-Gamma in both tails with variance parameter $\bigl(8\vee\frac{2}{\sigma^2}\bigr)\epsilon^2$ and scale parameter $c\coloneqq \frac{2R}{\sigma}\vee\frac{4R^2}{3\sigma^2}$.
        \item[(d)] $\frac{1-e^{-R^2/(2\sigma^2)}}{2R^2} \cdot \mathbb{E}\bigl\{\bigl(g_1(X)-g_2(X)\bigr)^2\bigr\} \leq d_{\mathrm{H}}^2(P_1,P_2) \leq \frac{1}{4\sigma^2}\cdot \mathbb{E}\bigl\{\bigl(g_1(X)-g_2(X)\bigr)^2\bigr\}$.
    \end{itemize}
\end{lemma}
\begin{proof}
    \emph{(a)} Let $P_{Y_1|X}$ and $P_{Y_2|X}$ denote the conditional distributions of $Y_1$ given $X$ and $Y_2$ given $X$ respectively. Then, using the fact that $Y_1=g_1(X)+\xi$,
    \begin{align}
        \log\frac{\d P_2}{\d P_1}(X,Y_1) = \log\frac{\d P_{Y_2|X}}{\d P_{Y_1|X}}(X,Y_1) &= -\frac{\bigl(Y_1-g_2(X)\bigr)^2}{2\sigma^2} + \frac{\bigl(Y_1-g_1(X)\bigr)^2}{2\sigma^2}\nonumber\\
        &= \frac{-\bigl(g_1(X)-g_2(X)\bigr)^2 - 2\bigl(g_1(X)-g_2(X)\bigr)\xi}{2\sigma^2}. \label{eq:log-ll-gaussian-regression}
    \end{align}
    Hence,
    \begin{align}
        \mathrm{KL}(P_1,P_2) = -\mathbb{E} \log\frac{\d P_2}{\d P_1}(X,Y_1) = \frac{1}{2\sigma^2} \mathbb{E}\bigl\{\bigl(g_1(X)-g_2(X)\bigr)^2\bigr\}. \label{eq:kl-gaussian-regression}
    \end{align}
    
    \medskip
    \emph{(b)} By~\eqref{eq:log-ll-gaussian-regression} and~\eqref{eq:kl-gaussian-regression}, we have
    \begin{align*}
        \mathrm{V}_2(P_1,P_2) &= \frac{1}{2\sigma^2}\mathbb{E} \Bigl[ \Bigl\{ \bigl(g_1(X)-g_2(X)\bigr)^2 + 2\bigl(g_1(X)-g_2(X)\bigr)\xi - \mathbb{E}\bigl\{\bigl(g_1(X)-g_2(X)\bigr)^2\bigr\} \Bigr\}^2 \Bigr]\\
        &= \frac{1}{2\sigma^2} \Var\bigl\{\bigl(g_1(X)-g_2(X)\bigr)^2\bigr\} + 2\mathbb{E}\bigl\{ \bigl(g_1(X)-g_2(X)\bigr)^2 \bigr\},
    \end{align*}
    where in the final equality, we used the fact that $\xi\indep X$, $\mathbb{E}(\xi)=0$ and $\mathbb{E}(\xi^2)=\sigma^2$.

    \medskip
    \emph{(c)} From~\eqref{eq:log-ll-gaussian-regression} and~\eqref{eq:kl-gaussian-regression}, we deduce that
    \begin{align*}
        W &= \log \frac{\mathrm{d}P_2}{\mathrm{d}P_1}(X,Y_1) - \mathbb{E} \log \frac{\mathrm{d}P_2}{\mathrm{d}P_1}(X,Y_1)\\
        &= \frac{\bigl(g_2(X)-g_1(X)\bigr)\xi}{\sigma^2} - \frac{\bigl(g_1(X)-g_2(X)\bigr)^2 - \mathbb{E}\bigl\{\bigl(g_1(X)-g_2(X)\bigr)^2\bigr\}}{2\sigma^2} \eqqcolon W_1+W_2.
    \end{align*}
    For $\lambda\in [0, \sigma/R)$, using the fact that $\log(1+x) \leq x$ for $x \geq 0$ and Fubini's theorem,
    \begin{align*}
        \log\mathbb{E}e^{\lambda W_1} &= \log\mathbb{E} \exp\biggl(\frac{\lambda\bigl(g_2(X)-g_1(X)\bigr)\xi}{\sigma^2}\biggr) = \log\mathbb{E} \exp\biggl(\frac{\lambda^2 \bigl(g_2(X)-g_1(X)\bigr)^2}{2\sigma^2}\biggr)\\
        &\leq \sum_{r=1}^{\infty} \frac{\lambda^{2r}\mathbb{E}\bigl\{\bigl(g_2(X)-g_1(X)\bigr)^{2r}\bigr\}}{(2\sigma^2)^r \cdot r!} \leq \mathbb{E}\bigl\{\bigl(g_2(X)-g_1(X)\bigr)^2\bigr\} \sum_{r=1}^{\infty} \frac{\lambda^{2r} (2R)^{2r-2}}{(2\sigma^2)^r \cdot r!}\\
        &\leq \lambda^2\epsilon^2 \sum_{r=1}^{\infty} \frac{2^{r-1}}{r!}\biggl(\frac{\lambda R}{\sigma}\biggr)^{2r-2} \leq \lambda^2\epsilon^2 \sum_{r=1}^{\infty} \biggl(\frac{\lambda R}{\sigma}\biggr)^{r-1} = \frac{2\lambda^2\epsilon^2}{2\bigl(1-\frac{\lambda R}{\sigma}\bigr)}.
    \end{align*}
    The same holds true if $W_1$ is replaced by $-W_1$.
    Thus $W_1$ is sub-Gamma in both tails with variance parameter $2\epsilon^2$ and scale parameter $R/\sigma$. By part \emph{(b)} and the assumption that $\mathrm{V}_2(P_1,P_2)\leq\epsilon^2$, we have $\mathbb{E}(W_2^2) \leq \epsilon^2/(2\sigma^2)$. Moreover, since $W_2$ is bounded by $2R^2/\sigma^2$, we have for $r\geq 3$ that 
    \[
    \mathbb{E}(|W_2|^r) \leq \mathbb{E}(W_2^2) \cdot \Bigl(\frac{2R^2}{\sigma^2}\Bigr)^{r-2} \leq \frac{r!\epsilon^2}{2\cdot 2\sigma^2} \Bigl(\frac{2R^2}{3\sigma^2}\Bigr)^{r-2}.
    \]
    The same bound holds for $-W_2$, so by \citet[Theorem 2.10]{boucheron2013concentration}, $W_2$ is sub-Gamma in both tails with variance parameter $\frac{\epsilon^2}{2\sigma^2}$ and scale parameter $\frac{2R^2}{3\sigma^2}$. Finally, it follows that that $W=W_1+W_2$ is sub-Gamma in both tails with variance parameter $\bigl(8\vee\frac{2}{\sigma^2}\bigr)\epsilon^2$ and scale parameter $c= \frac{2R}{\sigma}\vee\frac{4R^2}{3\sigma^2}$.

    \medskip
    \emph{(d)} We have
    \begin{align*}
        d_{\mathrm{H}}^2(P_1,P_2) = \mathbb{E}\, d_{\mathrm{H}}^2(P_{Y_1|X},P_{Y_2|X}) &= 2\,\mathbb{E} \biggl\{1-\exp\biggl(-\frac{\bigl(g_1(X)-g_2(X)\bigr)^2}{8\sigma^2}\biggr)\biggr\}\\
        &\geq \frac{1-e^{-R^2/(2\sigma^2)}}{2R^2} \cdot \mathbb{E}\bigl\{\bigl(g_1(X)-g_2(X)\bigr)^2\bigr\},
    \end{align*}
    where the inequality follows since $x \mapsto (1-e^{-x})/x$ is  decreasing on $(0,\infty)$ and moreover $\frac{(g_1(X)-g_2(X))^2}{8\sigma^2} \leq \frac{R^2}{2\sigma^2}$. On the other hand, using the fact that $1-e^{-x} \leq x$ for $x \geq 0$, we also have
    \begin{align*}
        2\mathbb{E} \biggl\{1-\exp\biggl(-\frac{\bigl(g_1(X)-g_2(X)\bigr)^2}{8\sigma^2}\biggr)\biggr\} \leq \frac{1}{4\sigma^2}\cdot \mathbb{E}\bigl\{\bigl(g_1(X)-g_2(X)\bigr)^2\bigr\},
    \end{align*}
    as required.
\end{proof}

\begin{lemma}\label{lemma:L2-subset-lb}
    Suppose $f\in L^2([0,1]^d)$ has wavelet decomposition
    \begin{align*}
        f = \sum_{k\in K} a_k\Phi_k + \sum_{\ell=\ell_0}^{\infty} \sum_{\gamma\in\Gamma_{\ell}} b_{\ell,\gamma} \Psi_{\ell,\gamma}.
    \end{align*}
    Let $0\leq c_1 < c_2 \leq 1$ and let $\Lambda \coloneqq \bigl\{(\ell,\gamma) : \mathrm{supp}(\Psi_{\ell,\gamma}) \subseteq [c_1,c_2]^d\bigr\}$. Then 
    \begin{align*}
        \|f\|_{L^2([c_1,c_2]^d)}^2 \geq \sum_{(\ell,\gamma) \in \Lambda} b^2_{\ell,\gamma}.
    \end{align*}
\end{lemma}
\begin{proof}
    Note that $\mathrm{supp}(\Psi_{\ell,\gamma}) \subseteq [c_1,c_2]^d$ for all $(\ell,\gamma) \in \Lambda$, so by the orthonormality of the  wavelet basis, we have that $(\Psi_{\ell,\gamma})_{(\ell,\gamma) \in \Lambda}$ is also a sequence of orthonormal functions in the Hilbert space $L^2([c_1,c_2]^d)$ equipped with the inner product
    \begin{align*}
        \langle f_1,f_2 \rangle_{[c_1,c_2]^d} \coloneqq \int_{[c_1,c_2]^d} f_1(x)f_2(x) \,\d x.
    \end{align*}
    Moreover, for $(\ell,\gamma) \in \Lambda$, we have
    \begin{align*}
        \langle f, \Psi_{\ell,\gamma} \rangle_{[c_1,c_2]^d} = \int_{[c_1,c_2]^d} f(x)\Psi_{\ell,\gamma}(x) \,\d x = \int_{[0,1]^d} f(x)\Psi_{\ell,\gamma}(x) \,\d x = b_{\ell,\gamma}.
    \end{align*}
    Thus, an application of Bessel's inequality \citep[e.g.][Theorem~4.17]{rudin1987real} yields the desired result.
\end{proof}

\begin{lemma}\label{eq:small-ball-prob-stiefel-manifold}
    Let $U^{(r)}$ be uniformly distributed on the Stiefel manifold $V_r(\mathbb{R}^d)$ and let $U\in V_r(\mathbb{R}^d)$ be deterministic. Then for any $\epsilon>0$, there exists $c>0$, depending only on $d$ and $r$, such that
    \begin{align*}
        \mathbb{P}\bigl(\|U^{(r)} -U\|_{\mathrm{op}}\leq\epsilon\bigr) \geq c\biggl(\epsilon\wedge \frac{1}{2}\biggr)^{dr}.
    \end{align*}
\end{lemma}
\begin{proof}
    We may assume that $\epsilon \leq 1/2$.
    By \citet[Theorem~2.2.1(iii)]{chikuse2003statistics}, we may assume without loss of generality that $U^{(r)} = Z(Z^\top Z)^{-1/2} \eqqcolon F(Z)$, where $Z$ is a $d\times r$ random matrix with independent $N(0,1)$ entries. Let $\Delta\in\mathbb{R}^{d\times r}$ be such that $\delta\coloneqq \|\Delta\|_{\mathrm{op}} \leq 1/6$, and define
    \begin{align*}
        A\coloneqq U^\top\Delta + \Delta^\top U + \Delta^\top\Delta \in \mathbb{R}^{r \times r}.
    \end{align*}
    Then $\|A\|_{\mathrm{op}} \leq 3\delta \leq 1/2$, so $(U+\Delta)^\top(U+\Delta) = I_r + A$ is positive definite.  Thus
    \begin{align*}
        \|F(U+\Delta) - U\|_{\mathrm{op}} &= \|(U+\Delta)(I_r+A)^{-1/2} - U\|_{\mathrm{op}}\\
        &\leq \|U\|_{\mathrm{op}}\|(I_r+A)^{-1/2} - I_r\|_{\mathrm{op}} + \|\Delta\|_{\mathrm{op}}\|(I_r+A)^{-1/2}\|_{\mathrm{op}}.
    \end{align*}
    Moreover, we can write $I_r+A = VDV^\top$ where $V \in \mathbb{R}^{r \times r}$ is orthogonal and $D \in \mathbb{R}^{r \times r}$ is diagonal with diagonal elements bounded between $1-3\delta$ and $1+3\delta$. Thus, $\|(I_r+A)^{-1/2} - I_r\|_{\mathrm{op}} \leq 3\delta$ and $\|(I_r+A)^{-1/2}\|_{\mathrm{op}} \leq 2$, so
    \begin{align*}
        \|F(U+\Delta) - U\|_{\mathrm{op}} \leq 3\delta + 2\delta = 5\delta.
    \end{align*}
    Hence, if $\|Z-U\|_{\mathrm{op}} \leq \epsilon/5 < 1/6$, then $\|F(Z)-U\|_{\mathrm{op}}\leq\epsilon$, so
    \begin{align*}
        \mathbb{P}\bigl(\|U^{(r)} &-U\|_{\mathrm{op}}\leq\epsilon\bigr) = \mathbb{P}\bigl(\|F(Z)-U\|_{\mathrm{op}}\leq\epsilon\bigr) \geq \mathbb{P}\biggl(\|Z-U\|_{\mathrm{op}}\leq \frac{\epsilon}{5}\biggr)\\
        &\geq \mathbb{P}\biggl(\|Z-U\|_{\infty} \leq \frac{\epsilon}{5\sqrt{dr}}\biggr) \geq \biggl\{\mathbb{P}\biggl(|Z_{11} - 1| \leq \frac{\epsilon}{5\sqrt{dr}}\biggr)\biggr\}^{dr} \geq \biggl(\frac{2\epsilon}{25\sqrt{dr}}\biggr)^{dr},
    \end{align*}
    where the third inequality follows since $\|U\|_{\infty} \leq 1$ and $Z$ has independent $N(0,1)$ entries, and the final inequality follows since $\frac{\epsilon}{5\sqrt{dr}}\leq \frac{1}{10}$ and the standard normal density is bounded below by $1/5$ on the interval $[0.9,1.1]$. 
\end{proof}

\begin{lemma} \label{lemma:density-lb}
    Let $P_X$ be a distribution on $\mathbb{B}^d$ that satisfies the assumptions of Lemma~\ref{lemma:prior-mass-ub-multi-index}.  Let $p\in[d]$, $U\in V_p(\mathbb{R}^d)$, $X\sim P_X$ and let $Q$ denote the distribution of $\frac{U^\top X + 1_p}{2}$.  Then there exist $c_0' \equiv c_0'(c_0,\tau,d,p)$ and a hypercube $A'\subseteq [0,1]^p$ of side length~$\frac{\tau}{2\sqrt{p}}$ with the property that $Q(A_0') \geq c_0'\mathrm{Vol}_p(A_0')$ for all measurable $A_0'\subseteq A'$.
\end{lemma}
\begin{proof}
    If $p=d$, then $U \in \mathbb{R}^{d \times d}$ is orthogonal, so for all measurable $A_0\subseteq A$, 
    \[
    Q\biggl(\frac{U^\top A_0 + 1_d}{2}\biggr) = \mathbb{P}(X \in A_0) \geq c_0 \mathrm{Vol}_d(A_0). 
    \]
    Moreover, $(U^\top A + 1_d)/2$ is a Euclidean ball of radius $\tau/2$, so it contains a hypercube of side length $\frac{\tau}{2\sqrt{p}}$.  Now suppose that $p \in [d-1]$.  For $v=(v_1,\ldots,v_d)^\top \in\mathbb{R}^d$ and $1\leq j_1\leq j_2 \leq d$, we write $v_{j_1:j_2} \coloneqq (v_{j_1}, \ldots, v_{j_2})^\top \in \mathbb{R}^{j_1-j_2+1}$.
    Let $o\in\mathbb{R}^d$ be the centre of $A$ and let $K\subseteq A$ be the Euclidean ball of radius $\tau/2$ and centred at $o$. We may write $U = VP$ where $V\in\mathbb{R}^{d\times d}$ is orthogonal and $P\in\mathbb{R}^{d\times p}$ is the first~$p$ columns of $I_d$. Let $S\subseteq K$ be Borel measurable, and define $T\coloneqq \{x\in V^\top A : x_{1:p} \in P^\top V^\top S\}$. Then $VT \subseteq A$; moreover, if $y= (V^\top x)_{1:p}$ for some $x\in S \subseteq K$ and $z\in\mathbb{R}^{d-p}$ satisfies $\|z - (V^\top o)_{(p+1):d}\|_2 \leq \tau/2$, then by the choice of $K$,
    \begin{align*}
        \biggl\|\begin{pmatrix} y\\z\end{pmatrix} - V^\top o\biggr\|_2^2 = \|y - (V^\top o)_{1:p}\|_2^2 + \|z - (V^\top o)_{(p+1):d}\|_2^2 \leq (\tau/2)^2 + (\tau/2)^2 < \tau^2,
    \end{align*}
    so $(y^\top, z^\top)^\top \in T$. Thus
    \begin{align}
        T'\coloneqq \bigl\{(y^\top, z^\top)^\top\in \mathbb{R}^{p+(d-p)} : y\in P^\top V^\top S,\, \|z - (V^\top o)_{(p+1):d}\|_2 \leq \tau/2\bigr\} \subseteq T. \label{eq:subset-of-T}
    \end{align}
    Therefore, writing $v(\tau,d,p)$ for the Lebesgue measure of a Euclidean ball in $\mathbb{R}^{d-p}$ of radius $\tau/2$, we deduce that
    \begin{align*}
        Q\biggl(\frac{U^\top S + 1_p}{2} \biggr) &= \mathbb{P}\biggl(\frac{P^\top V^\top X + 1_p}{2} \in \frac{P^\top V^\top S + 1_p}{2} \biggr)\\
        &\geq \mathbb{P}(V^\top X \in T) = \mathbb{P}(X \in VT) \geq c_0\mathrm{Vol}_d(VT) = c_0\mathrm{Vol}_d(T) \geq c_0\mathrm{Vol}_d(T')\\
        &= c_0 \int_{y\in P^\top V^\top S} \int_{z\in\mathbb{R}^{d-p} : \|z - (V^\top o)_{(p+1):d}\|_2 \leq \tau/2} \,\d z \,\d y\\
        &= c_0 v(\tau,d,p) \mathrm{Vol}_p(P^\top V^\top S) \eqqcolon c_0'\mathrm{Vol}_p\biggl(\frac{U^\top S + 1_p}{2}\biggr).
    \end{align*}
    Since this holds for all measurable $S\subseteq K$, we deduce that $Q(A_0') \geq c_0'\mathrm{Vol}_p(A_0')$ for all $A_0' \subseteq \frac{U^\top K + 1_p}{2}$. Now $\frac{U^\top K + 1_p}{2}$ is a ball of radius $\tau/4$ in $\mathbb{R}^p$, so it  contains a hypercube $A'$ of side length~$\frac{\tau}{2\sqrt{p}}$.
\end{proof}

\section{ICL Example} \label{sec:ICL-example}
When we provide ChatGPT with the first two in-context examples and the query from the right column of Figure~\ref{fig:ICL-example} (i.e.~the prompt is `Cat $\to$ 10, Fish $\to$ 25, Bird $\to$ ?'), we obtain the following output (though this may vary due to uncertainty in the sampling procedure):
\begin{center}
\fbox{%
  \begin{minipage}{0.8\linewidth}
  There isn't a unique answer from just Cat $\to$ 10 and Fish $\to$ 25. 
    
    With only two example pairs, infinitely many rules can fit them, and those rules would give different values for Bird. So Bird $\to$ ? is underdetermined (cannot be determined uniquely).
    
    If you give one more example mapping (any other word $\to$ number), I can pin down the rule and compute Bird.
  \end{minipage}
}
\end{center}

When ChatGPT is provided with all four examples and the query from the right column of Figure~\ref{fig:ICL-example} (i.e.~now the prompt is `Cat $\to$ 10, Fish $\to$ 25, Deer $\to$ 17, Dog $\to$ 13, Bird $\to$ ?'), we obtain the following:
\begin{center}
\fbox{%
  \begin{minipage}{0.8\linewidth}
  The pattern is:

    Number = (number of letters in the word) $\times$ (alphabet position of the first letter) + 1

    Check:

    \begin{itemize}[leftmargin=*, itemsep=0pt, topsep=0pt]
        \item Cat $\to$ 3 $\times$ C(3) + 1 = 10
        \item Fish $\to$ 4 $\times$ F(6) + 1 = 25
        \item Deer $\to$ 4 $\times$ D(4) + 1 = 17
        \item Dog $\to$ 3 $\times$ D(4) + 1 = 13
    \end{itemize}
    So:
    \begin{itemize}[leftmargin=*, itemsep=0pt, topsep=0pt]
    \item Bird $\to$ 4 $\times$ B(2) + 1 = 9
    \end{itemize}
  \end{minipage}
}
\end{center}

\section{Further details for simulations in Section~\ref{sec:simulations}}\label{sec:details-for-simulations}
\subsection{Data generating mechanism} \label{sec:data-generating-mechanism}
The generating mechanism of our pretraining data is as described in Section~\ref{sec:multi-index-regression}, with $C_0=2$, $(a_k^{(\alpha)}), (b_{\ell,\gamma}^{(\alpha)}) \overset{\mathrm{iid}}{\sim} \mathrm{Unif}[-1,1]$ in~\eqref{eq:random-g-alpha} and $P_X$ being the uniform distribution on $\mathbb{B}^d$. 

At test time, we randomly generate new single-index models with fixed smoothness levels under two settings:
\begin{itemize}[itemsep=0pt,topsep=0pt]
    \item[(i)] No distribution shift: the random test functions are generated from the same distribution as the pretraining data of that particular smoothness,
    \item[(ii)] With distribution shift: similarly as (i), except that the random variables $(a_k^{(\alpha)})$ in~\eqref{eq:random-g-alpha} are sampled from $\mathrm{Unif}[0,1]$ instead of $\mathrm{Unif}[-1,1]$.
\end{itemize} 
See Figure~\ref{fig:random-besov} for random Besov functions with and without distribution shift (blue: no distribution shift; orange: with distribution shift). 

\subsection{Training details} \label{sec:training-details}
For training, we used AdamW \citep{loshchilov2019decoupled}, with weight decay 0.001 and learning rate 0.0003 (with a cosine annealing schedule). We also employed pre-layer normalisation \citep{xiong2020layer} and gradient clipping to stabilise the optimisation. Our Transformer was trained for 200{,}000 optimisation steps. At each step, we sampled 36 independent single-index regression functions, and for each function, we randomly generated 16 prompts and queries (i.e.~the covariate vectors and queries are random while keeping the regression function fixed), yielding a total of 576 pretraining sequences per optimisation step. The training was done on 12 NVIDIA A100 GPUs (80GB) over approximately 17 hours.  To facilitate training, the difficulty of the tasks was progressively increased via a curriculum learning strategy. In the initial warm-up phase (the first 10{,}000 training steps), the ambient dimension $d$ was gradually increased from 1 to~5, while keeping $p=1$ and $\alpha=4$ fixed.  Subsequently, for the next 50{,}000 steps, we introduced a smoothness curriculum, where regression functions were drawn from increasingly diverse mixtures whose smoothness parameters were gradually decreased over time. Finally, for the remaining 140{,}000 steps, training proceeded with a uniform mixture over tasks, where the smoothness parameter was sampled uniformly from $\{2,\, 2.5,\, 3,\, 4\}$.

\subsection{Further simulation results} \label{sec:further-simulations}
Our theoretical guarantees in Section~\ref{sec:adaptive-icl} assume that the support of the test distribution is contained in the support of the pretraining prior. Our simulations in Section~\ref{sec:simulations} also consider a within-support distribution shift. In this subsection, we further examine the empirical effect of different within-support and out-of-support distribution shifts on the performance of our pretrained Transformer. To this end, we consider the following six types of distribution shifts:
\begin{itemize}[itemsep=0pt,topsep=0pt]
    \item[(i)] No distribution shift: the random test functions are generated from the same distribution as the pretraining data of that particular smoothness,
    \item[(ii)] Within-support shift $[0,1]$: similarly as (i), except that the random variables $(a_k^{(\alpha)})$ in~\eqref{eq:random-g-alpha} are sampled from $\mathrm{Unif}[0,1]$ instead of $\mathrm{Unif}[-1,1]$.
    \item[(iii)] Within-support shift $[0.5,1]$: similarly as (i), except that the random variables $(a_k^{(\alpha)})$ in~\eqref{eq:random-g-alpha} are sampled from $\mathrm{Unif}[0.5,1]$ instead of $\mathrm{Unif}[-1,1]$.
    \item[(iv)] Within-support shift $[0.75, 1]$: similarly as (i), except that the random variables $(a_k^{(\alpha)})$ in~\eqref{eq:random-g-alpha} are sampled from $\mathrm{Unif}[0.75,1]$ instead of $\mathrm{Unif}[-1,1]$.
    \item[(v)] Out-of-support shift $[0,2]$: similarly as (i), except that the random variables $(a_k^{(\alpha)})$ in~\eqref{eq:random-g-alpha} are sampled from $\mathrm{Unif}[0,2]$.
    \item[(vi)] Out-of-support shift $[1,2]$: similarly as (i), except that the random variables $(a_k^{(\alpha)})$ in~\eqref{eq:random-g-alpha} are sampled from $\mathrm{Unif}[1,2]$.
\end{itemize} 
Settings (i) and (ii) above are the same as those in Appendix~\ref{sec:data-generating-mechanism}, while settings (iii) and (iv) impose more severe within-support distribution shifts and settings (v) and (vi) are out-of-support distribution shifts. The MSEs under these settings are given in Figure~\ref{fig:OOD}.

\begin{figure}[htbp]
    \centering
    \includegraphics[width=0.7\linewidth]{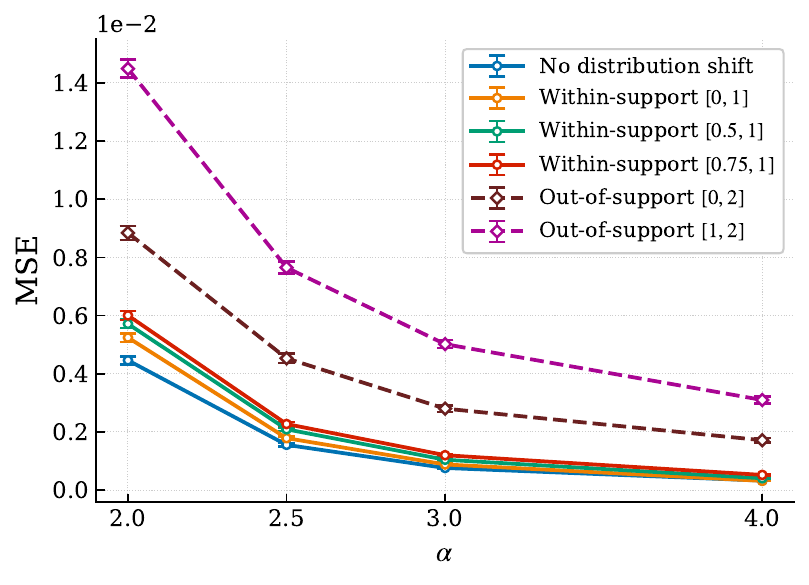}
    \caption{MSE at test time based on $100{,}000$ random test functions. Error bars represent 90\% confidence intervals.}
    \label{fig:OOD}
\end{figure}

To understand the effect of the pretraining data sample size,  in Figure~\ref{fig:log-log-pretraining} we plot the logarithm of the MSE against the logarithm of the number of pretraining sequences, obtained from the checkpoints of our simulation study. We see that there is an approximately linear trend, suggesting a polynomial decay of the MSE with respect to the number of pretraining sequences.

\begin{figure}[htbp]
    \centering
    \includegraphics[width=0.7\linewidth]{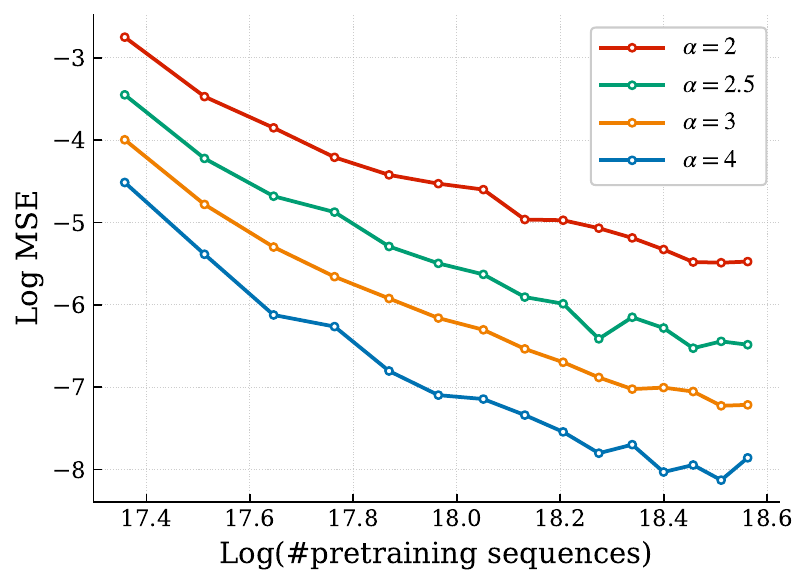}
    \caption{The logarithm of the MSE against the logarithm of the number of pretraining sequences.}
    \label{fig:log-log-pretraining}
\end{figure}

\end{document}